\documentclass{alt2027}
\usepackage{times}
\usepackage[T1]{fontenc}
\usepackage[utf8]{inputenc}
\usepackage{bm,mathtools,microtype,booktabs,tikz,xurl}
\usetikzlibrary{calc}
\usepackage[breakable]{tcolorbox}
\newtcolorbox{yellowbox}{
  colback=yellow!7!white,
  colframe=blue!35!black,
  boxrule=0.4pt,arc=1.5pt,
  left=7pt,right=7pt,top=5pt,bottom=5pt,
  before skip=8pt,after skip=8pt,
  fontupper=\normalfont\normalsize
}

\let\If\relax
\let\Else\relax
\let\For\relax

\usepackage[noend]{algpseudocode}
\usepackage{wrapfig,ragged2e}
\makeatletter
\newcommand{\wrappedalgorithmcaption}{%
  \def\@captype{algocf}%
  \renewcommand\@makecaption[2]{%
    \parbox[t]{\linewidth}{\algocf@captiontext{##1}{##2}}%
  }%
}
\newcommand{\WrappedState}[1]{%
  \State\parbox[t]{\dimexpr\linewidth-\ALG@tlm\relax}{%
    \raggedright\strut#1\strut}%
}
\makeatother
\AtBeginEnvironment{algorithm}{\small}
\hypersetup{pdftitle={Three Worlds for the Price of a Logarithm: Baseline-Protected Best-of-Both-Worlds Combinatorial Semi-Bandits},pdfauthor={},pdfsubject={Anonymous ALT 2027 submission}}

\newtheorem{mainlemma}{Lemma}
\newtheorem{mainremark}{Remark}
\let\appendixlemma\lemma
\let\endappendixlemma\endlemma
\let\appendixremark\remark
\let\endappendixremark\endremark
\let\lemma\mainlemma
\let\endlemma\endmainlemma
\let\remark\mainremark
\let\endremark\endmainremark
\AddToHook{cmd/appendix/after}{%
  \let\lemma\appendixlemma
  \let\endlemma\endappendixlemma
  \let\remark\appendixremark
  \let\endremark\endappendixremark
  \numberwithin{theorem}{section}%
}
\AtBeginEnvironment{theorem}{\phantomsection}
\AtBeginEnvironment{lemma}{\phantomsection}
\AtBeginEnvironment{remark}{\phantomsection}
\AtBeginEnvironment{proposition}{\phantomsection}
\AtBeginEnvironment{corollary}{\phantomsection}
\AtBeginEnvironment{definition}{\phantomsection}
\numberwithin{equation}{section}
\newtagform{blue}[\textcolor{blue}]{(}{)}
\usetagform{blue}
\theorembodyfont{\normalfont\itshape}
\definecolor{DarkCyan}{HTML}{008B8B}
\newtheorem{assumption}{Assumption}

\makeatletter
\newcommand{\assumption@begintheorem}[2]{%
  \assumption@original@begintheorem
    {\textcolor{DarkCyan}{#1}}{\textcolor{DarkCyan}{#2}}%
}
\newcommand{\assumption@opargbegintheorem}[3]{%
  \assumption@original@opargbegintheorem
    {\textcolor{DarkCyan}{#1}}{\textcolor{DarkCyan}{#2}}{#3}%
}
\AtBeginEnvironment{assumption}{%
  \setlength{\topsep}{3pt plus 1pt minus 1pt}
  \setlength{\partopsep}{0pt}
  \phantomsection
  \let\assumption@original@begintheorem\@begintheorem
  \let\assumption@original@opargbegintheorem\@opargbegintheorem
  \let\@begintheorem\assumption@begintheorem
  \let\@opargbegintheorem\assumption@opargbegintheorem
}
\makeatother
\newenvironment{informaltheorem}[1][Theorem~\ref{thm:main-result}]
  {\trivlist\item[\hskip\labelsep{\normalfont\bfseries Theorem (informal version of #1)}]\normalfont\itshape}
  {\endtrivlist}
\newtcolorbox{questionbox}{
  colback=blue!3!white, colframe=blue!35!black,
  boxrule=0.4pt, arc=1.5pt,
  left=7pt, right=7pt, top=5pt, bottom=5pt,
  before skip=8pt, after skip=8pt,
  halign=flush center,
  fontupper=\normalfont\normalsize\itshape
}
\newcommand{\cX}{\mathcal X}
\newcommand{\cP}{\mathcal P}
\newcommand{\cF}{\mathcal F}
\newcommand{\cG}{\mathcal G}
\newcommand{\E}{\mathbb E}
\newcommand{\Prb}{\mathbb P}

\newcommand{\one}{\mathbf 1}
\def\1{\mathbf 1}
\newcommand{\KL}{\operatorname{KL}}
\newcommand{\Var}{\operatorname{Var}}
\newcommand{\supp}{\operatorname{supp}}
\newcommand{\conv}{\operatorname{conv}}
\newcommand{\Duet}{\textnormal{\textsc{Duet}}}
\usepackage{algorithm}
\usepackage{algpseudocode}
\usepackage{float}

\floatstyle{ruled}
\restylefloat{algorithm}%
\editor{}

\makeatletter
\let\ps@jmlrtps\ps@plain
\makeatother
\hypersetup{pdfsubject={Preprint},pdfauthor={Ting Hu; Luanda Cai; Emmanouil-Vasileios Vlatakis-Gkaragkounis}}
\title[Three Worlds for the Price of a
Logarithm]{Three Worlds for the Price of a
Logarithm:\titlebreak Baseline-Protected Best-of-Both-Worlds Combinatorial Semi-Bandits}
\altauthor{%
  \Name{Ting Hu}
  \Email{ting.hu@wisc.edu}\\
  \addr Department of Economics\\
  University of Wisconsin--Madison
  \AND
  \Name{Luanda Cai}
  \Email{luanda.cai@wisc.edu}\\
  \addr Department of Finance\\
  University of Wisconsin--Madison
  \AND
  \Name{Emmanouil-Vasileios Vlatakis-Gkaragkounis}
 \Email{vlatakis@wisc.edu}\\
  \addr Department of Computer Sciences\\
  University of Wisconsin--Madison
}

\begin{document}
\maketitle

\begin{abstract}
Improving on a baseline requires exploring away from
it, precisely where its performance may be lost.
We propose \Duet{} (\textbf{D}oubling \textbf{U}nder \textbf{E}mpirical \textbf{T}ests), a
horizon-free phased mixing scheme for combinatorial semi-bandits with
arbitrary finite action sets. Two observable tests decide
when to double the learning weight and when
to double the empirical-regret budget and restart. We
instantiate the scheme with entropy-FTRL in mixed phases
and hybrid-FTRL in pure phases. Assuming the baseline
selects every arm with positive probability, this instantiation
keeps expected baseline regret logarithmic in $T$, while
its expected regret to every fixed action is
$\widetilde O(\sqrt T)$ under non-anticipating adversarial losses and polylogarithmic, gap-dependent,
under i.i.d.\ losses with a unique optimal action. All
three guarantees hold at every deterministic horizon, and
the algorithm knows neither the horizon, nor the
gap, nor the environment type. Baseline protection, adversarial
robustness, and stochastic adaptivity thus come from one
scheme and two base learners, at the price
of a logarithm.
\end{abstract}

%
\newpage
\section{Introduction}
\label{main:introduction}

Combinatorial semi-bandit learning models a repeated decision process
in which, at each round, (i)~the learner selects
a feasible subset of $d$ arms---a path in
a network, a slate of items, a set
of channels; (ii)~it incurs the sum of the
losses of the selected arms and observes each
of these losses individually; and (iii)~it updates its
selection rule, and the process repeats \citep{mainbib:audibert2014regret,mainbib:chen2013combinatorial,mainbib:kveton2015tight}.

In this general context, we measure performance by
expected \emph{regret}: the learner's expected cumulative excess loss
relative to any fixed action. The picture for
this benchmark is well understood: online mirror
descent \citep{mainbib:audibert2014regret} and follow-the-perturbed-leader \citep{mainbib:neu2016importance} guarantee $O(\sqrt T)$ expected
regret to every fixed action under adversarial losses;
optimistic and sampling-based methods obtain
gap-dependent $O(\log T)$ regret under i.i.d.\ losses \citep{mainbib:chen2013combinatorial,mainbib:kveton2015tight,mainbib:combes2015revisited,mainbib:wang2018thompson,mainbib:perrault2020statistical}; and
the hybrid regularizer of \citet{mainbib:zimmert2019beating} achieves both
guarantees at once, without being told which
regime it is facing---the \emph{best-of-both-worlds} program initiated
for multi-armed bandits by \citet{mainbib:bubeck2012best}; see also
\citet{mainbib:seldin2014practical,mainbib:auer2016nearly,mainbib:zimmert2021tsallis}.

In many of the applications behind this model,
however, the learner does not start from scratch.
A deployed routing policy, a production recommendation slate,
a hand-tuned scheduling rule: each is a
\emph{baseline}---a known, fixed, possibly randomized rule with
acceptable performance---and the learner is asked to
improve on it without paying much for
the attempt. This raises naturally the question:
\begin{questionbox}
Can one (semi-)bandit learner be robust to adversaries,
adaptive to stochastic gaps, and loyal to
its baseline at a logarithmic price---in the
dark about horizon, gap, and world alike?
\end{questionbox}
The regret benchmark is silent about the baseline:
the baseline is itself a fixed distribution over
actions, so a no-regret algorithm guarantees $O(\sqrt T)$ excess
loss relative to it, but nothing better. This excess accrues during rounds spent exploring away
from a rule that may already be nearly
optimal. Requiring instead that the expected excess loss
relative to the baseline remain bounded, or grow
at most polylogarithmically in $T$, gives the \emph{additive}
form of baseline protection studied for experts by
\citet{mainbib:evendar2008regret,mainbib:hu2026compass} and under bandit feedback by \citet{mainbib:muller2025best}. This
requirement is closely related to conservative exploration \citep{mainbib:wu2016conservative,mainbib:kazerouni2017conservative,mainbib:garcelon2020improved},
in which the learner must preserve a
prescribed fraction of the baseline's cumulative performance.

Baseline protection alone is easy: the learner can
simply follow the baseline. The challenge is to
preserve this protection while achieving both adversarial robustness
and stochastic adaptivity, without knowing the loss regime,
the horizon, or the stochastic gap. Improvement requires
departing from the baseline, but partial feedback makes
it difficult to determine how much departure is
justified. The relevant regret scales differ sharply: square-root
in $T$ against an adversary, but only polylogarithmic
in $T$, with inverse-gap dependence, under stochastic losses.
A budget calibrated to a stochastic gap may
be exhausted under adversarial losses; one calibrated to
the adversarial scale may permit unnecessary excess loss
when a gap exists. The budget must also
adapt as play continues. The learner must therefore
use feedback alone to decide when, and how
far, to depart from the baseline.

Three qualifications make the question well posed. First, the baseline must
\emph{cover} every arm, i.e., select each arm with probability at least
$\delta>0$.\footnote{Coverage lets the baseline itself supply exploration.
Without it, uniformly logarithmic baseline regret can
force $\Omega(T/\log T)$ regret to the best action,
precluding our adversarial guarantee (Proposition~\ref{prop:coverage-necessary}).} This is our only requirement on the
baseline, and it is what lets baseline rounds double as exploration rounds.
In particular, the baseline need not put positive mass on every
\emph{action}---prohibitive when the action set is exponential in $d$---and
may be supported on a handful of actions; in the routing example, every edge
must lie on some route the baseline uses, although the baseline may never use
the best route itself. Second, the protection is in expectation at each
deterministic horizon; conservative baseline constraints can instead hold
with high probability uniformly over time \citep{mainbib:wu2016conservative}.
Third, the adversary is \emph{non-anticipating}: it may choose the round-$t$
losses from the past and from the learner's announced distribution, but not
from the realized draw---the standard protocol under which importance-weighted
estimates are unbiased
\citep{mainbib:audibert2014regret,mainbib:neu2016importance}.

\paragraph{Our contributions in the context of related work.}
Our main contribution is that, with these qualifications
in mind, \emph{the question above can be answered
in the affirmative}. We propose \Duet{} (\textbf{D}oubling \textbf{U}nder
\textbf{E}mpirical \textbf{T}ests), a horizon-free phased mixing scheme with
two observable tests, instantiated with entropy-FTRL in mixed
phases and the hybrid-FTRL of \citet{mainbib:zimmert2019beating} in pure
phases (Algorithm~\ref{alg:phased-sb}). Throughout, \Duet{} refers to this instantiation.
These learners provide the empirical-regret certificates that turn
the two tests into the following guarantee.
\begin{informaltheorem}
For any finite action set and baseline covering
every arm, \Duet{} guarantees at every $T$ expected
regret (i)~$O(\log T)$ to the baseline; (ii)~$\widetilde O(\sqrt T)$
to any fixed action against non-anticipating adversaries; and
(iii)~gap-dependent $O(\mathrm{polylog}\,T)$ under i.i.d.\ losses with a unique
optimum. No horizon, gap, or regime knowledge is
required.
\end{informaltheorem}
Here and throughout, $\widetilde O$ hides polylogarithmic factors. The exact
bounds in Theorem~\ref{thm:main-result} make the dependence on the action set,
the coverage $\delta$ and the gap $\Delta$ explicit; (i) and~(ii) need no
stochastic assumption at all.
The baseline regret in~(i) is logarithmic, hence
near-constant: it is within a $\log T$ factor of
the constant guarantees available under full information \citep{mainbib:evendar2008regret}
and under bandit feedback with every action in
the baseline's support \citep{mainbib:muller2025best}. An overhead beyond
$O(\sqrt T)$ is unavoidable in general: \citet[Theorem~6]{mainbib:evendar2008regret} show
that any algorithm with $O(\sqrt T)$ regret to the
best expert must incur $\Omega(\sqrt T)$ regret to the
uniform average in the worst case, already for
two experts under full information. This rules out
combining $O(\sqrt T)$ worst-case regret with logarithmic baseline regret,
but does not determine the necessary overhead. The
particular logarithmic factors in our bounds reflect the
present construction and analysis.

To put this result in context, consider first
the line of work it extends. Under full information, \citet[Theorem~4]{mainbib:evendar2008regret} showed how to combine
constant regret to a baseline with near-$\sqrt T$
adversarial regret. \citet[Theorem~3.1]{mainbib:muller2025best} extended this guarantee to bandit
feedback, assuming a known horizon and a baseline
that assigns positive probability to every action. Returning
to full information, \citet[Theorem~4.1]{mainbib:hu2026compass} obtained a horizon-free procedure
that also adapts to stochastic losses. Their COMPASS-Hedge
achieves gap-dependent $O(\log^2 T)$ stochastic regret, $O(\log^2 T)$
baseline regret, and $\widetilde O(\sqrt T)$ adversarial regret
under controlled adaptivity. Its baseline guarantee is weaker
than constant regret, although it requires no coverage
assumption. These results leave open a universal procedure
for combinatorial semi-bandits: one algorithm combining baseline protection,
adversarial robustness, and stochastic adaptivity, without knowing the
horizon, gap, or regime. \Duet{} provides this combination.
Relative to \citet{mainbib:muller2025best}, we
continue this line in three directions at once:
the feedback is semi-bandit and the action set
combinatorial, so that the baseline only needs to
cover every \emph{arm}; the procedure is horizon-free and
learns its restart threshold from the data; and
the same procedure delivers a gap-dependent polylogarithmic
rate under stochastic losses, where \citet[Theorem~B.1]{mainbib:muller2025best} give
a separate, gap-independent bound. On the baseline
axis itself we do not improve on
their constant guarantee; Table~\ref{tab:comparison} summarizes the comparison.

The second line is best-of-both-worlds learning. For semi-bandits,
\citet{mainbib:zimmert2019beating} combine $O(\log T)$ stochastic and $O(\sqrt T)$ adversarial regret
with a single hybrid regularizer. Later work adds
corruption robustness and data-dependent guarantees \citep{mainbib:ito2021hybrid}, variance-dependent bounds
through refined regularization \citep{mainbib:tsuchiya2023further}, and general adaptive learning-rate
techniques \citep{mainbib:ito2024adaptive}. Further developments include perturbation-based learning for
$m$-set semi-bandits \citep{mainbib:zhan2025nearly} and efficient contextual combinatorial semi-bandits
\citep{mainbib:li2026efficient}. None of these
algorithms protects a baseline, and the difference is
not cosmetic: a best-of-both-worlds analysis bounds the expected
regret to a \emph{fixed} comparator at a \emph{fixed}
horizon, whereas a restart rule that compares the
learner to the \emph{empirically} best action at a
\emph{data-dependent} time needs the positive empirical maximum at
that time---and for the random regrets $D_T(X)$ to
the individual actions one only has $\max_X\E D_T(X)\le\E\max_X D_T(X)$. Supplying
this stronger certificate for both of our base
learners is the main technical contribution of the
paper, and we return to it below. Finally,
conservative bandits and reduction-based extensions preserve a prescribed
fraction of the baseline's cumulative performance \citep{mainbib:wu2016conservative,mainbib:du2021one,mainbib:yang2022reduction}.
Conservative interleaving instead controls the quality of selected
items relative to matched baseline items at each
round \citep{mainbib:katariya2019conservative}, while \citet{mainbib:hou2023probably} impose a per-action variance
budget. Our constraint is additive, in expectation, at
each deterministic horizon, and the baseline also supplies
observations for learning. Appendix~\ref{app:related-work} discusses these and further
related works in detail.

\begin{table}[t]
\makeatletter
\long\def\@makecaption#1#2{%
  \begingroup\footnotesize
  \setbox\@tempboxa\hbox{#1: }%
  \hangindent=\wd\@tempboxa
  \noindent #1: #2\par
  \endgroup
}
\makeatother
\renewcommand{\tableconts}[3]{%
  #2\label{#1}\vskip3pt
  {\centering #3\par}%
}
\floatconts{tab:comparison}%
{\caption{Regret guarantees, with other parameter dependencies suppressed;
see Theorem~\ref{thm:main-result} for exact \Duet{} bounds.
\citet[Theorem~B.1]{mainbib:muller2025best} use a separate stochastic algorithm;
the adversarial guarantee of \citet{mainbib:hu2026compass} assumes controlled adaptivity.}}%
{\footnotesize\setlength{\tabcolsep}{3pt}
\renewcommand{\arraystretch}{0.93}
\setlength{\aboverulesep}{1.5pt}
\setlength{\belowrulesep}{1.5pt}
\vspace{-2pt}
\begin{tabular}{@{}lllllll@{}}
\toprule
 & Feedback & Baseline covers & Baseline regret & Adversarial & Stochastic & Horizon \\
\midrule
\citet{mainbib:evendar2008regret} & full info. & (designated distr.) & $O(1)$ & near-$\sqrt T$ & --- & known \\
\citet{mainbib:muller2025best} & bandit & every action & $O(1)$ & $\widetilde O(\sqrt T)$ & $O(\sqrt T)$ & known \\
\citet{mainbib:hu2026compass} & full info. & not required & $O(\log^2 T)$ & $\widetilde O(\sqrt T)$ & $O(\log^2 T)$ & free \\
\citet{mainbib:zimmert2019beating} & semi-bandit & no baseline & --- & $O(\sqrt T)$ & $O(\log T)$ & free \\
\Duet{} (this paper) & semi-bandit & every arm & $O(\log T)$ & $\widetilde O(\sqrt T)$ & $O(\mathrm{polylog}\,T/\delta)$ & free \\
\bottomrule
\end{tabular}}
\vspace{-14pt}
\end{table}

\subsection{Key idea and technical contributions: two tests, one ledger}
\Duet{} adapts two quantities from the observed losses:
the weight placed on the learner and the
budget allowed for its empirical regret. Within a
\emph{phase}, it mixes a fresh base learner with
the baseline at a fixed weight; a \emph{stage}
keeps the budget fixed across phases. From importance-weighted
loss estimates, both tests compare performance with the
same empirically best fixed action over the phase,
so they share one linear minimization. \emph{Double the
budget} if the learner's positive empirical regret exceeds
it, and restart with a smaller weight. Otherwise,
\emph{double the weight} if the baseline's empirical excess
loss exceeds four times the budget and sampling
is not yet pure. Our instantiation uses entropy-FTRL
while mixing with the baseline and hybrid-FTRL when
sampling only from the learner. The doublings give
logarithmically many stages and phases per stage, without
knowing the horizon.

The two tests keep a single ledger. A
phase ending with a weight increase has accumulated
an estimated advantage over the baseline. Because weights
double within a stage, the gains banked by
these phases cover the potential loss of the
final phase, up to a fixed residual per
stage (Section~\ref{main:accounting}). Summing over stages and using unbiasedness gives
the expected baseline guarantee. For the learning guarantees,
each increased budget is less than twice the
empirical regret of the phase that triggered it,
so every budget increase is supported by observed
learning cost (cf.~Lemma~\ref{lem:main-accounting-transfer}). The remaining learning guarantees therefore require
control of the positive empirical regret of a
fresh FTRL learner at the random time when
its phase ends, summed over phases.

This is where the difficulty lies. A phase
ends at a stopping time determined by the
very estimates being bounded, and the comparator is
the empirically best action rather than a fixed
one: favourable estimation noise can benefit this comparator
and also trigger the phase's end. We therefore
need to bound the expected \emph{positive empirical maximum}
at a stopped endpoint, conditionally on the history
at the phase's start (cf.~Lemma~\ref{lem:main-stopped-certificates}). The two sampling regimes
call for different arguments. In \emph{mixed} phases, every
arm retains at least half its baseline observation
probability. This coverage controls the second moments of
the loss estimates against an adversary. Under stochastic
losses, the exponentiated estimated advantage of each suboptimal
action is a supermartingale; Ville's inequality and a
stopped-martingale variance bound then control the maximum. In
\emph{pure} phases, sampling uses only the learner, and
observation probabilities can approach zero. Here the hybrid-FTRL
decomposition supplies a nonnegative, \emph{nondecreasing} upper bound on
empirical regret. Monotonicity lets us compare a stopped
phase with a run of deterministic length, where
the self-bounding argument of \citet{mainbib:zimmert2019beating} gives the stochastic improvement.
Zero-loss padding equalizes action sizes while preserving the
tests, and conditional bounds make the phase analysis
valid after data-dependent restarts (Section~\ref{main:empirical-maximum}).

 \emph{Organization.} Sections~\ref{main:setting}--\ref{main:discussion} cover the setting, \Duet{} and its
guarantees, analysis, and discussion, respectively; complete proofs and
further related work appear in the appendices.
%

\section{Setting and objectives}
\label{main:setting}

\paragraph{Actions.}
Let $\cX\subseteq\{0,1\}^d$ be a finite, nonempty set of actions; action $X$
selects arm $i$ when $X_i=1$, and actions may select different numbers of
arms. Write $N=|\cX|$, $m=\max_{X\in\cX}\|X\|_1$, and $\cP=\conv(\cX)$, so that
a point of $\cP$ lists the arm-selection probabilities of a randomized action.
If $N=1$, playing the unique action gives zero regret, so we assume $N\geq2$;
then $1\leq m\leq d$.
\vspace{-0.5em}
\paragraph{Protocol.}
On round $t$, the learner announces a distribution $q_t$ on $\cX$; the
environment selects $\ell_t\in[0,1]^d$; the learner draws $X_t\sim q_t$, incurs
$\langle X_t,\ell_t\rangle$, and observes $\ell_{t,i}$ exactly on the selected
arms. A known bound $0<L\leq m$ satisfies $\langle X,\ell_t\rangle\leq L$ for
every action and round; $L=m$ always suffices.
\vspace{-0.5em}
\paragraph{Histories.}
The draw $X_t$ is a fixed measurable function of $q_t$ and of fresh learner
randomness $\xi_t$, and two sigma-fields record what is known when. The
completed history $\cF_t$ is generated by the initial information and
$(\ell_s,\xi_s)_{s\leq t}$; it may contain losses the learner never observed,
while the learner itself uses only observed feedback and its own randomness,
so its round-$t$ state is $\cF_{t-1}$-measurable. The pre-draw field $\cG_t$,
with $\cF_{t-1}\subseteq\cG_t\subseteq\cF_t$, contains $q_t$ but not the fresh
draw, and $x_t=\E[X_t\mid\cG_t]\in\cP$ is the arm-selection marginal fixed
before the draw.\newpage
\begin{assumption}[Non-anticipating losses]
\phantomsection\label{ass:non-anticipation}
For every $t\geq1$ and $i\in[d]$, {\small$\Prb(X_{t,i}=1\mid\cG_t,\ell_t)=x_{t,i}$}.
\end{assumption}
The environment may thus react to the past and to the announced $q_t$, but not
to the coin that selects $X_t$; this is exactly what makes the
importance-weighted estimates $X_{t,i}\ell_{t,i}/x_{t,i}$ unbiased whenever
$x_{t,i}>0$ \citep{mainbib:audibert2014regret,mainbib:neu2016importance}.
\vspace{-0.5em}
\paragraph{Baseline.}
A known distribution $q^c$ on $\cX$ is given, with marginal
$x^c=\E_{X\sim q^c}X$.
\begin{assumption}[Baseline arm coverage]
\phantomsection\label{ass:coverage}
$\delta:=\min_{i\in[d]}x_i^c>0$.
\end{assumption}
Coverage makes the baseline informative: any positive weight on $q^c$ keeps
every arm observable, hence every action's loss estimable, including actions
the baseline never plays---with $d=3$, a baseline playing $\{1,2\}$ and
$\{2,3\}$ with probability $1/2$ each has $x^c=(1/2,1,1/2)$ although $\{1,3\}$
is never selected. It also implies $\bigcup_{X\in\cX}\supp(X)=[d]$ and
$d\delta\leq m$.
\vspace{-0.5em}
\paragraph{Regret and objectives.}
For a fixed comparator $x\in\cP$, the expected regret is
\begingroup
\setlength{\abovedisplayskip}{2pt}
\setlength{\belowdisplayskip}{2pt}
\setlength{\abovedisplayshortskip}{0pt}
\setlength{\belowdisplayshortskip}{2pt}
\begin{equation*}
 \mathcal R_T(x):=\E\sum_{t=1}^T\langle X_t-x,\ell_t\rangle
 =\E\sum_{t=1}^T\langle x_t-x,\ell_t\rangle,
 \tag{Regret}
 \label{main:regret}
\end{equation*}
\endgroup
where the equality follows from Assumption~\ref{ass:non-anticipation}. Our objective is to minimize
$\max_{X\in\cX}\mathcal R_T(X)=\sup_{x\in\cP}\mathcal R_T(x)$, with the maximum
outside the expectation: the comparator is fixed in advance, not chosen for
the realized losses. At every deterministic horizon $T$ we seek,
simultaneously, logarithmic regret to $x^c$, sublinear adversarial regret, and
a gap-dependent stochastic fast rate.
\vspace{-0.5em}
\paragraph{Stochastic regime.}
Only the fast rate needs the next two assumptions.
\begin{assumption}[i.i.d.\ stochastic losses]
\phantomsection\label{ass:stochastic}
The losses $\ell_1,\ell_2,\ldots\overset{\mathrm{i.i.d.}}{\sim}P$, a law on
$[0,1]^d$, are independent of the learner seeds
$\xi_1,\xi_2,\ldots\overset{\mathrm{i.i.d.}}{\sim}\nu$. Coordinates within a
loss vector may be correlated.
\end{assumption}
A fixed law makes the past informative about the future; independence from
the seeds keeps the model i.i.d.\ after a data-dependent restart.
\begin{assumption}[Unique optimal action]
\phantomsection\label{ass:gap}
For $\bar\ell=\E_P\ell$, $X^\star$ uniquely minimizes $\langle X,\bar\ell\rangle$
over $\cX$; let $\Delta_X:=\langle X-X^\star,\bar\ell\rangle$ and
$\Delta:=\min_{X\neq X^\star}\Delta_X>0$.
\end{assumption}
Playing a suboptimal action with probability $p$ costs at least $p\Delta$ in
mean loss; this self-bounding relation between regret and suboptimal plays is
what drives the fast rate.
%
\section{Our results}
\label{main:algorithm}

Improving on the baseline requires learning better actions
while controlling the cost of trying them. \Duet{}
(\textbf{D}oubling \textbf{U}nder \textbf{E}mpirical \textbf{T}ests) gives these tasks to a
base learner and a mixing rule. We first
explain the two observable tests that adjust the
learner's weight and empirical-regret budget, then specify the
learners and state the guarantees. Section~\ref{main:proof} follows the
same route, explaining how the tests and learners
together control baseline regret and learning regret.

\subsection{The \Duet{} mixing scheme}
\label{main:tests}

\Duet{} takes a base "no-regret" learner's distribution
$\widehat q_t$, with marginal $b_t$, and plays
\begin{equation}
 q_t=\alpha\widehat q_t+(1-\alpha)q^c,
 \qquad x_t=\alpha b_t+(1-\alpha)x^c.
 \label{main:mixture}
\end{equation}
The weight on the learner, $\alpha\in(0,1]$, controls how
far the algorithm departs from the baseline. A
small weight retains baseline observations; at $\alpha=1$, the
learner samples directly from its learning distribution.
The design must decide both when to
increase this weight and how much empirical
regret to allow before restarting more cautiously.

\paragraph{Two tests, two doublings.}
A \emph{stage} maintains a budget $R_s$ for the
base learner's empirical regret. Within that stage, a
\emph{phase} keeps $\alpha$ fixed and runs a fresh
base learner. If the learner exceeds its threshold,
a new stage doubles the threshold and resets
the weight. Otherwise, a sufficiently large empirical loss
advantage over the baseline permits the next phase
to increase the weight. These comparisons use estimated
losses: their deterministic accounting will yield expected
baseline protection after applying conditional unbiasedness. The
two parameters have different roles: $R_s$ sets
the scale of learning cost the stage
allows, while $\alpha$ sets how much of
that cost affects the actions actually played.

\paragraph{The empirical tests.}
On round $t$, set $\widehat\ell_{t,i}=X_{t,i}\ell_{t,i}/x_{t,i}$, with value zero
on unselected arms. Both base learners ensure $x_{t,i}>0$.
Dividing by the selection probability compensates for observing
an arm only on some rounds, so these
estimates can be used to compare all actions.
For a phase beginning on round $u$, maintain
\[
 S_{u:t}=\sum_{j=u}^t\widehat\ell_j,\qquad
 A_{u:t}=\sum_{j=u}^t\langle b_j,\widehat\ell_j\rangle,\qquad
 v_{u:t}=\min_{X\in\cX}\langle X,S_{u:t}\rangle,
\]
and define
\begin{equation}
 U_{u:t}=[A_{u:t}-v_{u:t}]_+,\qquad
 G^c_{u:t}=\langle x^c,S_{u:t}\rangle-v_{u:t}.
 \label{main:certificates}
\end{equation}
Here $[a]_+=\max\{a,0\}$. The first statistic measures the learner's
positive empirical regret; the second measures the baseline's
empirical excess loss relative to the same best
fixed action. Since $x^c\in\cP$, $G^c_{u:t}\geq0$. The two
statistics share one linear minimization. Their difference
has a direct interpretation: when $G^c_{u:t}$ is
large but $U_{u:t}$ is small, the learning
distribution has accumulated less estimated loss than
the baseline. This is the evidence used
to increase the weight on the learner.

\begin{samepage}
We now turn these statistics into update rules.
Stage $s$ starts cautiously, at phase $k=1$, with
weight $L/R_s$ on the learner. Within that stage,
the rule $\alpha=\min\{L2^{k-1}/R_s,1\}$ doubles the weight whenever the
phase index increases. Initially $R_1=2L$, so the two
distributions receive equal weight. After each round, we
first check whether the budget was exceeded, then
whether the evidence justifies increasing the learning weight:
\begin{list}{\labelitemi}{%
 \setlength{\leftmargin}{1.5em}%
 \setlength{\topsep}{3pt}%
 \setlength{\partopsep}{0pt}%
 \setlength{\itemsep}{0pt}%
 \setlength{\parsep}{0pt}%
}
 \item \textbf{Double the budget.} If $U_{u:t}>R_s$, start stage
 $s+1$ with $R_{s+1}=2R_s$ and $k=1$.
 \item \textbf{Double the weight.} Otherwise, if $G^c_{u:t}>4R_s$ and
 $\alpha<1$, keep $R_s$ and start phase $k+1$.
 \item \textbf{Continue.} Otherwise, update the current base learner.
\end{list}
\end{samepage}
The threshold test takes priority. Every restart
clears the phase statistics and initializes a
fresh learner whose learning rate uses the
number of rounds since that restart. Consequently,
\begin{equation}
 R_s=2^sL,\qquad 1\leq k\leq s+1,\qquad
 \alpha_{s,k}=2^{k-1-s},\qquad \alpha_{s,k}R_s=L2^{k-1}.
 \label{main:phase-schedule}
\end{equation}
A mixed phase therefore has $\alpha\leq1/2$; phase
$s+1$ is pure. The weight increases within
a stage, but a stage restart can
return a pure learner to mixed
sampling. Algorithm~\ref{alg:phased-sb} specifies \Duet{}.

\begin{algorithm}[t]
\caption{\Duet: \textbf{D}oubling \textbf{U}nder \textbf{E}mpirical \textbf{T}ests}
\label{alg:phased-sb}
\small
\algrenewcommand{\algorithmicrequire}{\textbf{Input:}}
\algrenewcommand{\algorithmicensure}{\textbf{Output:}}
\algrenewcommand{\algorithmicindent}{1.2em}
\begin{algorithmic}[1]
\Require action set $\cX$ with $|\cX|\geq2$, baseline $q^c$ with full coverage, action-loss bound $L$
\Ensure feasible actions $X_1,X_2,\ldots$
\Procedure{StartPhase}{$r$}
    \WrappedState{$u\gets r$; $\alpha\gets\min\{L2^{k-1}/R_s,1\}$; $S\gets0_d$; $A\gets0$}
    \WrappedState{Initialize $B$ on $\widetilde\cX$ via~\eqref{main:entropy-update} if \mbox{$\alpha<1$};\newline otherwise on $\widetilde\cX_0$ via~\eqref{main:hybrid-update} if \mbox{$\alpha=1$}}
\EndProcedure
\WrappedState{$x^c\gets\sum_{X\in\cX}q^c(X)X$; $\delta\gets\min_{i\in[d]}x_i^c$}
\WrappedState{$m\gets\max_{X\in\cX}\|X\|_1$; $\rho\gets\delta/2$; $\Lambda\gets\log(e|\cX|)$}
\WrappedState{Form $\widetilde\cX$ by zero-loss padding; form $\widetilde\cX_0$ by deleting its constant coordinates}
\WrappedState{$s\gets1$; $R_s\gets2L$; $k\gets1$; \Call{StartPhase}{$1$}}
\For{$t=1,2,\ldots$}
    \WrappedState{Query $B$ at $n=t-u+1$ for $\widehat q_t$ mapped back to $\cX$ and its marginal $b_t$}
    \WrappedState{$q_t\gets\alpha\widehat q_t+(1-\alpha)q^c$; $x_t\gets\alpha b_t+(1-\alpha)x^c$}
    \WrappedState{Sample $X_t\sim q_t$; observe $\ell_{t,i}$ for the selected arms $i$}
    \WrappedState{Set $\widehat\ell_{t,i}\gets\ell_{t,i}/x_{t,i}$ on selected arms and $0$ otherwise}
    \WrappedState{$S\gets S+\widehat\ell_t$; $A\gets A+\langle b_t,\widehat\ell_t\rangle$}
    \WrappedState{$v\gets\min_{X\in\cX}\langle X,S\rangle$;\newline
        $U_{u:t}\gets[A-v]_+$; $G^c_{u:t}\gets\langle x^c,S\rangle-v$}
    \If{$U_{u:t}>R_s$}
        \WrappedState{$s\gets s+1$; $R_s\gets2R_{s-1}$; $k\gets1$; \mbox{\Call{StartPhase}{$t+1$}}}
    \ElsIf{$G^c_{u:t}>4R_s$ and $\alpha<1$}
        \WrappedState{$k\gets k+1$; \mbox{\Call{StartPhase}{$t+1$}}}
    \Else
        \WrappedState{Update $B$ with $\widehat L_t(\widetilde X)=\langle\widetilde X,(\widehat\ell_t,0_m)\rangle$\newline if \mbox{$\alpha<1$}; otherwise use the shifted estimate in~\eqref{main:hybrid-update}}
    \EndIf
\EndFor
\end{algorithmic}
\end{algorithm}

\paragraph{Instantiating the two sampling\newline regimes.}
\label{main:bases}
\Duet{} uses entropy-FTRL in mixed phases and hybrid-FTRL
in pure phases. Both use follow-the-regularized-leader (FTRL): each update
balances cumulative estimated loss against a regularizer that
encourages exploration.\\\medskip

The sampling regime determines which regularizer
supplies the bounds needed by the restart tests.
The restart tests need more than an ordinary
regret guarantee: they need to control positive empirical
regret in~\eqref{main:certificates} when a phase ends. That endpoint
is chosen from the observations, so a guarantee
for a predetermined duration alone does not suffice.\\\medskip

We therefore derive bounds conditional on the phase's
starting history, under its actual sampling distribution. Throughout,
the outer tests use the original, nonnegative estimates
$\widehat\ell_t$, including when the learner's update uses shifted estimates.

\begin{remark}[Why two FTRL learners]
\label{rem:phase-base-choice}
Mixed phases retain uniform observation coverage: $x_{t,i}\geq(1-\alpha)x_i^c\geq\delta/2$.
Entropy-FTRL exploits this bound. In pure phases,
the observation marginal is $b_t$, whose coordinates
can approach zero; hybrid-FTRL instead supplies a
nonnegative, nondecreasing majorant of empirical regret.
\end{remark}

\paragraph{Action representations.}
Before writing the updates, we make every action
have the same size. Append $m$ zero-loss coordinates
to each $X$, setting the first $m-\|X\|_1$ to one
and the rest to zero. The resulting action
$\widetilde X$ has exactly $m$ ones in dimension $D=d+m\leq2d$;
write $\widetilde\cX:=\{\widetilde X:X\in\cX\}$. Both learners use this representation. The
extra coordinates only simplify their updates and analysis:
removing them recovers the original action, with exactly
the same loss.

\paragraph{Mixed phases.}
In a mixed phase, the baseline supplies observations:
every real arm is observed with probability at
least $\rho=\delta/2$. This coverage lets entropy-FTRL compare actions
through their estimated losses. At local round $n$,
write $\widehat L_n(\widetilde X)=\langle\widetilde X,(\widehat\ell_{u+n-1},0_m)\rangle$ and $\Lambda=\log(eN)$. The update below
puts more weight on actions with smaller cumulative
estimates; its learning rate uses the elapsed phase
length, so it needs no prescribed stopping time:
\begin{equation}
 \widetilde q_n(\widetilde X)=
 \frac{\exp\{-\eta_n\sum_{r<n}\widehat L_r(\widetilde X)\}}
 {\sum_{\widetilde Y\in\widetilde\cX}\exp\{-\eta_n\sum_{r<n}\widehat L_r(\widetilde Y)\}},
 \qquad \eta_n=\frac{\sqrt{\rho\Lambda}}{m\sqrt n}.
 \tag{Entropy-FTRL}
 \label{main:entropy-update}
\end{equation}
Return $\widehat q_n(X)=\widetilde q_n(\widetilde X)$. Zero-loss padding leaves
these probabilities unchanged (\hyperref[par:entropy-padding]{Appendix~\ref*{sec:base-learner}}).

\paragraph{Pure phases.}
At weight one, the baseline no longer guarantees
a minimum observation probability. Hybrid-FTRL handles this change
in feedback. First,
remove coordinates fixed across $\widetilde\cX$: they offer no choice
and need not be optimized. This gives $\widetilde\cX_0\subseteq\{0,1\}^{D_{\mathrm{var}}}$,
with $m_0\leq m$ ones per action and $D_{\mathrm{var}}\leq D$.
Apply the same padding and removal to the
original loss vector to obtain $g_n$. Over feasible
marginals $\mathcal Q=\conv(\widetilde\cX_0)$, we use the hybrid regularizer of
\citet{mainbib:zimmert2019beating}:
\begin{align}
 \Psi(z)&=\sum_i[-\sqrt{z_i}+(1-z_i)\log(1-z_i)],\notag\\
 z_n&=\arg\min_{z\in\mathcal Q}
 \left\{\left\langle z,\sum_{r<n}\widetilde\ell_r\right\rangle
                 +\sqrt n\,\Psi(z)\right\},
 \qquad
 \widetilde\ell_{n,i}=\frac{X^0_{n,i}g_{n,i}}{z_{n,i}}-1.
 \tag{Hybrid-FTRL}
 \label{main:hybrid-update}
\end{align}
Write $\widehat g_{n,i}=X^0_{n,i}g_{n,i}/z_{n,i}$. Sample a transformed action $X_n^0$ with
marginal $z_n$, using a fixed measurable convex-decomposition rule,
and map it back to $\cX$. On retained
coordinates, $0<z_{n,i}<1$ (cf.~Lemma~\ref{lem:pure-interior}). Restoring fixed original
coordinates ensures positive observation probability for every original
arm.

We now check that these transformations preserve the
comparisons used by the tests. Fixed coordinates contribute
equally to every action, and
subtracting one from each retained estimate subtracts exactly
$m_0$ from every transformed action's estimated loss. These
common offsets cancel when comparing actions. Thus the
stochastic gaps and the outer tests' empirical comparisons
are preserved. Remark~\ref{rem:pure-representation} gives the identity that connects
the shifted update to the original empirical-regret certificate.

These updates specify the learners used by \Duet{}.
Mixed updates enumerate the $N$ padded actions, whereas
pure updates require convex optimization and decomposition; the
outer tests minimize linear losses over $\cX$. The
guarantees below concern regret, and do not assert
an efficient oracle implementation. Neither update needs the
horizon or the stochastic gap: their learning rates
use elapsed phase time and known problem parameters.

\FloatBarrier
\subsection{Main guarantees}
\label{main:guarantees}

We can now put the pieces together. The
tests decide when to trust the learner more
and when to enlarge its budget; the two
FTRL updates supply the regret bounds that justify
those decisions. The following theorem gives the three
guarantees for this single procedure. Throughout, $C$ is
a universal constant, independent of the action set,
baseline, and horizon; its value may change between
bounds.
\newpage
\begin{samepage}
\begin{theorem}[Joint expected-regret guarantees]
\label{thm:main-result}
Suppose $N\geq2$ and Assumptions~\ref{ass:non-anticipation} and~\ref{ass:coverage} hold. Run \Duet{}
(Algorithm~\ref{alg:phased-sb}) with the mixed entropy-FTRL and pure hybrid-FTRL
updates above. For each deterministic $T\geq1$, let
\(
 J_T=3+\left\lceil\log_2\frac{2mT}{L\delta}\right\rceil.
\)
Then:
\begin{enumerate}
 \renewcommand{\theenumi}{\roman{enumi}}
 \renewcommand{\labelenumi}{\textup{(\theenumi)}}
 \item \label{main:safety-bound}\textbf{Expected baseline protection.}
 {\small$\mathcal R_T(x^c)\leq2LJ_T$}.
 \item \label{main:adversarial-bound}\textbf{Adversarial learning.}
 {\small$\sup_{x\in\cP}\mathcal R_T(x)
 \leq CmJ_T^3\sqrt{T\log(eN)/\delta}+CLJ_T^2$}.
 \item \label{main:stochastic-bound}\textbf{Stochastic adaptivity.} Under~\ref{ass:stochastic} and~\ref{ass:gap},
 {\small$\sup_{x\in\cP}\mathcal R_T(x)
 \leq C\frac{m^2}{\delta\Delta}J_T^3
 [\log(eN)+\log(eT)]+CLJ_T^2$}.
\end{enumerate}
The algorithm takes neither $T$, $\Delta$, nor
the environment type as input. For $T=0$
or $N=1$, all relevant regrets are zero.
\end{theorem}
\end{samepage}

The complete proof appears in Appendix~\ref{sec:regret-phased} (cf.~Corollary~\ref{cor:complete-regret-bounds}).
Below, Section~\ref{main:proof} explains its components: pathwise accounting (see
Appendix~\ref{sec:regret-phased}); mixed-phase certificates (see Appendix~\ref{sec:base-empirical-regret}); pure-phase majorants and
self-bounding (see Appendices~\ref{sec:base-empirical-regret} and~\ref{app:hybrid-ftrl}); and conditional restarts and
summation (see Appendix~\ref{sec:stopping-and-counting}).

\begin{remark}[Reading the guarantees]
\label{rem:reading-guarantees}
\begin{enumerate}
 \renewcommand{\theenumi}{\roman{enumi}}
 \renewcommand{\labelenumi}{\textup{(\theenumi)}}
 \setlength{\itemsep}{3pt}
 \setlength{\parsep}{0pt}
 \item \textbf{What the rates mean.}
 For a fixed action set and covered baseline,
 $J_T=O(\log(eT))$. Thus regret to the baseline grows as
 $O(\log(eT))$, while regret to the best fixed action
 grows as $O(\sqrt{T/\delta}\,\log^3(eT))$ against an adversary. Under~\ref{ass:stochastic} and~\ref{ass:gap},
 the latter becomes $O(\log^4(eT)/(\delta\Delta))$: a larger gap makes
 learning easier, while smaller coverage makes it harder.
 Only this last rate requires independence and a
 unique optimum.
 \item \textbf{Why one algorithm adapts to both regimes.}
 When a phase exceeds its budget, the next
 stage's threshold is less than twice that phase's
 empirical regret. Each budget increase is therefore paid
 for by the learning cost already observed. The
 transfer bound (cf.~Lemma~\ref{lem:main-accounting-transfer}) turns bounds on the
 sum of these phase regrets into a bound
 on total learning regret.
 The same empirical quantity admits both adversarial and
 stochastic bounds, so \Duet{} needs no test to
 identify the regime. Summing over stages and phases
 introduces the additional logarithmic factors.
\end{enumerate}
\end{remark}
%
\section{Why \Duet{} works}
\label{main:proof}

This section explains how \Duet{} can trust
the learner more while preserving baseline protection. We
begin with an accounting argument: a phase that
increases the learner's weight has already earned an
estimated advantage over
the baseline. These gains pay for later exposure.
A separate calculation explains why increasing the stage
budget is affordable: each increase can be charged
to the empirical regret that triggered it. We
then use entropy-FTRL in mixed phases and hybrid-FTRL
in pure phases to control
the positive empirical maximum at a stopped endpoint.
Conditional restart bounds connect these parts and yield
the three expected-regret guarantees for this instantiation.

We begin in Section~\ref{main:accounting} by seeing how gains
from completed phases protect the baseline. Section~\ref{main:threshold-accounting} then
connects budget increases to empirical phase regret. With
this accounting in place, Section~\ref{main:empirical-maximum} explains how to
control that regret at phase endpoints determined by
the observations. Finally, Section~\ref{main:restart-summation} combines the phases to
obtain all three guarantees.

\subsection{Completed phases pay for baseline protection}
\label{main:accounting}

We first work with the estimates on one
realized run; expectations will enter only at the
end. For a phase $I$, write
$A_I(x)=\sum_{t\in I}\langle b_t-x,\widehat\ell_t\rangle$, so that
$U_I=[\max_{X\in\cX}A_I(X)]_+$. Our first task is to limit
the damage from the round that triggers a
restart. Nonnegative estimates and $x_t\geq\alpha b_t$ give
$U_{u:t}-U_{u:t-1}\leq\langle b_t,\widehat\ell_t\rangle\leq L/\alpha\leq R_s$: one
round cannot add more than the current budget.
Thus the round that first crosses a stage
threshold cannot overshoot it by more than $R_s$:
$R_s<U_I\leq2R_s$ and $R_{s+1}=2R_s$.
Every played phase, including one interrupted by the
final horizon, consequently has $U_I\leq2R_s$. Also, $U_I\leq2mT/\delta$ for
a phase played through round $T$. Thresholds and
within-stage weights double, giving at most $J_T$
stages and $J_T$ phases per stage
(cf.~Lemma~\ref{lem:threshold-crossing}).

To see how one phase pays for the
next, fix a stage and write
$r_k=\alpha_kR_s=L2^{k-1}$. This is the budget scaled by
the weight placed on the learner. When the
baseline test ends a phase, we know both
$G_I^c>4R_s$ and $U_I\leq R_s$, because the stage test
has priority. The baseline's estimated disadvantage therefore exceeds
the learner's empirical regret, producing the following gain:
\[
 \widehat B_I
 :=\sum_{t\in I}\langle x_t-x^c,\widehat\ell_t\rangle
 =\alpha_k A_I(x^c)
 \leq\alpha_k(U_I-G_I^c)<-3r_k.
\]
On the estimated losses, such a phase has
gained relative to the baseline before the weight
on the learner increases. If phase $k$ is
the last one played in the stage, its
contribution is at most $2r_k$; every preceding
phase ended through the baseline test. Since
$\sum_{j<k}r_j=r_k-L$, the entire stage contributes at most
\begin{equation}
 -3\sum_{j<k}r_j+2r_k=3L-r_k\leq2L.
 \label{main:safety-cancellation}
\end{equation}
This also covers a stage interrupted at the
final horizon. Summing over stages bounds the estimated
baseline regret by $2LJ_T$ (cf.~Lemma~\ref{lem:baseline-one-stage}). Converting
this statement to expected regret is the final step
of the proof. The cancellation explains the
priority of the tests: increasing the weight
on the learner requires a completed phase
whose empirical learning cost still fits the
current threshold. That estimated gain then offsets
the larger exposure in the next phase.

\subsection{Bounding stage thresholds by empirical phase regret}
\label{main:threshold-accounting}

Protecting the baseline is only the first task;
we must also compete with any fixed comparator.
Let $\mathcal I$ contain the phase intervals played through
$T$, and set $U_T=\sum_{I\in\mathcal I}U_I$. There is now an
extra cost: we must account for the baseline's
loss relative to that comparator. We bound this
cost by the stage threshold, then charge each
threshold increase to the
phase whose empirical regret made that increase necessary.

\begingroup
\setlength{\abovedisplayskip}{3pt}
\setlength{\belowdisplayskip}{3pt}
\setlength{\abovedisplayshortskip}{0pt}
\setlength{\belowdisplayshortskip}{3pt}
\begin{lemma}[From empirical phase regret to estimated total regret]
\label{lem:main-accounting-transfer}
\par\nobreak For every realized run through a
deterministic horizon $T\geq1$ and every $x\in\cP$,
\begin{equation}
 \sum_{t=1}^T\langle x_t-x,\widehat\ell_t\rangle
 \leq(1+8J_T)U_T+L(J_T^2+8J_T).
 \label{main:master-transfer}
\end{equation}
\end{lemma}

\begin{proof}
On a phase $I$, the mixture splits estimated
regret into the learner contribution $\alpha A_I(x)\leq U_I$ and a
weighted baseline contribution. In a mixed phase, $G_I^c$
is at most $4R_s$ before the final round.
The final round adds at most $L$ to
the weighted baseline contribution, because $(1-\alpha)\langle x^c,\widehat\ell_t\rangle \leq\langle X_t,\ell_t\rangle\leq L$. In a
pure phase this contribution is zero. Hence
\begin{equation}
 \sum_{t\in I}\langle x_t-x,\widehat\ell_t\rangle
 \leq U_I+4R_{s(I)}+L.
 \label{main:interval-transfer}
\end{equation}
For each stage $s>1$, let $I_{s-1}^{\times}$ denote
the final phase of stage $s-1$, whose
empirical regret exceeds $R_{s-1}$ and triggers
the start of stage $s$. Thus,
\[
 R_s=2R_{s-1}<2U_{I_{s-1}^{\times}}.
\]
In the following, the first sum ranges over
all stages played through round $T$. For
each phase interval $I\in\mathcal I$, let $s(I)$ denote
the index of the stage containing $I$.
Since each stage contains at most $J_T$
phases, its threshold appears at most $J_T$
times in the sum over phase intervals.
\[
 \sum_sR_s\leq2L+2U_T,
 \qquad
 \sum_{I\in\mathcal I}R_{s(I)}\leq J_T(2L+2U_T).
\]
Summing~\eqref{main:interval-transfer} over at most $J_T^2$ phases proves \eqref{main:master-transfer}.
\end{proof}
\endgroup

The threshold terms have thus been absorbed into
the same quantity $U_T$ that measures the learners'
empirical regret. The remaining task is to bound
its expectation despite the data-dependent phase endpoints. The
factor $J_T$ in this bound arises because
each stage threshold appears once for every
phase in that stage, and each stage
contains at most $J_T$ phases. The subsequent
summation of phase certificates accounts for the
other factors in the final learning bounds.

\subsection{Stopped phases require control before taking expectation}
\label{main:empirical-maximum}
\label{main:certificates-proof}

We have reduced the problem to phase regret,
but these phases end when the observations tell
them to. The same observations also select the
empirically best action. An expected bound for each
fixed comparator, of the form $\max_X\E A_I(X)$, does
not control these selections. We therefore need a
bound on $\E U_I$, keeping both the empirical
maximum and the selected endpoint inside the expectation.
An empirical leader can benefit from favorable estimation
fluctuations, and a phase may end precisely when
such fluctuations cross a threshold. The conditional certificate
must also remain available when earlier observations determine
the next start.

\begin{lemma}[Conditional stopped-phase certificates]
\label{lem:main-stopped-certificates}
Fix a possible phase label and a deterministic
horizon $T$. Let $\sigma\leq T$ be the after-round time
at which the phase is declared to start
on round $\sigma+1$, with $\sigma=T$ if it is
not reached. Then $\sigma$ is a stopping time
and its reach event $A=\{\sigma<T\}$ belongs to $\cF_\sigma$.
Number rounds after the start locally, and let
$\tau\leq H$ be the stopped phase length, capped by
a deterministic integer $H\geq1$. Set $U_{1:\tau}=0$ off $A$.
Use the specified fresh FTRL learner and, in
a mixed phase, its fixed weight $\alpha\in(0,1/2]$. Put
$\rho=\delta/2$, $\Lambda=\log(eN)$, and $D=d+m$. Under the corresponding
assumptions of Theorem~\ref{thm:main-result}, almost surely on $A$,
\begin{equation}
 \E[U_{1:\tau}\mid\cF_\sigma]\leq C
 \begin{cases}
  m\sqrt{H\Lambda/\rho},
       &\text{mixed, adversarial},\\
  m^2\Lambda/(\rho\Delta),
       &\text{mixed, stochastic},\\
  \sqrt{mDH}+m,
       &\text{pure, adversarial},\\
  \dfrac{mD}{\Delta}\log(eH)+\dfrac{m^2}{\Delta}+m,
       &\text{pure, stochastic}.
 \end{cases}
 \label{main:four-certificates}
\end{equation}
\end{lemma}

The dependence on $H$ explains the eventual learning
rates: arbitrary losses give square-root bounds, whereas stochastic
losses give a bound independent of $H$ in
mixed phases and only logarithmic growth in
pure phases. The cap $H$ is used
only in the analysis. These bounds combine the
single-phase analysis with its conditional extension
(cf.~Lemmas~\ref{lem:one-phase-empirical-regret} and~\ref{lem:conditional-restart}). The two phase types obtain
them differently; throughout the explanation below, time
is local to a single fresh run.

\paragraph{Mixed phases: coverage controls fluctuations.}
In a mixed phase, the baseline keeps every
arm observable: coverage gives $x_{t,i}\geq\rho$. The resulting control
of estimation variance, together with the entropy-FTRL inequality,
gives the adversarial certificate. We can stop this
calculation at the test because whether round $t$
is played is known beforehand: $\one\{t\leq\tau\}\in\cF_{t-1}$.
For the stochastic improvement, we next
study a continuation used only in the analysis.
For stochastic losses, continue the learner to $H$
rounds without restarting and keep $\alpha$ fixed. Concentration
of estimated gaps and exponential weighting give
\begingroup
\setlength{\abovedisplayskip}{0pt}
\setlength{\belowdisplayskip}{0pt}
\setlength{\abovedisplayshortskip}{0pt}
\setlength{\belowdisplayshortskip}{0pt}
\begin{equation}
 R_H^{\mathrm{mix}}
 :=\sum_{t=1}^H\E\langle b_t-X^\star,\bar\ell\rangle
 \leq C\frac{m^2\Lambda}{\rho\Delta}.
 \label{main:mixed-ordinary}
\end{equation}
\endgroup
This bounds the learning distribution's regret;
the baseline's contribution was handled above.

The empirical maximum needs an additional argument
because the mean-optimal action need not minimize
estimated cumulative loss. We must also control
how much a competing action can appear
better due to random fluctuations in the
loss estimates. Define $D_n(X)=\sum_{t=1}^n\langle X^\star-X,\widehat\ell_t\rangle$ and split
\begin{equation}
 U_{1:\tau}\leq[A_{1:\tau}(X^\star)]_+
                         +\max_X[D_\tau(X)]_+.
 \label{main:maximum-split}
\end{equation}
A stopped martingale variance bound controls the first
term in expectation by $Cm^2\Lambda/(\rho\Delta)$. For the second,
each suboptimal action has negative conditional mean
increments in $D_n(X)$, so a positive value represents
a misleading empirical comparison. In fact, $\exp\{\lambda D_n(X)\}$ is
a nonnegative supermartingale for $\lambda=\rho\Delta/(4m^2)$. Ville's inequality, a
union bound over actions, and tail integration give\vspace{-1em}
\[
 \E\max_{X\in\cX}\sup_{n\leq H}[D_n(X)]_+
 \leq\frac{4m^2\Lambda}{\rho\Delta}.
\]
This bounds how much the estimated cumulative
loss of $X^\star$ can exceed that of another
action at any phase endpoint, completing the mixed
stochastic certificate. The two terms serve distinct purposes:
the first controls fluctuations of the learner's
loss relative to the mean-optimal action, while
the second bounds the difference between the
estimated cumulative loss of $X^\star$ and the
minimum estimated cumulative loss over all actions.

\paragraph{Pure phases: a monotone majorant survives stopping.}
Pure phases require a different argument because the
baseline no longer supplies a uniform lower bound
on observation probabilities. Before analyzing stopping,
we must check that the hybrid learner bounds
the same empirical quantity that our tests monitor.
The representation in Section~\ref{main:algorithm} was chosen for this
purpose: the loss shift cancels in comparisons, while
equal cardinality makes the rate-change terms nonnegative. The
following remark records these two useful properties.

\begin{remark}[Why the representation preserves the certificate]
\label{rem:pure-representation}
Every $z\in\mathcal Q$ and $Y\in\widetilde\cX_0$ has coordinate
sum $m_0$, so the shift cancels:
\(
 \langle z-Y,\widetilde\ell_n\rangle
 =\langle z-Y,\widehat g_n-\one\rangle
 =\langle z-Y,\widehat g_n\rangle.
\)\\
The hybrid analysis therefore controls the original $U$,
with unchanged stochastic gaps. Equal cardinality also
gives $\Psi(Y)=-m_0$ and $-\Psi(z)-m_0\geq0$, supplying nonnegative rate-change terms.
Removing fixed coordinates permits the interior stability analysis.
\end{remark}

With that representation in place, we can address
stopping. Simply analyzing a longer run is insufficient:
its empirical maximum could fall after the test
has already ended the phase. We need something
that remembers the cost accumulated so far. The
hybrid analysis supplies such an upper bound, one
that can only increase as more rounds are
added. The FTRL decomposition has nonnegative stability increments
and nonnegative rate-change increments $(\sqrt n-\sqrt{n-1})[-\Psi(z_n)-m_0]$. This decomposition
gives a nonnegative, nondecreasing majorant $C_n$ with
$0\leq U_{1:n}\leq C_n$.

We now use this monotonicity to replace the
random endpoint by a deterministic one. For analysis
only, continue the run to $H$ with fresh
learner randomness, leaving everything through $\tau$ unchanged. After
$\tau$, supply fresh i.i.d. losses in the stochastic
case and zero losses in the adversarial case.
Since the upper bound never decreases,
$U_{1:\tau}\leq C_\tau=C_\tau^{\mathrm{ext}}\leq C_H^{\mathrm{ext}}$. This continuation is not
a step performed by \Duet{}.
Write $C_H$ for this terminal majorant. Its expectation
satisfies $\E C_H\leq C(\sqrt{mDH}+m)$ adversarially
(cf.~Theorem~\ref{thm:hybrid-ftrl}). Monotonicity permits comparison
with a deterministic-length run even when the
empirical maximum itself decreases on later rounds.

\begingroup
\setlength{\abovedisplayskip}{0pt}
\setlength{\belowdisplayskip}{0pt}
\setlength{\abovedisplayshortskip}{0pt}
\setlength{\belowdisplayshortskip}{0pt}
For the standalone continued stochastic run, let
$\pi_t=\Prb(X_t\neq X^\star)$ be the unconditional probability of a
suboptimal action, averaging over losses and learner
randomness, and let $R_H^{\mathrm{pure}}$ be its
expected regret. The same analysis gives
\[
 R_H^{\mathrm{pure}}\leq\E C_H\leq B_H,
 \qquad
 B_H\leq Cm+C\sum_{t=1}^H
 \frac{\sqrt{mD\pi_t}+m\pi_t\log(e/\pi_t)}{\sqrt t}.
\]
Each suboptimal play contributes at least $\Delta$
to expected regret, so $R_H^{\mathrm{pure}}\geq\Delta\sum_t\pi_t$.
Since $B_H-R_H^{\mathrm{pure}}\geq0$, we have
\vspace{-1.25em}
\[
\E C_H
\leq B_H+(B_H-R_H^{\mathrm{pure}})
\leq 2B_H-\Delta\sum_t\pi_t.
\]
Using the bound above for $B_H$, we
maximize the right-hand side over each $\pi_t\in[0,1]$. This
removes the unknown probabilities and yields
\[
 \E C_H\leq2B_H-R_H^{\mathrm{pure}}
 \leq C\left[\frac{mD}{\Delta}\log(eH)
                  +\frac{m^2}{\Delta}+m\right].
\]
Thus self-bounding controls the majorant of
the stopped empirical maximum, which is the
quantity required by the outer accounting.
\par
\endgroup

\subsection{Conditional restarts turn phase bounds into three guarantees}
\label{main:restart-summation}

We now return from one isolated phase to
the actual sequence of restarts. A phase starts
because of earlier observations, so its bound must
remain valid after those observations are fixed. This
is why we need conditional phase bounds
(cf.~Lemma~\ref{lem:main-stopped-certificates}). Under Assumption~\ref{ass:stochastic}, future
losses and fresh learner seeds retain their independent
product law given the start history; under Assumption~\ref{ass:non-anticipation},
the required conditional sampling identities still hold.
Neither argument conditions on the future event that
ends the phase. The required conditional bounds hold
almost surely on each reach event (cf.~Lemma~\ref{lem:conditional-restart}).

To add the phase bounds, we use a
fixed list of possible phases rather than first
selecting the phases that happened to occur. If we apply~\eqref{main:four-certificates}
with $H=T$ to each label $(s,k)\in[J_T]^2$, assigning zero
when that label is unreached. Conditioning at its
start and then taking expectation bounds each contribution.
We can now sum over this deterministic list,
even though the played intervals are random.
For each label, its reach event is
already measurable at its start, so its expected
contribution is bounded by its reach probability times
the corresponding certificate. Bounding those probabilities by
one gives a sum over at most
$J_T^2$ labels without conditioning on which phases
eventually occur (cf.~Corollary~\ref{cor:total-phase-empirical-regret}). Using $D\leq2d$ and $d\delta\leq m$, we obtain
$\E U_T\leq CmJ_T^2\sqrt{T\Lambda/\delta}$ adversarially and
$\E U_T\leq C\frac{m^2}{\delta\Delta}J_T^2[\Lambda+\log(eT)]$ stochastically.

Only one step remains: returning from estimated losses
to the regret stated in the theorem. Fix
a comparator $x$ before taking expectations. Conditional unbiasedness
then identifies $\E\sum_{t=1}^T\langle x_t-x,\widehat\ell_t\rangle$ with $\mathcal R_T(x)$.
This conversion concerns a fixed comparator; the data-dependent
empirical maximum has already been controlled inside the
expectation by the phase certificates.

The stage bound~\eqref{main:safety-cancellation} therefore
proves Theorem~\ref{thm:main-result}\textup{(\ref{main:safety-bound})}. Taking expectations in~\eqref{main:master-transfer}, inserting these bounds, and then
taking the supremum over fixed comparators proves
parts~\textup{(\ref{main:adversarial-bound})} and~\textup{(\ref{main:stochastic-bound})} of the theorem.
The proof uses stopped phases internally while evaluating
these final guarantees at each deterministic horizon $T$.

%
\section{Discussion}
\label{main:discussion}

\Duet{} organizes baseline protection around two observable tests:
one doubles the learning weight when empirical gains
justify it; the other doubles the empirical-regret budget
when learning exceeds that budget. With entropy-FTRL in
mixed phases and hybrid-FTRL in pure phases, this
instantiation combines logarithmic expected baseline regret with adversarial robustness
and stochastic adaptivity. The stopped empirical-maximum certificates connect
these learners to the mixing scheme and remain
valid after data-dependent restarts. Both the tests and
the learner updates operate without knowing the horizon,
the stochastic gap, or the environment type.

The guarantees require a baseline with full coverage,
and a small coverage probability $\delta$ worsens the
displayed bounds. The stochastic rate additionally assumes i.i.d.
loss vectors independent of learner randomization and a
unique optimal action. Baseline regret is logarithmic
in $T$ in expectation at each fixed
horizon; no pathwise baseline guarantee is asserted.
The present finite-action implementation also uses enumeration
and convex decomposition. Reducing these computational costs,
improving the coverage dependence, and sharpening the
logarithmic factors remain directions for further work.

\paragraph{Future work.}
A further direction is to sharpen the bounds
when actions reveal many arms. For general feasible
families, the relevant benchmark is the $O(\sqrt{mdT})$ adversarial
bound of \citet[Theorem~1]{mainbib:zimmert2019beating}. Their Theorem~3 gives a sharper,
two-regime bound for the complete family of subsets
of exactly $m$ arms. In our model, $m$
is only the largest feasible action size: $m=d$
does not imply full-information feedback when the learner
selects a smaller action. Zero-loss padding also does
not create the complete fixed-size family.

Obtaining such a refinement for \Duet{} calls for
coordinated changes to both learners: adjusting the mixed-phase
learning rate and the pure-phase regularizer, deriving new
adversarial and stochastic stopped-phase bounds for each, and
propagating these bounds through the conditional restart analysis.
A sharper pure-phase bound alone would leave the
mixed phase's current dependence on $m$. Achieving the
refined rates while retaining baseline protection, stochastic adaptivity,
and horizon-free operation remains open.
%
\par
\begingroup
\interlinepenalty=10000
\paragraph{AI Disclosure.}
The student authors and their supervisor began working
on this problem well before the release of
GPT-5.6 models. Feedback from earlier reviews played a
key role in clarifying that pathwise arguments could
be used in the proof. From that point
onward, AI tools were used solely to improve
the manuscript's wording and exposition.
\par
\endgroup
\label{main:body-end}
\clearpage

\bibliography{references}
\clearpage
\appendix
\section{Proof dependency diagram}
\label{app:proof-diagram}

Read Figure~\ref{fig:proof-dependencies} from top to bottom. The common
inputs are \Duet{} (Algorithm~\ref{alg:phased-sb}), Assumptions~\ref{ass:non-anticipation} and~\ref{ass:coverage}, and unbiasedness (Lemma~\ref{lem:importance-unbiasedness}).
Only the stochastic bounds additionally require Assumptions~\ref{ass:stochastic} and~\ref{ass:gap}.

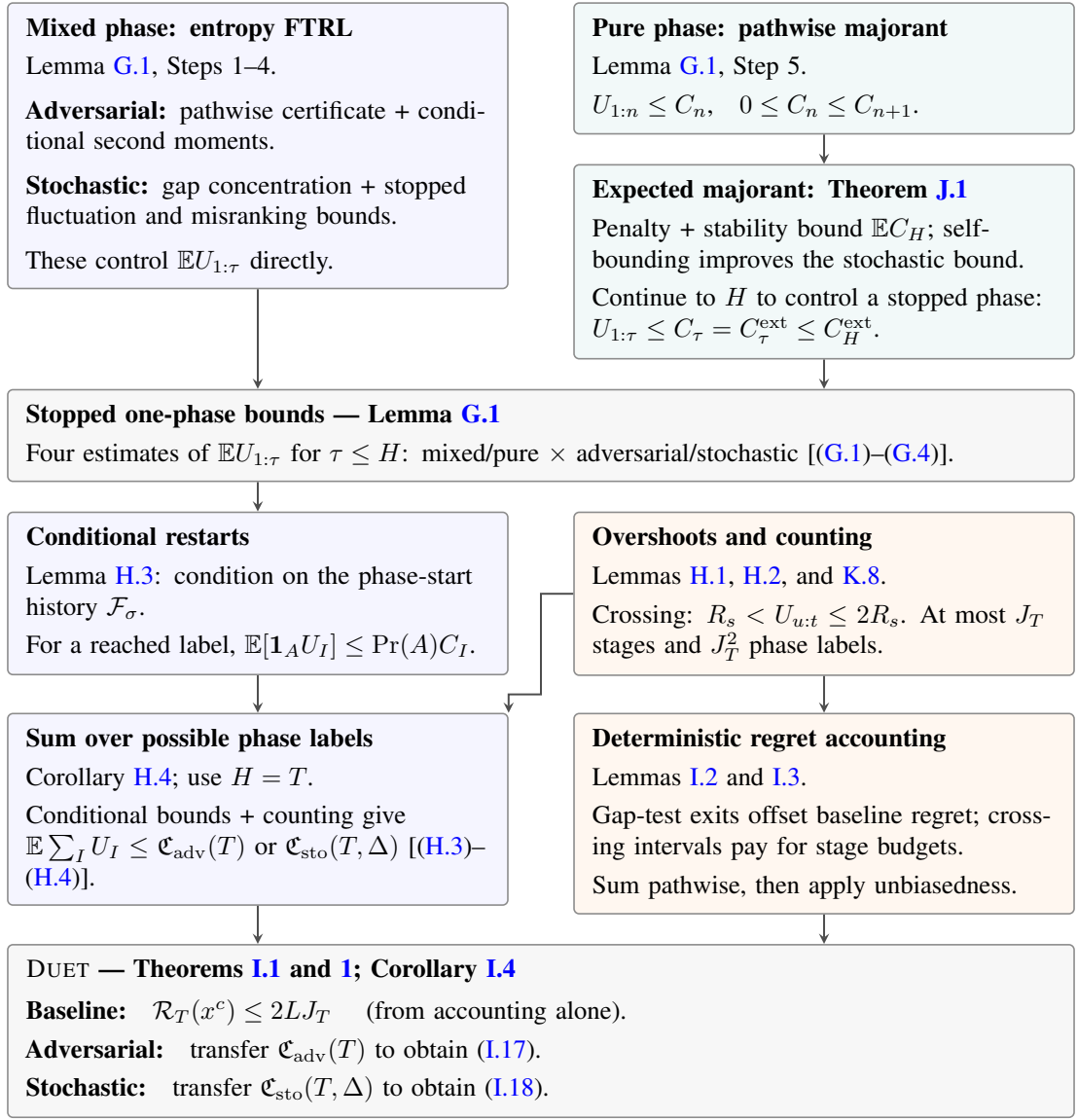
\begin{figure}[!ht]
\centering
\begin{tikzpicture}[
  x=1cm,y=1cm,
  every node/.style={font=\fontsize{10}{12}\selectfont,align=left},
  proofbox/.style={draw=black!45,line width=.6pt,rounded corners=2pt,
    inner xsep=2.5mm,inner ysep=2mm,text width=6.3cm,anchor=north},
  mixedbox/.style={proofbox,fill=blue!4},
  purebox/.style={proofbox,fill=teal!5},
  accountbox/.style={proofbox,fill=orange!6},
  fullbox/.style={proofbox,text width=14cm,fill=black!3},
  flow/.style={->,>=stealth,line width=.75pt,draw=black!70}
]
  \node[mixedbox,minimum height=3.65cm] (mixed) at (3.4,0) {
    \textbf{Mixed phase: entropy FTRL}\\[3pt]
    Lemma~\ref{lem:one-phase-empirical-regret}, Steps 1--4.\\[5pt]
    \textbf{Adversarial:} pathwise certificate
    + conditional second moments.\\[5pt]
    \textbf{Stochastic:} gap concentration
    + stopped fluctuation and misranking bounds.\\[5pt]
    These control $\E U_{1:\tau}$ directly.
  };
  \node[purebox] (purepath) at (11.1,0) {
    \textbf{Pure phase: pathwise majorant}\\[3pt]
    Lemma~\ref{lem:one-phase-empirical-regret}, Step 5.\\[3pt]
    $U_{1:n}\leq C_n$,\quad $0\leq C_n\leq C_{n+1}$.
  };
  \node[purebox] (puremean) at ([yshift=-4mm]purepath.south) {
    \textbf{Expected majorant: Theorem~\ref{thm:hybrid-ftrl}}\\[3pt]
    Penalty + stability bound $\E C_H$;
    self-bounding improves the stochastic bound.\\[3pt]
    Continue to $H$ to control a stopped phase:\\
    $U_{1:\tau}\leq C_\tau=C_\tau^{\mathrm{ext}}\leq C_H^{\mathrm{ext}}$.
  };
  \draw[flow] (purepath.south) -- (puremean.north);
  \path let \p1=(mixed.south), \p2=(puremean.south) in
    coordinate (localbottom) at (7.25,{min(\y1,\y2)});

  \node[fullbox] (onephase) at ([yshift=-4mm]localbottom) {
    \textbf{Stopped one-phase bounds --- Lemma~\ref{lem:one-phase-empirical-regret}}\\[3pt]
    Four estimates of $\E U_{1:\tau}$ for $\tau\leq H$:
    mixed/pure $\times$ adversarial/stochastic
    [\eqref{eq:mixed-adv-phase}--\eqref{eq:pure-sto-phase}].
  };
  \draw[flow] (mixed.south) -- (mixed.south |- onephase.north);
  \draw[flow] (puremean.south) -- (puremean.south |- onephase.north);

  \node[mixedbox] (restart) at ([yshift=-4mm]onephase.south -| mixed.center) {
    \textbf{Conditional restarts}\\[3pt]
    Lemma~\ref{lem:conditional-restart}: condition on the
    phase-start history $\mathcal F_\sigma$.\\[3pt]
    For a reached label,
    $\E[\mathbf1_AU_I]\leq\Pr(A)C_I$.
  };
  \node[accountbox] (count) at ([yshift=-4mm]onephase.south -| puremean.center) {
    \textbf{Overshoots and counting}\\[3pt]
    Lemmas~\ref{lem:unbiased-and-overshoot},
    \ref{lem:threshold-crossing}, and~\ref{lem:running-max-increment}.\\[3pt]
    Crossing: $R_s<U_{u:t}\leq2R_s$.
    At most $J_T$ stages and $J_T^2$ phase labels.
  };
  \draw[flow] (restart.north |- onephase.south) -- (restart.north);
  \path let \p1=(restart.south), \p2=(count.south) in
    coordinate (middlebottom) at (7.25,{min(\y1,\y2)});

  \node[mixedbox] (total) at ([yshift=-5mm]middlebottom -| mixed.center) {
    \textbf{Sum over possible phase labels}\\[3pt]
    Corollary~\ref{cor:total-phase-empirical-regret}; use $H=T$.\\[3pt]
    Conditional bounds + counting give\\
    $\E\sum_I U_I\leq\mathfrak C_{\mathrm{adv}}(T)$
    or $\mathfrak C_{\mathrm{sto}}(T,\Delta)$
    [\eqref{eq:total-adv-empirical-regret}--\eqref{eq:total-sto-empirical-regret}].
  };
  \node[accountbox] (transfer) at ([yshift=-5mm]middlebottom -| puremean.center) {
    \textbf{Deterministic regret accounting}\\[3pt]
    Lemmas~\ref{lem:baseline-one-stage}
    and~\ref{lem:fixed-comparator-all-phases}.\\[3pt]
    Gap-test exits offset baseline regret;
    crossing intervals pay for stage budgets.\\[3pt]
    Sum pathwise, then apply unbiasedness.
  };
  \draw[flow] (restart.south) -- (total.north);
  \draw[flow] (count.south) -- (transfer.north);
  \draw[flow] (count.west) -- (7.25,0 |- count.west)
    -- ([yshift=2.5mm]total.north east -| middlebottom)
    -- ([yshift=2.5mm]total.north east) -- (total.north east);
  \path let \p1=(total.south), \p2=(transfer.south) in
    coordinate (accountbottom) at (7.25,{min(\y1,\y2)});

  \node[fullbox] (result) at ([yshift=-4mm]accountbottom) {
    \textbf{\Duet{} --- Theorems~\ref{thm:phasewise-to-regret}
    and~\ref{thm:main-result}; Corollary~\ref{cor:complete-regret-bounds}}\\[4pt]
    \textbf{Baseline:}\quad $\mathcal R_T(x^c)\leq2LJ_T$
    \quad (from accounting alone).\\[3pt]
    \textbf{Adversarial:}\quad transfer $\mathfrak C_{\mathrm{adv}}(T)$
    to obtain~\eqref{eq:complete-adversarial-regret}.\\[3pt]
    \textbf{Stochastic:}\quad transfer $\mathfrak C_{\mathrm{sto}}(T,\Delta)$
    to obtain~\eqref{eq:complete-stochastic-regret}.
  };
  \draw[flow] (total.south) -- (total.south |- result.north);
  \draw[flow] (transfer.south) -- (transfer.south |- result.north);
\end{tikzpicture}
\caption{Proof dependencies for \Duet{}. Arrows show the inputs
to each step, from phase
certificates through restart bounds and regret
accounting to the final guarantees.}
\label{fig:proof-dependencies}
\end{figure}

\clearpage
\renewcommand{\appendix}{}

\section{Notation and setting}
\label{sec:model}

Let $\cX\subseteq\{0,1\}^d$ be the feasible action
set on $d\geq1$ arms, with $X_i=1$
indicating selection of arm $i$. Write
\[
  m:=\max_{X\in\cX}\lVert X\rVert_1,
  \qquad
  \cP:=\conv(\cX),
  \qquad
  N:=|\cX|.
\]
Throughout this appendix, assume $N\geq2$ and $T\geq1$.

We identify a binary action with its support
$\supp(X):=\{i:X_i=1\}$ when using set notation, so $X\setminus Y:=\{i:X_i=1,\ Y_i=0\}$.
For a nonempty finite set $\mathcal A$, write
\[
  \Delta(\mathcal A):=\left\{q\in[0,1]^{\mathcal A}:\sum_{a\in\mathcal A}q(a)=1\right\}
\]
for its probability simplex and $\mathrm{Unif}_{\mathcal A}(a):=1/|\mathcal A|$ for
its uniform distribution. We omit the subscript
on $\mathrm{Unif}$ when its support is specified.

For distributions $q,p$ on the same
finite set, the relative entropy is
\[
  \KL(q\Vert p):=\sum_{a:q(a)>0}q(a)\log\frac{q(a)}{p(a)},
\]
with value $+\infty$ if some $q(a)>0=p(a)$. We use
$[r]_+:=\max\{r,0\}$, the all-one vector $\one$, the zero vector
$0_r\in\mathbb R^r$, and natural logarithms unless a base is
shown. Constants denoted by $C$ or $c$ are
positive universal constants whose values may change between
displays. We repeatedly use $\sum_{t=1}^Ht^{-1/2}\leq2\sqrt H$ and $\sum_{t=1}^H1/t\leq1+\log H$.

At the beginning of round $t$,
the learner chooses a distribution $q_t$
over $\cX$, samples $X_t\sim q_t$, incurs
\[
  \langle X_t,\ell_t\rangle,
  \qquad \ell_t\in[0,1]^d,
\]
and observes exactly those coordinates
$\ell_{t,i}$ for which $X_{t,i}=1$.

Let $L\in(0,m]$ be a known upper bound on
$\langle X,\ell_t\rangle$ for all $X,t$; $L=m$ always suffices. Let
$\xi_t$ collect the learner's fresh round-$t$ randomness in
a fixed standard Borel space, and sample $X_t$
from $q_t$ using $\xi_t$ and a fixed measurable
rule. Let $\cF_0$ be the completed sigma-field of
fixed initial information and $\cF_t$ the completion of
$\sigma(\cF_0,\ell_1,\xi_1,\ldots,\ell_t,\xi_t)$. Round-$t$ phase states, learning marginals, and sampling
laws are $\cF_{t-1}$-measurable. The pre-draw sigma-field $\cG_t$ satisfies
$\cF_{t-1}\subseteq\cG_t\subseteq\cF_t$ and includes $q_t$ but not the
fresh draw. These analysis histories may contain
unobserved losses; the learner uses only its
randomization and observed feedback. A round-$t$ vector
is predictable if it is $\cG_t$-measurable. Let
\[
  x_t:=\E[X_t\mid\cG_t]\in\cP
\]
be the played arm-selection marginal, with $x_{t,i}=\Prb(X_{t,i}=1\mid\cG_t)$.

Assumption~\ref{ass:non-anticipation} permits losses depending on
the past and $q_t$, provided
\begin{equation}
  \Prb(X_{t,i}=1\mid\cG_t,\ell_t)=x_{t,i}
  \qquad\text{for every }i\in[d].
  \label{eq:fresh-draw}
\end{equation}
Choosing $\ell_t$ before $X_t$ suffices for Equation~\eqref{eq:fresh-draw}.

The fixed baseline $q^c$ has marginal
\[
  x^c:=\E_{X\sim q^c}[X]\in\cP.
\]
For fixed $x\in\cP$, define
\[
  Z_T(x):=\sum_{t=1}^T\langle x_t-x,\ell_t\rangle,
  \qquad
  \mathcal R_T(x):=\E[Z_T(x)].
\]
The corresponding action-based pseudo-regret is
\[
  \overline{\mathcal R}_T(x)
  :=\E\!\left[\sum_{t=1}^T\langle X_t-x,\ell_t\rangle\right].
\]
Under Assumption~\ref{ass:non-anticipation}, conditioning first on $(\cG_t,\ell_t)$ gives
\begin{align*}
  \overline{\mathcal R}_T(x)
  &=\sum_{t=1}^T\E\!\left[\E[\langle X_t-x,\ell_t\rangle\mid\cG_t,\ell_t]\right]\\
  &=\sum_{t=1}^T\E\langle x_t-x,\ell_t\rangle
  =\mathcal R_T(x).
\end{align*}

The baseline-regret bound is $\mathcal R_T(x^c)\leq\widetilde O(1)$.
Define the worst-case pseudo-regret as
\[
  \mathcal E_T:=\sup_{x\in\cP}\mathcal R_T(x)
  =\max_{X\in\cX}\mathcal R_T(X).
\]

Assumption~\ref{ass:coverage} requires
\[
  \delta:=\min_{i\in[d]}x_i^c>0.
\]
This permits $q^c(X)=0$ for some actions but implies
\[
  \bigcup_{X\in\cX}\supp(X)=[d].
\]
Coverage also gives
\[
  d\delta\leq\sum_{i=1}^d x_i^c
  =\mathbb E_{X\sim q^c}\lVert X\rVert_1\leq m.
\]

%
\section{Stochastic assumption and full action set}
\label{sec:additional-assumptions}

\subsection{Stochastic assumption}
\label{sec:stochastic-assumption}

Under Assumption~\ref{ass:stochastic}, the loss-and-seed pairs $\zeta_t:=(\ell_t,\xi_t)$ satisfy
\begin{equation}
  \zeta_1,\zeta_2,\ldots\overset{\mathrm{i.i.d.}}{\sim}P\otimes\nu.
  \label{eq:stochastic-innovation}
\end{equation}
Future pairs are independent of $\cF_t$; Lemma~\ref{lem:conditional-restart}
extends this property to data-dependent restarts.
Within-round losses may be correlated. Only
the stochastic bounds require this assumption.

\subsection{Full-action convention and stochastic gap}
\label{sec:full-action-convention}

All bounds use the full action
set $\cX$, including dominated actions. Padding
and projection in Section~\ref{sec:base-learner} are bijections.

Under Assumption~\ref{ass:gap}, $X^\star$ uniquely minimizes $\langle X,\bar\ell\rangle$ over
$\cX$, where $\bar\ell:=\mathbb E_P\ell$. As in Assumption~\ref{ass:gap}, define
\[
  \Delta_X:=\langle X-X^\star,\bar\ell\rangle,
  \qquad
  \Delta:=\min_{X\in\cX\setminus\{X^\star\}}\Delta_X>0.
\]

\subsection{Necessity of coverage}
\label{sec:coverage-necessary}

Assumption~\ref{ass:coverage} is the only requirement placed on the
baseline. The following lower bound shows why coverage
is needed to combine uniformly logarithmic baseline regret
with a $\widetilde O(\sqrt T)$ learning guarantee, already in
i.i.d.\ environments. The argument uses the change-of-measure technique
behind bandit lower bounds \citep[Chapters~14--15]{mainbib:lattimore2020bandit}.

\begin{proposition}[Coverage is necessary]
\label{prop:coverage-necessary}
Let $d=2$, $\cX=\{e_1,e_2\}$, and let the baseline be $q^c=\delta_{e_1}$, so
that $x^c=(1,0)$ violates Assumption~\ref{ass:coverage} at arm~$2$. For
$\varepsilon\in(0,1/4]$, let $P^{\pm}_\varepsilon$ be the i.i.d.\ environment
in which $\ell_{t,1}=1/2$ and
$\ell_{t,2}\sim\mathrm{Bernoulli}(1/2\pm\varepsilon)$; both satisfy
Assumptions~\ref{ass:non-anticipation}, \ref{ass:stochastic} and~\ref{ass:gap}.
Fix any algorithm and any $T\geq1$, and let
$B:=\max\{1,\sup_{\varepsilon\in(0,1/4]}\mathcal R_T(x^c)\}$, the supremum
being over the environments $P^+_\varepsilon$. Then either $B\geq\sqrt T/30$,
or, under $P^-_{\varepsilon}$ with $\varepsilon=1/(22B)$,
\[
 \sup_{x\in\cP}\mathcal R_T(x)\;\geq\;\frac{T}{330\,B}.
\]
In particular, a uniform baseline-regret bound $B=O(\log(eT))$ implies
that, for every sufficiently large $T$, some i.i.d.\
environment with a unique optimal action forces learning
regret $\Omega(T/\log(eT))$. Thus coverage cannot generally be dropped
while retaining both the logarithmic baseline rate and
the near-square-root adversarial rate of Theorem~\ref{thm:main-result}.
\end{proposition}

\begin{proof}
Write $N=\sum_{t\le T}\one\{X_t=e_2\}$, and $\E^{\pm}$, $\Prb^{\pm}$ for
expectation and probability under $P^{\pm}_\varepsilon$. In an i.i.d.\
environment, Assumption~\ref{ass:non-anticipation} gives
$\mathcal R_T(x)=\E\sum_{t\le T}\langle x_t-x,\bar\ell\rangle$. Since
$x^c=e_1$ and $x_{t,1}=1-x_{t,2}$, we have
$\langle x_t-x^c,\bar\ell\rangle=x_{t,2}(\bar\ell_2-\bar\ell_1)$, hence
$\mathcal R_T(x^c)=\varepsilon\,\E^+[N]$ under $P^+_\varepsilon$; likewise,
$e_2$ is the unique optimal action under $P^-_\varepsilon$ and
$\mathcal R_T(e_2)=\varepsilon\,\E^-[T-N]
\geq\varepsilon\,\tfrac T2\,\Prb^-(N\le T/2)$.

The learner observes $\ell_{t,2}$ only on rounds with $X_t=e_2$, and the law
of $\ell_{t,1}$ is the same in both environments, so the divergence
decomposition for bandit feedback
\citep[Lemma~15.1]{mainbib:lattimore2020bandit} gives, for the laws of the
interaction and all $\varepsilon\le1/4$,
\[
 \KL(P^+_\varepsilon\|P^-_\varepsilon)
 =\E^+[N]\,\KL\big(\mathrm{Ber}(\tfrac12+\varepsilon)\,\big\|\,\mathrm{Ber}(\tfrac12-\varepsilon)\big)
 \le22\,\varepsilon^2\,\E^+[N].
\]
The
Bretagnolle--Huber inequality
\citep[Theorem~14.2]{mainbib:lattimore2020bandit}, applied to the event
$\{N>T/2\}$, and Markov's inequality then yield
\[
 \Prb^-(N\le T/2)\;\ge\;\tfrac12\,e^{-22\varepsilon^2\E^+[N]}-\Prb^+(N>T/2)
 \;\ge\;\tfrac12\,e^{-22\varepsilon^2\E^+[N]}-\frac{2\,\E^+[N]}{T}.
\]
Take $\varepsilon=1/(22B)\le1/22$. By definition of $B$,
$\E^+[N]=\mathcal R_T(x^c)/\varepsilon\le22B^2$, so the exponential is at
least $e^{-1}$, while $2\E^+[N]/T\le44B^2/T\le0.049$ whenever
$B\le\sqrt T/30$. Hence $\Prb^-(N\le T/2)\ge\tfrac12e^{-1}-0.049>0.135$, and
$\mathcal R_T(e_2)\ge\varepsilon\,\tfrac T2\cdot0.135\ge T/(330B)$.
\end{proof}
%
\section{Loss estimation and empirical-regret certificate}
\label{sec:loss-estimation}

\begin{lemma}[Importance-weighted unbiasedness]
\label{lem:importance-unbiasedness}
Suppose the played arm-selection marginal satisfies $x_{t,i}>0$ for
every $i\in[d]$. Define the standard semi-bandit estimate
\begin{equation}
  \widehat\ell_{t,i}:=\frac{X_{t,i}\ell_{t,i}}{x_{t,i}},
  \qquad i\in[d].
  \label{eq:iw-estimator}
\end{equation}
Then the estimator is conditionally unbiased
coordinatewise. Consequently, for every predictable $z_t\in\cP$
and every fixed $z\in\cP$, Equation~\eqref{eq:unbiasedness-bridge} holds.
\end{lemma}

\begin{yellowbox}
\textbf{Role.} Transfers estimated-regret bounds to expected regret. \textbf{Proof~idea.}
Condition on the pre-draw history and current
losses. The fresh-draw identity cancels the sampling
marginal in the estimator; predictability and the
tower property then give the regret identity.
\end{yellowbox}

\begin{proof}
\begin{align*}
  \E[\widehat\ell_{t,i}\mid\cG_t,\ell_t]
  &=\frac{\ell_{t,i}}{x_{t,i}}\E[X_{t,i}\mid\cG_t,\ell_t]\\
  &=\frac{\ell_{t,i}}{x_{t,i}}x_{t,i}=\ell_{t,i},
\end{align*}
where the second equality is
Equation~\eqref{eq:fresh-draw}. Since $z_t$ is $\cG_t$-measurable,
\begin{align}
  \E[\langle z_t-z,\widehat\ell_t\rangle\mid\cG_t,\ell_t]
  &=\sum_i(z_{t,i}-z_i)\E[\widehat\ell_{t,i}\mid\cG_t,\ell_t]
    =\langle z_t-z,\ell_t\rangle,\notag\\
  \E[\langle z_t-z,\widehat\ell_t\rangle]
  &=\E[\langle z_t-z,\ell_t\rangle].
  \label{eq:unbiasedness-bridge}
\end{align}
\end{proof}

On each phase interval $I$, let $\widehat q_t$ be
the learning distribution over $\cX$, with arm-selection marginal
\[
  b_t:=\E_{X\sim\widehat q_t}[X]
      =\sum_{X\in\cX}\widehat q_t(X)X
      \in\cP.
\]
For a constant mixing weight
$\alpha$ on the phase, play
\[
  q_t:=\alpha\widehat q_t+(1-\alpha)q^c.
\]
Its arm-selection marginal is
\[
  x_t=\alpha b_t+(1-\alpha)x^c.
\]

\begin{definition}[Phasewise empirical-regret certificate]
\label{def:base-interface}
When reset immediately after a stopping time
$\sigma$, with first new round $u=\sigma+1$,
the learning algorithm $\mathsf B$ outputs predictable
marginals $b_t\in\cP$. Define the observable statistic
\begin{equation}
  U_{u:t}:=\left[
    \max_{x\in\cP}\sum_{j=u}^t\langle b_j-x,\widehat\ell_j\rangle
  \right]_+.
  \label{eq:empirical-max}
\end{equation}
Pathwise, $U_{u:t}$ bounds estimated regret for all comparators
simultaneously, including data-dependent choices. Pure phases ($\alpha=1$)
require $b_{t,i}>0$ for all $i\in[d]$, ensuring that
Equation~\eqref{eq:iw-estimator} is well-defined and unbiased. For each
realized restart interval $I=[u,v]$, write $U_I:=U_{u:v}$. The
certificate requires, at every deterministic horizon $T$,
\begin{align}
  \E\!\left[\sum_IU_I\right]
  &\le \mathfrak C_{\mathrm{adv}}(T)
  &&\text{for non-anticipating adversarial losses},
  \label{eq:interface-adv}\\
  \E\!\left[\sum_IU_I\right]
  &\le \mathfrak C_{\mathrm{sto}}(T,\Delta)
  &&\text{under the independent stochastic condition with gap $\Delta$},
  \label{eq:interface-sto}
\end{align}
Corollary~\ref{cor:total-phase-empirical-regret} establishes these bounds for the entropy-FTRL and
hybrid-FTRL learners used in \Duet{}.
\end{definition}
%
\section{FTRL learners for \Duet}
\label{sec:base-learner}

The entropy-FTRL and hybrid-FTRL learners below define \Duet{}.
They satisfy Definition~\ref{def:base-interface} by Lemma~\ref{lem:one-phase-empirical-regret} and Corollary~\ref{cor:total-phase-empirical-regret}. Set
\[
  \Lambda:=\log(eN),
  \qquad
  \rho:=\frac{\delta}{2}.
\]
A \emph{mixed} phase uses $\alpha\le1/2$; a \emph{pure}
phase uses $\alpha=1$. The weight is constant
within each phase. Both learners use
\begin{equation}
  z_n\in\arg\min_{z\in\mathcal D}
  \left\{
    \left\langle z,\sum_{r<n}\widehat h_r\right\rangle+\Omega_n(z)
  \right\}.
  \label{eq:ftrl-template}
\end{equation}
The domain, regularizer, and estimates differ
as follows; see Remarks~\ref{rem:phase-base-choice} and~\ref{rem:pure-representation}.

\paragraph{Mixed phase.}
Observations come from the mixture of the
baseline and learning distributions. Each original arm
coordinate has sampling marginal at least $\rho$;
dummy coordinates have known zero loss and
need no estimator. In Equation~\eqref{eq:ftrl-template}, take $\mathcal D=\Delta(\widetilde\cX)$
and $\widehat h_r=(\widehat L_r(X))_{X\in\widetilde\cX}$, with $\Omega_n(q)=\eta_n^{-1}\KL(q\Vert\mathrm{Unif})$. A second-order FTRL
inequality and martingale concentration control its certificate.

\paragraph{Pure phase.}
Actions are sampled directly from the distribution
implementing the FTRL learner's arm-selection marginal. The
FTRL learner uses $\mathcal D=\conv(\widetilde\cX_0)$ and $\widehat h_r=\widetilde\ell_r$ together
with $\Omega_n(z)=\eta_n^{-1}\Psi(z)$, hence is hybrid-FTRL. A
pathwise nondecreasing majorant controls its certificate.

Theorem~\ref{thm:hybrid-ftrl} bounds the pure-phase majorant in expectation.
Section~\ref{sec:base-empirical-regret} and Lemma~\ref{lem:conditional-restart} extend the phase
bounds to stopping times and restarts.

Append $m$ zero-loss dummy arms and map
\[
  X\longmapsto\widetilde X:=(X,D_{m-\lVert X\rVert_1}),
\]
where $D_r\in\{0,1\}^m$ has its first $r$ entries equal
to one. Define $\widetilde\cX:=\{\widetilde X:X\in\cX\}$. Deleting the last
$m$ coordinates inverts the padding map, so
\[
  |\widetilde\cX|=|\cX|=N.
\]
Moreover,
\[
  \lVert\widetilde X\rVert_1
  =\lVert X\rVert_1+(m-\lVert X\rVert_1)=m.
\]
The padded dimension is $D:=d+m\le2d$. Action
losses, gaps, and regrets are unchanged
because, for the padded loss $\ell_t^{\mathrm{pad}}:=(\ell_t,0_m)$,
\[
  \langle\widetilde X,\ell_t^{\mathrm{pad}}\rangle=\langle X,\ell_t\rangle,
  \qquad
  \langle\widetilde X-\widetilde Y,\ell_t^{\mathrm{pad}}\rangle
  =\langle X-Y,\ell_t\rangle.
\]
\paragraph{Padding and loss shifts.}
For example, $(1,0,0)$ and $(0,1,1)$ become $(1,0,0,1,0)$
and $(0,1,1,0,0)$, both of size two. Equal
cardinality preserves action comparisons under the pure-phase
loss shift $g\mapsto g-\one$: for any $g\in\mathbb R^D$,
\[
  \langle\widetilde X-\widetilde Y,g-\one\rangle
  =\langle\widetilde X-\widetilde Y,g\rangle
   -(\lVert\widetilde X\rVert_1-\lVert\widetilde Y\rVert_1)
  =\langle\widetilde X-\widetilde Y,g\rangle.
\]
For $g=0$ in the example, the unpadded shifted
losses are $-1,-2$; padding makes both $-2$.

\paragraph{Notation reused after padding.}
For a distribution $q$ on $\cX$, define
\[
  \widetilde q(\widetilde X):=q(X),\qquad X\in\cX.
\]

If $\widehat\ell_t\in\mathbb R^d$ is an original loss estimate, define
\[
  \widehat\ell_t^{\mathrm{pad}}:=(\widehat\ell_t,0_m)\in\mathbb R^D.
\]
For all $X,Y\in\cX$,
\[
  \langle\widetilde X-\widetilde Y,\widehat\ell_t^{\mathrm{pad}}\rangle
  =\langle X-Y,\widehat\ell_t\rangle.
\]
On $\widetilde\cX$, reuse $q^c$, $\widehat q_t$, $\widehat\ell_t$, and $X$
(including $X^\star$) for their padded versions. In
mixed phases, $\widehat q_n$ in~\eqref{eq:mixed-explicit} is $\widetilde q_n$ in~\eqref{main:entropy-update}.

For pure-phase FTRL, delete the constant coordinates $C_{\mathrm{const}}$
from $\widetilde\cX$. Let $\Pi$ be the projection, $\widetilde\cX_0\subseteq\{0,1\}^{D_{\mathrm{var}}}$
its image, and $c_1$ the number of deleted
common-one coordinates. Then $D_{\mathrm{var}}\le D$ and, for all $\widetilde X,\widetilde Y\in\widetilde\cX$,
\[
  \lVert\Pi\widetilde X\rVert_1=m-c_1=:m_0\le m,
\]
and
\begin{align*}
  \langle\Pi\widetilde X-\Pi\widetilde Y,\Pi\ell_t^{\mathrm{pad}}\rangle
  &=\sum_{i\notin C_{\mathrm{const}}}(\widetilde X_i-\widetilde Y_i)\ell_{t,i}^{\mathrm{pad}}\\
  &=\sum_i(\widetilde X_i-\widetilde Y_i)\ell_{t,i}^{\mathrm{pad}}
   =\langle\widetilde X-\widetilde Y,\ell_t^{\mathrm{pad}}\rangle,
\end{align*}
where the second equality uses $\widetilde X_i-\widetilde Y_i=0$ for $i\in C_{\mathrm{const}}$.
The identity also holds for loss estimates
and mean losses, preserving all gaps $\Delta_X$.
On $\widetilde\cX_0$, use $X$ for $\Pi\widetilde X$ and
project loss vectors accordingly. The projected actions
have common size $m_0\le m$ and dimension $D_{\mathrm{var}}\le D$.

Lemma~\ref{lem:pure-interior} proves $0<b_i<1$ on every retained coordinate.

Deleting constant coordinates is injective, so
\[
  |\widetilde\cX_0|=|\widetilde\cX|=N\ge2.
\]

Since the projected set contains at least
two distinct binary actions, each with \(m_0\)
ones, we have $D_{\mathrm{var}}\ge1$ and $m_0\ge1$.

On a restarted \emph{mixed} phase, use entropy-FTRL
on the action simplex $\Delta(\widetilde\cX)$. If
\[
  \widehat L_t(X):=\langle X,\widehat\ell_t\rangle,
\]
put $L_{n-1}^{\mathrm{cum}}(X):=\sum_{r<n}\widehat L_r(X)$ and define
\begin{equation}
  \widehat q_n
  :=\arg\min_{q\in\Delta(\widetilde\cX)}
  \left\{
    \sum_{X\in\widetilde\cX}q(X)L_{n-1}^{\mathrm{cum}}(X)
    +\frac{\KL(q\Vert\mathrm{Unif})}{\eta_n}
  \right\},
  \qquad
  \eta_n:=\frac{\sqrt{\rho\Lambda}}{m\sqrt n}.
  \label{eq:mixed-ftrl}
\end{equation}
Put
\[
  \mathcal Z_{n-1}:=\sum_{Y\in\widetilde\cX}e^{-\eta_nL_{n-1}^{\mathrm{cum}}(Y)},
  \qquad
  \pi_n(X):=\frac{e^{-\eta_nL_{n-1}^{\mathrm{cum}}(X)}}{\mathcal Z_{n-1}}.
\]
For every $q\in\Delta(\widetilde\cX)$, expanding $\KL(q\Vert\pi_n)$ yields
\begin{align*}
  \sum_Xq(X)L_{n-1}^{\mathrm{cum}}(X)+\frac{\KL(q\Vert\mathrm{Unif})}{\eta_n}
  &=-\frac1{\eta_n}\log\frac{\mathcal Z_{n-1}}{N}
    +\frac{\KL(q\Vert\pi_n)}{\eta_n}\\
  &\ge-\frac1{\eta_n}\log\frac{\mathcal Z_{n-1}}{N},
\end{align*}
with equality if and only if $q=\pi_n$.
Hence the minimizer is unique and
\begin{equation}
  \widehat q_n(X)
  =\frac{\exp\{-\eta_nL_{n-1}^{\mathrm{cum}}(X)\}}
  {\sum_{Y\in\widetilde\cX}\exp\{-\eta_nL_{n-1}^{\mathrm{cum}}(Y)\}}.
  \label{eq:mixed-explicit}
\end{equation}

\paragraph{Padding invariance of entropy-FTRL.}
\label{par:entropy-padding}
Fix a phase-local round $n$ and an original
estimate history $\widehat\ell_r\in\mathbb R^d$, $r<n$. Zero padding gives $\langle\widetilde X,(\widehat\ell_r,0_m)\rangle =\langle X,\widehat\ell_r\rangle$.
The bijection $X\mapsto\widetilde X$ preserves the softmax denominator, so
\[
  \widetilde q_n(\widetilde X)
  =\frac{\exp\{-\eta_n\sum_{r<n}\langle X,\widehat\ell_r\rangle\}}
  {\sum_{Y\in\cX}\exp\{-\eta_n\sum_{r<n}\langle Y,\widehat\ell_r\rangle\}},
\]
which is the unpadded entropy-FTRL probability of $X$.
Both updates use the same $\eta_n$: padding preserves
$N,m$, and $\delta$ is defined on the original
arms. Starting from uniform weights, coupling the original-action
draws with the same baseline therefore gives identical
marginals, estimates, and restart decisions by induction.

To verify the mixed-phase coverage used below, write
$R_s/L=2^K$. Whenever phase $k$ is mixed, $\alpha_k<1$, and
\[
  \alpha_k=\frac{L2^{k-1}}{R_s}=2^{k-1-K}<1
  \quad\Longrightarrow\quad
  k\le K
  \quad\Longrightarrow\quad
  \alpha_k\le\frac12.
\]
Therefore every mixed round and
every arm coordinate $i\in[d]$ satisfies
\[
  x_{t,i}=\alpha b_{t,i}+(1-\alpha)x_i^c
  \ge(1-\alpha)x_i^c
  \ge\frac{x_i^c}{2}
  \ge\frac\delta2=\rho.
\]
Dummy coordinates are assigned the known estimate
zero and require no observation denominator.

\paragraph{Marginal representation for pure-phase FTRL.}
For $\mathcal A=\widetilde\cX_0$, implement each marginal $b_t\in\conv(\mathcal A)$ using
a fixed sampling rule $\mathsf P:\conv(\mathcal A)\to\Delta(\mathcal A)$ satisfying
\[
  \sum_{U\in\mathcal A}\mathsf P(b)(U)U=b
  \qquad\text{for every }b\in\conv(\mathcal A).
\]
Set $\widehat q_t=\mathsf P(b_t)$, sample $U_t\sim\widehat q_t$, and map it
back to $\cX$ through the padding/projection bijection.

Pure phases use hybrid-FTRL \citep{mainbib:zimmert2019beating} on $\conv(\widetilde\cX_0)$,
with local learning rate $\eta_n=n^{-1/2}$ and regularizer
\begin{equation}
  \Psi(z)=\sum_{i=1}^{D_{\mathrm{var}}}\left[-\sqrt{z_i}+(1-z_i)\log(1-z_i)\right].
  \label{eq:hybrid-regularizer}
\end{equation}
Let $X_t^0\in\widetilde\cX_0$ be the projected sampled action,
and let $g_t\in[0,1]^{D_{\mathrm{var}}}$ be the arm-loss vector
in the padded-and-projected representation. Use the
shifted vector $g_t-\one\in[-1,0]^{D_{\mathrm{var}}}$ and its estimator
\begin{equation}
  \widehat g_{t,i}:=\frac{X_{t,i}^0g_{t,i}}{b_{t,i}},
  \qquad
  \widetilde\ell_{t,i}:=\widehat g_{t,i}-1.
  \label{eq:shifted-estimator}
\end{equation}
The phase-local update is
\[
  b_n=\arg\min_{z\in\conv(\widetilde\cX_0)}
  \left\{\left\langle z,\sum_{r<n}\widetilde\ell_r\right\rangle+\sqrt n\,\Psi(z)\right\},
  \qquad \widehat q_n=\mathsf P(b_n).
\]
Map $\widehat q_n$ back to $\cX$ for sampling;
use its original marginal in Equation~\eqref{eq:iw-estimator} and
the phase tests. Dummy coordinates have shifted
estimate $-1$. Since $\langle x,\one\rangle=m_0$ for all $x\in\conv(\widetilde\cX_0)$,
\begin{align}
  \langle b_t-x,\widetilde\ell_t\rangle
  &=\langle b_t-x,\widehat g_t-\one\rangle\notag\\
  &=\langle b_t-x,\widehat g_t\rangle
   -\underbrace{\langle b_t-x,\one\rangle}_{m_0-m_0=0}\notag\\
  &=\langle b_t-x,\widehat g_t\rangle.
  \label{eq:shift-invariance}
\end{align}
Thus the shift preserves the empirical-regret statistic.

%
\section{Implementation details of \Duet}
\label{sec:phased-algorithm}

\Duet{} (Algorithm~\ref{alg:phased-sb}) uses two counters for its doublings:
\begin{itemize}
  \item A \emph{stage} $s$ has an empirical-regret threshold
  $R_s$. The algorithm starts a new stage when
  the realized empirical-regret statistic $U_{u:t}$ exceeds $R_s$.
  \item Within a stage, a \emph{phase} $k$
  has a fixed mixing weight $\alpha_k$. Increasing
  $k$ doubles the weight on the base
  learner until $\alpha_k=1$, at which point the
  baseline is no longer mixed in.
\end{itemize}
Each phase starts a fresh learner at
round $u$; a stage restart also resets
the phase. Both $U_{u:t}$ and $G^c_{u:t}$
use only the current phase's data.

In phase $k$, with $R_s/L$
a power of two, set
\[
\alpha_k:=\min\left\{\frac{L2^{k-1}}{R_s},1\right\}.
\]
Since a stage begins with $k=1$ and
$R_s\ge 2L$, its first phase has $\alpha_1\le 1/2$.

The current base learner returns a distribution $\widehat q_t$
over feasible actions. Its arm-selection marginal is
\[
  b_t:=\sum_{X\in\cX}\widehat q_t(X)X.
\]
The played distribution is the action-level
mixture $q_t=\alpha_k\widehat q_t+(1-\alpha_k)q^c$, where $\widehat q_t$ has marginal
$b_t$. The played marginal is
\[
  x_t=\alpha_k b_t+(1-\alpha_k)x^c.
\]

For \Duet, use entropy-FTRL~\eqref{eq:mixed-ftrl} on $\Delta(\widetilde\cX)$ when $\alpha<1$
and hybrid-FTRL~\eqref{eq:hybrid-regularizer}--\eqref{eq:shifted-estimator}
on $\conv(\widetilde\cX_0)$ when $\alpha=1$. In the pseudocode, $\widehat q_t$
and $b_t$ are mapped back to the original
action set and coordinates. The pure-phase sampling
map is fixed in advance, for example
by a deterministically tie-broken linear feasibility program.
Each phase fixes its regularizer and uses
local time $n=t-u+1$; a restart clears cumulative
estimates and resets $n=1$. The algorithm is
horizon-free, with guarantees at each deterministic $T$.

After incorporating round $t$ into the phase
statistics, apply the first applicable rule:
\begin{enumerate}
  \item If $U_{u:t}>R_s$, start stage $s+1$
  with $R_{s+1}=2R_s$ and reset $k=1$.
  \item Otherwise, if $G^c_{u:t}>4R_s$ and $\alpha<1$, keep
  the current stage and increment $k$ by
  one, doubling the mixing weight $\alpha$.
  \item Otherwise, update $B$ using the
  round-$t$ loss estimate specified in Algorithm~\ref{alg:phased-sb}.
\end{enumerate}
In the first two cases, call \textsc{StartPhase}$(t+1)$.

A triggering round belongs to the ending phase;
discard its learner without another update and restart
at $u=t+1$. Mixed phases satisfy $x_{t,i}\ge\delta/2$. In
pure phases, variable coordinates have positive marginals
and original constant-one coordinates have marginal one.
Thus all original-arm importance-weight denominators are positive.

Compute the statistics from
\[
  S=\sum_{j=u}^t\widehat\ell_j,
  \qquad
  A=\sum_{j=u}^t\langle b_j,\widehat\ell_j\rangle,
  \qquad
  v=\min_{X\in\mathcal X}\langle X,S\rangle,
\]
so $U_{u:t}=[A-v]_+$ and $G^c_{u:t}=\langle x^c,S\rangle-v$.

%
\section{Proof of the FTRL learners' empirical-regret bound}
\label{sec:base-empirical-regret}

For \Duet, we use local round indices and
a deterministic cap $H$ on the stopped phase
length $\tau$.
Lemma~\ref{lem:conditional-restart} extends the following bounds to data-dependent starts.

\begin{lemma}[One-phase empirical-regret bound]
\label{lem:one-phase-empirical-regret}
Fix an integer $H\geq1$ and initialize either entropy-FTRL
with a fixed baseline-mixture weight $\alpha\in(0,1/2]$ (a \emph{mixed
phase}) or hybrid-FTRL with direct sampling (a \emph{pure
phase}). Let $\tau\in\{0,\ldots,H\}$ be a stopping time for
the local after-round filtration $(\mathcal F_n)_{n=0}^H$, initialized with
the phase-start history; thus $\{n\leq\tau\}\in\mathcal F_{n-1}$. Define $U_{1:\tau}$
by \eqref{eq:empirical-max} on local rounds, with $U_{1:0}:=0$.

Conditionally on the phase-start history, the stochastic case
assumes $H$ i.i.d. losses with law $P$, independent
i.i.d. fresh seeds with law $\nu$, and a
unique optimal action with gap $\Delta>0$. The adversarial
case assumes Assumption~\ref{ass:non-anticipation} on each local round, relative
to its pre-draw history and current losses.

Then, for universal constants $c_1,c_2$,
\begin{align}
\E[U_{1:\tau}]
&\leq c_1m\sqrt{\frac{H\Lambda}{\rho}}
&&\text{on an adversarial mixed phase}, \label{eq:mixed-adv-phase}\\
\E[U_{1:\tau}]
&\leq c_2\frac{m^2\Lambda}{\rho\Delta}
&&\text{on an i.i.d. mixed phase}, \label{eq:mixed-sto-phase}\\
\E[U_{1:\tau}]
&\leq c_2\bigl(\sqrt{mDH}+m\bigr)
&&\text{on an adversarial pure phase}, \label{eq:pure-adv-phase}\\
\E[U_{1:\tau}]
&\leq c_2\left[\frac{mD}{\Delta}\log(eH)
       +\frac{m^2}{\Delta}+m\right]
&&\text{on an i.i.d. pure phase}. \label{eq:pure-sto-phase}
\end{align}
Expectations are over the fresh run
conditional on its start history.
\end{lemma}

\begin{yellowbox}
\textbf{Role.} This lemma supplies the four
stopped-phase certificates used in Lemma~\ref{lem:main-stopped-certificates}.

\textbf{Proof~idea.} Mixed phases use full coverage and fluctuation
bounds; pure phases use a nonnegative, nondecreasing
majorant and a run continued to $H$.
The empirical maximization stays inside expectation, and
$H$ is only an analysis cap. Lemma~\ref{lem:conditional-restart}
transfers these fresh-run bounds to data-dependent phase
starts by conditioning on the phase-start history.
\end{yellowbox}

\begin{proof}
Steps 1--4 treat mixed phases, Steps 5--7 treat
pure phases, and Step 8 handles stopped endpoints.

\paragraph{Step 1: a pathwise certificate
for a mixed phase.}
Define
\[
\widehat L_t^{\mathrm{cum}}(X)
    :=\sum_{r=1}^t\widehat L_r(X),
\qquad
\mathcal M_t(\eta)
    :=-\frac{1}{\eta}\log\!\left[
       \frac1N\sum_X e^{-\eta\widehat L_t^{\mathrm{cum}}(X)}
       \right],
\]
Set $\mathcal M_0(\eta)=0$ and define the mix loss
\[
\mu_t:=\mathcal M_t(\eta_t)-\mathcal M_{t-1}(\eta_t)
    =-\frac1{\eta_t}\log\sum_X
       \widehat q_t(X)e^{-\eta_t\widehat L_t(X)}.
\]
Set
\[
V_t:=\sum_{X\in\widetilde{\mathcal X}}
        \widehat q_t(X)\widehat L_t(X)^2.
\]
Because every $\widehat L_t(X)$ is nonnegative,
$e^{-z}\leq 1-z+z^2/2$ for $z\geq0$ gives
\[
\sum_X\widehat q_t(X)e^{-\eta_t\widehat L_t(X)}
\leq 1-\eta_t\sum_X\widehat q_t(X)\widehat L_t(X)
       +\frac{\eta_t^2}{2}V_t.
\]
Applying $\log(1+u)\leq u$ and rearranging yields
\begin{equation}
\sum_X\widehat q_t(X)\widehat L_t(X)
\leq \mu_t+\frac{\eta_t}{2}V_t.
\label{eq:mix-loss-second-order}
\end{equation}

Define
\[
\pi_{t,\eta}(X):=
\frac{e^{-\eta\widehat L_t^{\mathrm{cum}}(X)}}
     {\sum_Y e^{-\eta\widehat L_t^{\mathrm{cum}}(Y)}}.
\]
For every distribution $q$,
\begin{align}
\sum_Xq(X)\widehat L_t^{\mathrm{cum}}(X)
    +\frac{\KL(q\|\mathrm{Unif})}{\eta}
&=\mathcal M_t(\eta)
    +\frac{\KL(q\|\pi_{t,\eta})}{\eta}
\geq\mathcal M_t(\eta),
\label{eq:gibbs-variational-identity}
\end{align}
and therefore
\[
\mathcal M_t(\eta)
=\min_{\substack{q\text{ a distribution}\\
                  \text{on }\widetilde{\mathcal X}}}
  \left\{\sum_Xq(X)\widehat L_t^{\mathrm{cum}}(X)
        +\frac{\KL(q\|\mathrm{Unif})}{\eta}\right\},
\]
with minimizer $q=\pi_{t,\eta}$. Since $\KL(q\|\mathrm{Unif})\geq0$, we have
\begin{equation}
0<\eta'\leq\eta
\quad\Longrightarrow\quad
\mathcal M_t(\eta')\geq\mathcal M_t(\eta).
\label{eq:mix-potential-monotone}
\end{equation}
Since $\eta_{t+1}\leq\eta_t$, telescoping gives

\begin{align*}
\sum_{t=1}^n\mu_t
&=\sum_{t=1}^n
  [\mathcal M_t(\eta_t)-\mathcal M_{t-1}(\eta_t)]\\
&=\mathcal M_n(\eta_n)
  +\sum_{t=1}^{n-1}
   [\mathcal M_t(\eta_t)-\mathcal M_t(\eta_{t+1})]\\
&\leq\mathcal M_n(\eta_n)
\leq\min_X\widehat L_n^{\mathrm{cum}}(X)
     +\frac{\log N}{\eta_n}.
\end{align*}
The inequalities above use \eqref{eq:mix-potential-monotone}
and $X^*\in\arg\min_X\widehat L_n^{\mathrm{cum}}(X)$ with \(q^*(X^*)=1\).
\begin{align*}
\max_X\sum_{t=1}^n\langle b_t-X,\widehat\ell_t\rangle
&=\sum_{t=1}^n\sum_Z\widehat q_t(Z)\widehat L_t(Z)
  -\min_X\widehat L_n^{\mathrm{cum}}(X)\\
& \leq 
\mathcal M_n(\eta_n)-\min_X\widehat L_n^{\mathrm{cum}}(X)
+\frac12\sum_{t=1}^n\eta_tV_t\\
&\leq\frac{\log N}{\eta_n}
   +\frac12\sum_{t=1}^n\eta_tV_t
\leq\frac{\log N}{\eta_{n+1}}
   +\frac12\sum_{t=1}^n\eta_tV_t.
\end{align*}
where first inequality follows from \eqref{eq:mix-loss-second-order} and $\sum_{t=1}^n\mu_t\leq\mathcal M_n(\eta_n)$.
The second inequality is from telescoping result above.

Taking positive parts yields, pathwise for every $n\leq H$,
\begin{equation}
U_{1:n}\leq\frac{\log N}{\eta_{n+1}}
             +\frac12\sum_{t=1}^n\eta_tV_t.
\label{eq:pathwise-mixed-certificate}
\end{equation}
No uniform upper bound on
the importance-weighted losses was used.

\paragraph{Step 2: the mixed adversarial expectation bound.}
If $b_t$ is the arm-selection marginal
of $\widehat q_t$, Cauchy--Schwarz and $\lVert X\rVert_1\leq m$ give
\begin{align*}
\widehat L_t(X)^2
&=\left(\sum_iX_i\widehat\ell_{t,i}\right)^2
 \leq\|X\|_1\sum_iX_i\widehat\ell_{t,i}^2
 \leq m\sum_iX_i\widehat\ell_{t,i}^2,\\
V_t
&=\sum_X\widehat q_t(X)\widehat L_t(X)^2
 \leq m\sum_i\left(\sum_X\widehat q_t(X)X_i\right)
             \widehat\ell_{t,i}^2
 =m\sum_i b_{t,i}\widehat\ell_{t,i}^2.
\end{align*}
Write $p_t:=x_t$. Mixed-phase coverage gives
$p_{t,i}\geq\rho$ for $i\in[d]$, so
\begin{equation}
\E[V_t\mid\mathcal F_{t-1},\ell_t]
\leq m\sum_{i\leq d}\frac{b_{t,i}\ell_{t,i}^2}{p_{t,i}}
\leq\frac{m^2}{\rho}.
\label{eq:mixed-second-moment}
\end{equation}
Predictability of $\1\{t\leq\tau\}$ gives
\[
\E[\1\{t\leq\tau\}V_t]
\leq\frac{m^2}{\rho}\Pr(t\leq\tau)
\leq\frac{m^2}{\rho}.
\]
Using $\tau\leq H$, $\log N\leq\Lambda$, and $\sum_{t=1}^Ht^{-1/2}\leq2\sqrt H$ in Equation~\eqref{eq:pathwise-mixed-certificate} gives
\begin{align*}
\E U_{1:\tau}
&\leq\frac{\log N}{\eta_{H+1}}
   +\frac12\sum_{t=1}^H\eta_t
       \E[\1\{t\leq\tau\}V_t]\\
&\leq m\sqrt{\frac{(H+1)\Lambda}{\rho}}
  +\frac{m\sqrt\Lambda}{2\sqrt\rho}
      \sum_{t=1}^Ht^{-1/2}\\
&\leq(\sqrt2+1)m\sqrt{\frac{H\Lambda}{\rho}}.
\end{align*}
This proves Equation~\eqref{eq:mixed-adv-phase} with $c_1=\sqrt2+1$.

\paragraph{Step 3: stochastic gap control
in a mixed phase.}
Continue the phase to $H$ with the same
mixing weight and unused i.i.d. losses and independent
seeds. The continuation agrees through $\tau$. For the
unique optimum $X^\star$ and gap $\Delta_X:=\langle X-X^\star,\bar\ell\rangle$, define
\[
D_n(X):=\sum_{t=1}^n\langle X^\star-X,\widehat\ell_t\rangle.
\]
Write $\E_{t-1}[\cdot]:=\E[\cdot\mid\mathcal F_{t-1}]$. Unbiasedness and independence of the
current loss from the past give
\begin{align}
\E_{t-1}[D_t(X)-D_{t-1}(X)]
&=\E_{t-1}\langle X^\star-X,\widehat\ell_t\rangle
 =\langle X^\star-X,\bar\ell\rangle
 =-\Delta_X.
\label{eq:gap-martingale-drift}
\end{align}
Let $P_X=X^\star\setminus X$ and $M_X=X\setminus X^\star$. Padding gives $|P_X|=|M_X|=:k\leq m$.
Write $Y_t:=D_t(X)-D_{t-1}(X)=A_t-B_t$, with $A_t:=\sum_{i\in P_X}\widehat\ell_{t,i}\geq0$ and $B_t:=\sum_{i\in M_X}\widehat\ell_{t,i}\geq0$. On
original coordinates, the fresh-draw identity gives $\E[\widehat\ell_{t,i}^2\mid\mathcal F_{t-1},\ell_t] =\ell_{t,i}^2/p_{t,i}\leq1/\rho$;
dummy estimates vanish. Cauchy--Schwarz and $|P_X\cup M_X|=2k$ yield
\[
Y_t^2\leq A_t^2+B_t^2
 \leq k\sum_{i\in P_X\cup M_X}\widehat\ell_{t,i}^2,
\qquad
\E_{t-1}Y_t^2\leq\frac{2k^2}{\rho}
\leq\frac{2m^2}{\rho}.
\]
Also,
\[
0\leq A_t,B_t\leq\frac{k}{\rho},
\qquad
|Y_t|=|A_t-B_t|\leq\frac{k}{\rho}\leq\frac m\rho,
\]
and, because $\Delta_X\leq m$ and $\rho\leq1$,
\[
|Y_t-\E_{t-1}Y_t|=|Y_t+\Delta_X|
\leq\frac m\rho+m\leq\frac{2m}{\rho},
\qquad
\Var_{t-1}(Y_t)\leq\E_{t-1}Y_t^2\leq\frac{2m^2}{\rho}.
\]
Put $r=n-1$ and
\[
S_r:=D_r(X)+r\Delta_X.
\]
For $n\geq2$, apply Lemma~\ref{lem:freedman} to the martingale
sum $S_r$ at threshold $r\Delta_X/2$, with
\[
x=\frac{r\Delta_X}{2},
\qquad v=\frac{2rm^2}{\rho},
\qquad c=\frac{2m}{\rho}.
\]
Since $0\leq\Delta_X\leq m$, the denominator in
the Freedman exponent obeys
\[
2(v+cx/3)=\frac{4rm^2}{\rho}+\frac{2rm\Delta_X}{3\rho}
\leq\frac{14rm^2}{3\rho}\leq\frac{5rm^2}{\rho}.
\]
Thus
\begin{equation}
\Pr\!\left(D_{n-1}(X)>-\frac{(n-1)\Delta_X}{2}\right)
\leq\exp\!\left\{-\frac{(n-1)\rho\Delta_X^2}{20m^2}\right\}.
\label{eq:freedman-gap-event}
\end{equation}
For $n=1$, the event is empty. Moreover,
\[
\frac{\widehat q_n(X)}{\widehat q_n(X^\star)}
=\exp\{\eta_nD_{n-1}(X)\}.
\]
For $0<s\leq m$, put $Q_n(s):=\sum_{X:\Delta_X\geq s}\widehat q_n(X)$ and define
\[
\mathcal E_{n,s}:=
\bigcap_{X:\Delta_X\geq s}
\left\{D_{n-1}(X)\leq-\frac{(n-1)\Delta_X}{2}\right\}.
\]
On $\mathcal E_{n,s}$,
\[
\widehat q_n(X)
\leq\frac{\widehat q_n(X)}{\widehat q_n(X^\star)}
\leq\exp\!\left\{-\frac{\eta_n(n-1)s}{2}\right\}.
\]
For $n\geq2$,
\[
\frac{\eta_n(n-1)s}{2}
=\frac{s\sqrt{\rho\Lambda}}{2m}\frac{n-1}{\sqrt n}
\geq\frac{s\sqrt{\rho\Lambda(n-1)}}{4m}
\geq\frac{s\sqrt{\rho\Lambda(n-1)}}{8m},
\]
and hence
\[
\E[Q_n(s)\1_{\mathcal E_{n,s}}]
\leq\min\!\left\{1,
N\exp\!\left(-\frac{s\sqrt{\rho\Lambda(n-1)}}{8m}\right)\right\}.
\]
On the complementary event, a
union bound and Equation~\eqref{eq:freedman-gap-event} give
\begin{align*}
\Pr(\mathcal E_{n,s}^{c})
&\leq\sum_{X:\Delta_X\geq s}
 \exp\!\left(-\frac{(n-1)\rho\Delta_X^2}{20m^2}\right)\\
&\leq N\exp\!\left(-\frac{(n-1)\rho s^2}{20m^2}\right)
 \leq N\exp\!\left(-\frac{(n-1)\rho s^2}{40m^2}\right).
\end{align*}
Since $0\leq Q_n(s)\leq1$,
\[
\E Q_n(s)
\leq\E[Q_n(s)\1_{\mathcal E_{n,s}}]
   +\Pr(\mathcal E_{n,s}^{c}).
\]
The bound also holds for
$n=1$ since $Q_1(s)\leq1$. Thus,
\begin{align}
\E Q_n(s)
&\leq\min\!\left\{1,
N\exp\!\left(-\frac{s\sqrt{\rho\Lambda(n-1)}}{8m}\right)\right\}
\notag\\[-2pt]
&\quad+\min\!\left\{1,
N\exp\!\left(-\frac{(n-1)\rho s^2}{40m^2}\right)\right\}.
\label{eq:gap-mass-bound}
\end{align}
Summing with \eqref{eq:combined-exponential-series}, using $r=n-1$,
$a=s\sqrt{\rho\Lambda}/(8m)$, and $b=\rho s^2/(40m^2)$, gives
\begin{equation}
\sum_{n=1}^H\E Q_n(s)\leq404\frac{m^2\Lambda}{\rho s^2}.
\label{eq:gap-mass-series}
\end{equation}
Set $r_n:=\sum_X\widehat q_n(X)\Delta_X$ and $R_H:=\E\sum_{n=1}^Hr_n$. Layer cake gives
\[
r_n=\sum_X\widehat q_n(X)\int_0^m\1\{s\leq\Delta_X\}\,ds
    =\int_0^mQ_n(s)\,ds.
\]
Since no action gap lies in $(0,\Delta)$, $Q_n(s)=Q_n(\Delta)$
for $0<s\leq\Delta$. Tonelli's theorem and Equation~\eqref{eq:gap-mass-series} imply
\begin{align}
R_H
&=\int_0^m\sum_{n=1}^H\E Q_n(s)\,ds\notag\\
&\leq\Delta\sum_{n=1}^H\E Q_n(\Delta)
 +404\frac{m^2\Lambda}{\rho}\int_\Delta^m\frac{ds}{s^2}
\leq808\frac{m^2\Lambda}{\rho\Delta}.
\label{eq:ordinary-mixed-pseudoregret}
\end{align}

\paragraph{Step 4: the expected empirical
maximum in a mixed phase.}
Define
\[
A_\tau(X^\star):=
\sum_{t=1}^{\tau}\langle b_t-X^\star,\widehat\ell_t\rangle.
\]
Pathwise,
\begin{align}
U_{1:\tau}
&=\left[A_\tau(X^\star)+\max_XD_\tau(X)\right]_+\notag\\
&\leq[A_\tau(X^\star)]_+
   +\left[\max_XD_\tau(X)\right]_+
 =[A_\tau(X^\star)]_++\max_X[D_\tau(X)]_+.
\label{eq:empirical-maximum-decomposition}
\end{align}
For $v_t:=b_t-X^\star$, unbiasedness gives
\[
\E_{t-1}\langle v_t,\widehat\ell_t\rangle
=\langle b_t-X^\star,\bar\ell\rangle
=\sum_X\widehat q_t(X)\Delta_X=r_t.
\]
Conditional on the current loss vector,
at most $m$ summands in the
importance-weighted inner product are nonzero. Hence
\begin{align*}
\langle v_t,\widehat\ell_t\rangle^2
&=\left(\sum_{i\leq d}
\frac{v_{t,i}X_{t,i}\ell_{t,i}}{p_{t,i}}\right)^2
\leq m\sum_{i\leq d}
\frac{v_{t,i}^2X_{t,i}\ell_{t,i}^2}{p_{t,i}^2},\\
\E_{t-1}\langle v_t,\widehat\ell_t\rangle^2
&\leq m\sum_{i\leq d}\frac{v_{t,i}^2}{p_{t,i}}
\leq\frac m\rho\|v_t\|_2^2.
\end{align*}
Using $|v_{t,i}|\leq1$, the $\ell_1$ triangle inequality, $\|X-X^\star\|_1\leq2m$,
and $\Delta_X\geq\Delta$ for $X\neq X^\star$, we obtain
\begin{align*}
\|b_t-X^\star\|_2^2
&\leq\|b_t-X^\star\|_1
\leq\sum_X\widehat q_t(X)\|X-X^\star\|_1\\
&\leq2m\sum_{X\neq X^\star}\widehat q_t(X)
\leq\frac{2m}{\Delta}\sum_X\widehat q_t(X)\Delta_X
=\frac{2m}{\Delta}r_t.
\end{align*}
Combining the last two displays gives
\begin{equation}
\E_{t-1}\langle b_t-X^\star,\widehat\ell_t\rangle^2
\leq\frac m\rho\|b_t-X^\star\|_2^2
\leq\frac{2m^2}{\rho\Delta}r_t.
\label{eq:optimum-fluctuation-second-moment}
\end{equation}
Set $W_t:=\langle b_t-X^\star,\widehat\ell_t\rangle$ and $M_\tau:=\sum_{t\leq\tau}(W_t-r_t)$. Since $\E_{t-1}W_t=r_t\geq0$, $M_\tau$
is a stopped martingale sum and
\[
\E\sum_{t\leq\tau}r_t
=\sum_{t=1}^H\E[\1\{t\leq\tau\}r_t]
\leq\sum_{t=1}^H\E r_t=R_H.
\]
Moreover, $\E_{t-1}(W_t-r_t)^2=\operatorname{Var}_{t-1}(W_t) \leq\E_{t-1}W_t^2$. The stopped identity~\eqref{eq:stopped-l2-identity}
and \eqref{eq:optimum-fluctuation-second-moment} therefore give
\[
\E M_\tau^2
\leq\frac{2m^2}{\rho\Delta}
      \E\sum_{t\leq\tau}r_t
\leq\frac{2m^2R_H}{\rho\Delta}.
\]
Cauchy--Schwarz, \eqref{eq:ordinary-mixed-pseudoregret}, and $\Lambda\geq1$ yield
\begin{equation}
\E[[A_\tau(X^\star)]_+]
\leq\E\sum_{t\leq\tau}r_t+\E|M_\tau|
\leq R_H+\sqrt{\frac{2m^2R_H}{\rho\Delta}}
\leq C\frac{m^2\Lambda}{\rho\Delta},
\label{eq:optimum-positive-fluctuation}
\end{equation}
Finally, take $\lambda:=\rho\Delta/(4m^2)$. The same increment bounds
used in Equation~\eqref{eq:freedman-gap-event}, together with the conditional
Bernstein mgf inequality, Equation~\eqref{eq:conditional-bernstein-log-mgf}, give $\lambda(2m/\rho)/3=\Delta/(6m)\leq1/6$ and
\begin{align*}
\log\E_{t-1}e^{\lambda\langle X^\star-X,\widehat\ell_t\rangle}
&\leq-\lambda\Delta_X
 +\frac{\lambda^2(2m^2/\rho)}
        {2[1-\lambda(2m/\rho)/3]}\\
&\leq-\lambda\Delta_X+\frac3{10}\lambda\Delta
\leq0.
\end{align*}
Equivalently,
\[
\E_{t-1}\exp\{\lambda\langle X^\star-X,
\widehat\ell_t\rangle\}\leq1
\qquad(X\neq X^\star).
\]
Thus $e^{\lambda D_n(X)}$ is a nonnegative supermartingale. Equation~\eqref{eq:expected-max-via-ville} yields
\begin{equation}
\E\max_{X\in\widetilde{\mathcal X}}\sup_{n\leq H}[D_n(X)]_+
\leq\frac{\log N+1}{\lambda}
\leq\frac{4m^2\Lambda}{\rho\Delta}.
\label{eq:misranking-maximum}
\end{equation}
Combining Equations~\eqref{eq:empirical-maximum-decomposition}, \eqref{eq:optimum-positive-fluctuation}, and \eqref{eq:misranking-maximum} proves Equation~\eqref{eq:mixed-sto-phase}.

\paragraph{Step 5: a pathwise majorant
for a pure phase.}
Set $\widehat{\bm L}_n:=\sum_{t\leq n}\widetilde\ell_t$ and define
\[
\Phi_n(y):=\max_{z\in\conv(\widetilde{\mathcal X}_0)}
\{\langle z,y\rangle-\eta_n^{-1}\Psi(z)\}.
\]
Lemma~\ref{lem:pure-interior} gives $b_n=\nabla\Phi_n(-\widehat{\bm L}_{n-1})$. Define
\begin{align*}
S_n&:=\sum_{t=1}^n
 \left[\langle b_t,\widetilde\ell_t\rangle
 +\Phi_t(-\widehat{\bm L}_t)
 -\Phi_t(-\widehat{\bm L}_{t-1})\right],\\
P_n&:=-\Psi(b_1)-m_0
 +\sum_{t=2}^n(\sqrt t-\sqrt{t-1})[-\Psi(b_t)-m_0],\\
C_n&:=S_n+P_n.
\end{align*}
Here $S_n$ is the stability term and
$P_n$ collects the initial and rate-change penalties,
centered at the vertex value $-m_0$. Convexity
of $\Phi_n$ and its gradient identity give
\begin{align*}
\Phi_n(-\widehat{\bm L}_n)
&\geq\Phi_n(-\widehat{\bm L}_{n-1})
 +\langle b_n,-\widetilde\ell_n\rangle,\\
S_n-S_{n-1}
&=\langle b_n,\widetilde\ell_n\rangle
 +\Phi_n(-\widehat{\bm L}_n)
 -\Phi_n(-\widehat{\bm L}_{n-1})\geq0.
\end{align*}
Moreover, every vertex $x\in\widetilde{\mathcal X}_0$ satisfies $\Psi(x)=-m_0$. For
every $z$ in its convex hull,
\[
\Psi(z)=\sum_i\{-\sqrt{z_i}+(1-z_i)\log(1-z_i)\}
\leq-\sum_i z_i=-m_0.
\]
Therefore
\[
P_1=-\Psi(b_1)-m_0\geq0,
\qquad
P_n-P_{n-1}=(\sqrt n-\sqrt{n-1})[-\Psi(b_n)-m_0]\geq0
\quad(n\geq2).
\]
Thus $S_n$, $P_n$, and $C_n=S_n+P_n$
are nonnegative and nondecreasing.

For every $x\in\widetilde{\mathcal X}_0$, the definition of $S_n$ gives
\begin{equation}
\sum_{t=1}^n\langle b_t-x,\widetilde\ell_t\rangle
=S_n+\sum_{t=1}^n
 \{\Phi_t(-\widehat{\bm L}_{t-1})-
   \Phi_t(-\widehat{\bm L}_t)\}
 -\langle x,\widehat{\bm L}_n\rangle.
\label{eq:pure-ftrl-decomposition}
\end{equation}
The potential terms telescope as
\begin{align}
\sum_{t=1}^n
 \{\Phi_t(-\widehat{\bm L}_{t-1})-
   \Phi_t(-\widehat{\bm L}_t)\}
&=\Phi_1(0)-\Phi_n(-\widehat{\bm L}_n)\notag\\
&\quad+\sum_{t=1}^{n-1}
 \{\Phi_{t+1}(-\widehat{\bm L}_t)-
   \Phi_t(-\widehat{\bm L}_t)\}.
\label{eq:changing-potential-telescope}
\end{align}
We have $\Phi_1(0)=-\Psi(b_1)$. Plugging $b_{t+1}$ in $\Phi_t(-\widehat{\bm L}_t)$ gives
\begin{equation}
\Phi_{t+1}(-\widehat{\bm L}_t)-\Phi_t(-\widehat{\bm L}_t)
\leq(\sqrt{t+1}-\sqrt t)[-\Psi(b_{t+1})].
\label{eq:potential-rate-change}
\end{equation}
Finally, plugging $x$ in the maximum
defining $\Phi_n$ and using $\Psi(x)=-m_0$ gives
\begin{equation}
-\Phi_n(-\widehat{\bm L}_n)-\langle x,\widehat{\bm L}_n\rangle
\leq\sqrt n\Psi(x)=-m_0\sqrt n.
\label{eq:test-pure-comparator}
\end{equation}
Substituting Equations~\eqref{eq:changing-potential-telescope}-- \eqref{eq:test-pure-comparator} into Equation~\eqref{eq:pure-ftrl-decomposition},
and using $1+\sum_{t=2}^n(\sqrt t-\sqrt{t-1})=\sqrt n$, yields
\begin{align}
\sum_{t=1}^n\langle b_t-x,\widetilde\ell_t\rangle
&\leq S_n-\Psi(b_1)
 +\sum_{t=2}^n(\sqrt t-\sqrt{t-1})[-\Psi(b_t)]
 -m_0\sqrt n\notag\\
&=S_n+P_n=C_n.
\label{eq:pure-majorant-comparator}
\end{align}
Since the comparator objective is linear,
\[
\max_{x\in\conv(\widetilde{\mathcal X}_0)}
\sum_{t=1}^n\langle b_t-x,\widehat g_t\rangle
=\max_{X\in\widetilde{\mathcal X}_0}
\sum_{t=1}^n\langle b_t-X,\widetilde\ell_t\rangle
\leq C_n,
\]
where the equality uses Equation~\eqref{eq:shift-invariance} and
the inequality is Equation~\eqref{eq:pure-majorant-comparator}. Because $C_n\geq0$,
\[
U_{1:n}=\left[
\max_x\sum_{t=1}^n\langle b_t-x,\widehat g_t\rangle
\right]_+\leq C_n.
\]
Since $C_n$ is nondecreasing, $U_{1:\tau}\leq C_\tau\leq C_H$
for every stopping time $\tau\leq H$.

\paragraph{Step 6: bound the pure majorant adversarially.}
Apply Theorem~\ref{thm:hybrid-ftrl} with a fixed comparator $X^\circ\in\widetilde{\mathcal X}_0$
($X^\circ=X^\star$ in the stochastic case), using
\[
f(z):=\sum_{i:X_i^\circ=0}\sqrt{z_i},
\qquad
g(z):=\sum_{i:X_i^\circ=1}
       (1-\log(1-z_i))(1-z_i).
\]
For original coordinates, Assumption~\ref{ass:non-anticipation} and the tower property
give the conditional marginal identity after projection. For
dummy coordinates, $g_{t,i}=0$ and the shifted estimate is
$-1$. Thus \eqref{eq:hybrid-direct-sampling} holds in both cases, and
\[
\E\!\left[
\left.\left(\frac{X_{t,i}^0g_{t,i}}{b_{t,i}}-1\right)^2
\right|\mathcal F_{t-1},g_t\right]
=\frac{g_{t,i}^2}{b_{t,i}}-2g_{t,i}+1 
= \left( \frac{1}{b_{t,i}} -1 \right)g_{t,i}^2 + (1 - g_{t,i})^2
\leq\frac1{b_{t,i}}.
\]
For $t\geq58$, $\eta_t=t^{-1/2}\leq\min\{(\sqrt2-1)/2,\log(2)/4\}$, so the local inverse-Hessian
bound \eqref{eq:local inverse-Hessian bound} applies; \eqref{eq:hybrid-BH}
charges earlier rounds $116m_0$. Since $m_0\leq m$,
\[
\E C_H\leq\sum_{n=1}^H\frac{25}{\sqrt n}
      \{f(\E b_n)+g(\E b_n)\}+116m.
\]
In Theorem~\ref{thm:hybrid-ftrl}, take $z_n=b_n$, $A_n=X_n^0$, $\lambda_n=g_n-\1$,
and $\widehat\lambda_n=\widetilde\ell_n$. Equal cardinality gives
\[
\langle b_t,\widetilde\ell_t\rangle
=\sum_i b_{t,i}
\left(\frac{X_{t,i}^0g_{t,i}}{b_{t,i}}-1\right)
=\langle X_t^0,g_t\rangle-m_0
=\langle X_t^0,g_t-\mathbf1\rangle
=\langle A_t,\lambda_t\rangle.
\]
 Thus \eqref{eq:hybrid-stability} gives the representation of $\E S_H$,
 and \eqref{eq:hybrid-stability-bound} gives its upper bound.
 For $z\in\conv(\widetilde{\mathcal X}_0)$, $\sum_i z_i=m_0$ and Cauchy--Schwarz give
\[
f(z)\leq\sqrt{D_{\mathrm{var}}\sum_{i:X_i^\circ=0}z_i}
\leq\sqrt{D_{\mathrm{var}}m_0}.
\]
Since $0\leq v(1-\log v)\leq1$ for $v\in[0,1]$, we have $g(z)\leq m_0$.
Together with $m_0\leq D_{\mathrm{var}}$, this gives $f(z)+g(z)\leq2\sqrt{m_0D_{\mathrm{var}}}\leq2\sqrt{mD}$. Using
$\sum_{n=1}^Hn^{-1/2}\leq2\sqrt H$ in the preceding theorem bound gives
\[
\E C_H\leq100\sqrt{mDH}+116m,
\]
proving~\eqref{eq:pure-adv-phase}.

\paragraph{Step 7: the pure stochastic bound.}
In the stochastic case, the fixed optimal action
is included in Equation~\eqref{eq:pure-majorant-comparator}, and unbiasedness gives
\[
r(q):=\sum_{X\in\widetilde{\mathcal X}_0}q(X)\Delta_X,
\qquad \bar q_n:=\E[\widehat q_n],
\qquad
R_H^{\mathrm{pure}}:=\sum_{n=1}^Hr(\bar q_n)
=\E\sum_{n=1}^H\langle b_n-X^\star,\widetilde\ell_n\rangle
\leq\E C_H.
\]
The marginal of $\bar q_n$ is $\E b_n$, so this
is the regret in Theorem~\ref{thm:hybrid-ftrl}. Let $B_H$ denote
its bound on $\E C_H$. Then $R_H^{\mathrm{pure}}\leq\E C_H\leq B_H$, hence
\[
B_H-R_H^{\mathrm{pure}}\geq B_H-\E C_H\geq0,
\qquad
\E C_H\leq B_H\leq B_H+(B_H-R_H^{\mathrm{pure}})=2B_H-R_H^{\mathrm{pure}}.
\]
Theorem~\ref{thm:hybrid-ftrl} proves the required self-bounding inequality in
Equation~\eqref{eq:hybrid-stochastic-self-bound}. Since $m_0\leq m$ and $D_{\mathrm{var}}\leq D$, it gives
\[
\E C_H\leq c\left[
\frac{mD}{\Delta}\log(eH)+\frac{m^2}{\Delta}+m\right],
\]
proving~\eqref{eq:pure-sto-phase}.

\paragraph{Step 8: random phase endings.}
For the stopped phase length $\tau\leq H$,
\[
I_t:=\1\{t\leq\tau\}=\1\{\tau>t-1\}\in\mathcal F_{t-1}.
\]
Thus the stopped sums retain
the conditional moment bounds.

\emph{Mixed adversarial phase.} Substitute $n=\tau$ in \eqref{eq:pathwise-mixed-certificate}:
\[
U_{1:\tau}
\leq \frac{\log N}{\eta_{\tau+1}}
 +\frac12\sum_{t=1}^{H}I_t\eta_tV_t.
\]
Apply Step~2 using predictable $I_t$ and $\tau\leq H$.

\emph{Mixed stochastic phase.} In the
continuation from Step~3, $r_t\geq0$ gives
\[
\E\sum_{t\leq\tau}r_t\leq R_H,
\qquad
\max_X[D_\tau(X)]_+
\leq\max_X\sup_{n\leq H}[D_n(X)]_+.
\]
Steps~3--4 therefore bound $U_{1:\tau}$.

\emph{Pure phase.} Continue hybrid-FTRL from $\tau$ to $H$,
preserving the run through $\tau$. Let $C_n^{\mathrm{ext}}$ be
its majorant, with $U_{1:0}=C_0=C_0^{\mathrm{ext}}:=0$. Step~5 gives pathwise
\[
U_{1:\tau}\leq C_\tau=C_\tau^{\mathrm{ext}}
\leq C_H^{\mathrm{ext}},
\qquad
\E U_{1:\tau}\leq\E C_H^{\mathrm{ext}}.
\]
After $\tau$, use unused i.i.d. losses and seeds
in the stochastic case; in the adversarial
case, use fresh action seeds and set
\[
g_t^{\mathrm{ext}}:=
\begin{cases}
g_t,&t\leq\tau,\\
0,&t>\tau.
\end{cases}
\]
Predictability of $\{t\leq\tau\}$ preserves non-anticipation. Theorem~\ref{thm:hybrid-ftrl} therefore
applies in both cases. The continuation
is used only in the analysis.
\end{proof}

%
\section{Restarts and phase counting}
\label{sec:stopping-and-counting}

For a fixed horizon $T\geq1$, define
\[
M_T:=\frac{2mT}{\delta},
\qquad
J_T:=3+\left\lceil\log_2\frac{M_T}{L}\right\rceil.
\]

\begin{lemma}[Unbiased estimates and one-round overshoot bounds]
\label{lem:unbiased-and-overshoot}
For every predictable $z_t\in\mathcal P$ and fixed $z\in\mathcal P$,
\[
\E[\langle z_t-z,\widehat\ell_t\rangle\mid\mathcal G_t,\ell_t]
=\langle z_t-z,\ell_t\rangle.
\]
On a round with mixing weight $\alpha\in(0,1]$,
\begin{align}
\alpha\max_{x\in\mathcal P}
\langle b_t-x,\widehat\ell_t\rangle
&\leq L, \label{eq:base-one-round-overshoot}\\
(1-\alpha)\max_{x\in\mathcal P}
\langle x^c-x,\widehat\ell_t\rangle
&\leq L. \label{eq:baseline-one-round-overshoot}
\end{align}
\end{lemma}

\begin{yellowbox}
\textbf{Role.} The two bounds control the terminal-round
costs of the restart tests, while unbiasedness
later converts estimated regret to expected regret.

\textbf{Proof~idea.} Nonnegative estimates let us drop each comparator's
loss. The mixture marginal dominates both $\alpha b_t$
and $(1-\alpha)x^c$, so the corresponding importance-weight ratios
are at most one. Each weighted estimated
loss is therefore bounded by the observed
action loss, which is at most $L$.
\end{yellowbox}

\begin{proof}
Lemma~\ref{lem:importance-unbiasedness} gives
\begin{align*}
\E[\langle z_t-z,\widehat\ell_t\rangle\mid\mathcal G_t,\ell_t]
&=\sum_i(z_{t,i}-z_i)
  \E[\widehat\ell_{t,i}\mid\mathcal G_t,\ell_t]
 =\langle z_t-z,\ell_t\rangle.
\end{align*}
Because $x\geq0$ and $\widehat\ell_t\geq0$ coordinatewise,
\begin{align*}
\max_{x\in\mathcal P}\langle b_t-x,\widehat\ell_t\rangle
&=\langle b_t,\widehat\ell_t\rangle
  -\min_{x\in\mathcal P}\langle x,\widehat\ell_t\rangle
 \leq\langle b_t,\widehat\ell_t\rangle,\\
\max_{x\in\mathcal P}\langle x^c-x,\widehat\ell_t\rangle
&\leq\langle x^c,\widehat\ell_t\rangle.
\end{align*}
Since $x_t\geq\alpha b_t$ and $x_t\geq(1-\alpha)x^c$ coordinatewise,
\begin{align*}
\alpha\langle b_t,\widehat\ell_t\rangle
&=\sum_i\frac{\alpha b_{t,i}}{x_{t,i}}X_{t,i}\ell_{t,i}
 \leq\sum_iX_{t,i}\ell_{t,i}\leq L,\\
(1-\alpha)\langle x^c,\widehat\ell_t\rangle
&=\sum_i\frac{(1-\alpha)x_i^c}{x_{t,i}}X_{t,i}\ell_{t,i}
 \leq\sum_iX_{t,i}\ell_{t,i}\leq L.
\end{align*}
Combining the last two displays
proves Equations~\eqref{eq:base-one-round-overshoot} and \eqref{eq:baseline-one-round-overshoot}.
\end{proof}

\begin{lemma}[Threshold crossing and number of restarts]
\label{lem:threshold-crossing}
At the first crossing of
a stage budget $R_s$,
\[
R_s<U_{u:t}\leq2R_s,
\]
and the next budget is $R_{s+1}=2R_s$. On every
phase interval, $U_{u:t}\leq M_T$. Consequently, Algorithm~\ref{alg:phased-sb} has at most
$J_T$ stages, at most $J_T$ phases per stage,
and at most $J_T^2$ phase intervals in total.
\end{lemma}

\begin{yellowbox}
\textbf{Role.} This lemma limits the number of phases
that require certificates and relates each new
stage budget to its triggering empirical regret.

\textbf{Proof~idea.} The one-round increment is at most $L/\alpha\leq R_s$,
so a first crossing lies in $(R_s,2R_s]$
and the next dyadic budget is exactly
$2R_s$. Coverage in mixed phases and the
pure-phase loss bound give $U_I\leq M_T$. Budget doubling
and $\alpha \le 1$ then give the stated counts.
\end{yellowbox}

\begin{proof}
Let $V_t=\max_x\sum_{j=u}^t\langle b_j-x,\widehat\ell_j\rangle$, so $U_{u:t}=[V_t]_+$. The
max-increment inequality, Equation~\eqref{eq:positive-running-max-increment}, gives
\[
U_{u:t}-U_{u:t-1}
\leq\left[\max_x\langle b_t-x,\widehat\ell_t\rangle\right]_+
\leq\langle b_t,\widehat\ell_t\rangle
\leq\frac L\alpha.
\]
Recall that in phase $k$,
$\alpha=\min\{L2^{k-1}/R_s,1\}$. If $\alpha<1$, then
\[
\frac L\alpha=\frac{R_s}{2^{k-1}}\leq R_s;
\]
if $\alpha=1$, then $L/\alpha=L\leq R_s$. Thus every round satisfies
\[
U_{u:t}-U_{u:t-1}\leq R_s.
\]
At a first crossing, $U_{u:t-1}\leq R_s$ (with $U_{u:u-1}:=0$), so
\[
R_s<U_{u:t}\leq U_{u:t-1}+R_s\leq2R_s.
\]
Since $R_s/L$ is a power of
two, rounding $U_{u:t}\in(R_s,2R_s]$ up to the
next dyadic $L$-multiple gives $R_{s+1}=2R_s$.

For the envelope, nonnegativity gives $U_{u:t}\leq\sum_{j=u}^t\langle b_j,\widehat\ell_j\rangle$. In a
mixed phase, $\alpha<1$ implies $\alpha\leq1/2$, hence $x_{j,i}\geq\delta/2$, $\widehat\ell_{j,i}\leq2/\delta$,
and $\langle b_j,\widehat\ell_j\rangle\leq2m/\delta$. In the pure phase $x_j=b_j$
and this inner product equals the observed
action loss, which is at most $L$.
Since $L\leq m\leq2m/\delta$, both cases give, for $I=\{u,\ldots,t\}$,
\[
U_{u:t}\leq|I|\max\left\{\frac{2m}{\delta},L\right\}
=|I|\frac{2m}{\delta}\leq T\frac{2m}{\delta}=M_T.
\]
Crossing budgets satisfy $R_s<M_T$. Starting at $R_1=2L$ and
doubling at each crossing gives at most
\[
1+\left\lceil\log_2\frac{M_T}{L}\right\rceil.
\]
Including the terminal stage gives
fewer than $J_T$ stages.

The first pure phase in
stage $s$ has index
\[
k_{\mathrm{pure}}=1+\log_2(R_s/L).
\]
Since $R_1=2L$ and later budgets satisfy $R_s=2R_{s-1}<2M_T$,
\[
k_{\mathrm{pure}}
<2+\log_2(M_T/L)
\leq2+\left\lceil\log_2(M_T/L)\right\rceil<J_T.
\]
Comparator-gap restarts stop at $\alpha=1$. Hence there
are at most $J_T$ phases per
stage and $J_T^2$ phases in total.
\end{proof}

\begin{lemma}[Conditional validity after a data-dependent restart]
\label{lem:conditional-restart}
\par\nobreak Pick a stage/phase label $(s,k)$. Let $\sigma+1$
be the first round of this phase. If
the phase is not reached by time $T$,
set $\sigma=T$, denote its length as $\tau$
and set empirical-regret statistic $U_I$ to zero.
Put $A:=\{\sigma<T\}$ and call it reach event.

Then $\sigma$ is a stopping time, $A\in\mathcal F_\sigma$, and
$\tau$ is a bounded stopping time for
$\mathcal H_n:=\mathcal F_{(\sigma+n)\wedge T}$. On $A$, Lemma~\ref{lem:one-phase-empirical-regret} holds conditionally on
$\mathcal F_\sigma$, under Assumption~\ref{ass:non-anticipation} in the adversarial case
and Assumptions~\ref{ass:stochastic} and~\ref{ass:gap} in the stochastic case.
If $C_I$ is its deterministic bound, then
\[
\E[\1_AU_I]\leq\Pr(A)C_I.
\]
\end{lemma}

\begin{yellowbox}
\textbf{Role.} This lemma makes a fresh-run
certificate reusable after adaptive restarts.

\textbf{Proof~idea.} Condition on the history when a phase
starts, where its reach event is already measurable.
Future losses and fresh seeds retain their independent
product law in the stochastic setting; the conditional
sampling identities transfer in the adversarial setting. Apply
the one-phase bound under this conditional law, then
the tower property. The contribution is weighted
by the probability of reaching the phase,
without conditioning on its future exit event.
\end{yellowbox}

\begin{proof}
The next state and threshold tests are
$\mathcal F_t$-measurable: by finiteness of $\mathcal X$ and linearity,
\[
U_{u:t}
=\left[\max_{X\in\mathcal X}
  \sum_{j=u}^t\langle b_j-X,\widehat\ell_j\rangle\right]_+,
\qquad
G^c_{u:t}=\max_{X\in\mathcal X}
  \sum_{j=u}^t\langle x^c-X,\widehat\ell_j\rangle.
\]
Hence $\{\sigma=t\}\in\mathcal F_t$ and $A\in\mathcal F_\sigma$. The first after-round
exit test, capped at $T-\sigma$, makes
$\tau$ a stopping time for $(\mathcal H_n)$.

Under Assumption~\ref{ass:stochastic}, $\zeta_t=(\ell_t,\xi_t)$ are i.i.d. with law
$\mu=P\otimes\nu$ and future blocks are independent of
$\mathcal F_s$ at deterministic $s$. For $B\in\mathcal F_\sigma$ and
bounded Borel $f$, splitting over $\sigma$ gives
\begin{align*}
\E[\1_Bf(\zeta_{\sigma+1},\ldots,\zeta_{\sigma+r})]
&=\sum_{s=0}^T
 \E[\1_{B\cap\{\sigma=s\}}f(\zeta_{s+1},\ldots,\zeta_{s+r})]\\
&=\Pr(B)\int f\,d\mu^{\otimes r}.
\end{align*}
The second equality uses $B\cap\{\sigma=s\}\in\mathcal F_s$. Thus the future
losses and seeds retain their independent i.i.d. laws
conditional on $\mathcal F_\sigma$, as required by Lemma~\ref{lem:one-phase-empirical-regret}. Fresh
independent pairs beyond $T$ supply any virtual continuation
without changing $U_I$; the restarted learner retains
its specified mixed or pure sampling rule.

For a non-anticipating adversary, take
bounded $\sigma(\mathcal G_{s+n},\ell_{s+n})$-measurable $W_s$. Equation~\eqref{eq:fresh-draw} gives
\begin{align*}
\E[\1_{\{\sigma=s\}}W_s(X_{s+n,i}-x_{s+n,i})]
&=\E[\1_{\{\sigma=s\}}W_s
 \E[X_{s+n,i}-x_{s+n,i}\mid\mathcal G_{s+n},\ell_{s+n}]]=0.
\end{align*}
Summing over $s$ transfers the conditional marginal identity
to local rounds and, by the tower property,
to retained original coordinates. Zero-loss dummy coordinates satisfy~\eqref{eq:hybrid-direct-sampling}
automatically. The conditional moment bounds therefore persist; a
continuation uses zero losses and fresh draws
with the announced marginals. In either regime,
Lemma~\ref{lem:one-phase-empirical-regret} yields $\E[U_I\mid\mathcal F_\sigma]\leq C_I$ on $A$. Since $A\in\mathcal F_\sigma$,
\[
\E[\1_AU_I]
=\E[\1_A\E[U_I\mid\mathcal F_\sigma]]
\leq C_I\E[\1_A]=\Pr(A)C_I.
\]

\end{proof}

\begin{corollary}[Total empirical-regret bound across phases]
\label{cor:total-phase-empirical-regret}
Assume Assumptions~\ref{ass:non-anticipation} and~\ref{ass:coverage}. For $T\geq1$, let
\[
J_T:=3+\left\lceil\log_2\frac{2mT}{L\delta}\right\rceil.
\]
For the at most $J_T^2$ intervals generated by
\Duet, the entropy-FTRL and hybrid-FTRL learners satisfy the
following adversarial and stochastic bounds in Definition~\ref{def:base-interface}:
\begin{align}
\mathfrak C_{\mathrm{adv}}(T)
&\leq CmJ_T^2\sqrt{\frac{T\Lambda}{\delta}},
\label{eq:total-adv-empirical-regret}\\
\mathfrak C_{\mathrm{sto}}(T,\Delta)
&\leq C\frac{m^2}{\delta\Delta}J_T^2
 [\Lambda+\log(eT)].
\label{eq:total-sto-empirical-regret}
\end{align}
The second inequality holds if Assumptions~\ref{ass:stochastic} and~\ref{ass:gap} hold.
\end{corollary}

\begin{yellowbox}
\textbf{Role.} This corollary verifies the total empirical-regret input
required by Theorem~\ref{thm:phasewise-to-regret} for the specified learners.

\textbf{Proof~idea.} Index all possible phases by a deterministic
set of at most $J_T^2$ labels, assigning zero
to unreached labels. Use $H=T$ in each conditional
restart certificate and multiply by its reach probability
before summing. The coverage relation $d\delta\leq m$ puts mixed
and pure certificates on the same parameter scale.
\end{yellowbox}

\begin{proof}
Let $\mathcal L_T$ be the deterministic set of
stage/phase labels that can occur before time
$T$. Since $R_1=2L$ and every stage crossing
doubles the budget, the $s$th possible stage
has $R_s=2^sL$. The counting lemma gives
\[
\mathcal L_T\subseteq\{1,\ldots,J_T\}\times\{1,\ldots,J_T\},
\qquad |\mathcal L_T|\leq J_T^2,
\]
For each $\lambda=(s,k)$, let $A_\lambda$ be its
reach event and set $U_\lambda=0$ off
$A_\lambda$. For the realized intervals $\mathcal I$,
\[
\sum_{I\in\mathcal I}U_I
=\sum_{\lambda\in\mathcal L_T}U_\lambda.
\]
On $A_\lambda$, the start time
$\sigma_\lambda$ and length $\tau_\lambda$ satisfy
\[
\tau_\lambda\leq T-\sigma_\lambda\leq T.
\]
Apply Lemma~\ref{lem:conditional-restart} with $H=T$ and
deterministic bound $C_\lambda$ to obtain
\begin{align*}
\E\left[\sum_{I\in\mathcal I}U_I\right]
&=\sum_{\lambda\in\mathcal L_T}
  \E[\1_{A_\lambda}U_\lambda]
\leq\sum_{\lambda\in\mathcal L_T}
  \Pr(A_\lambda)C_\lambda
\leq J_T^2\max_{\lambda\in\mathcal L_T}C_\lambda.
\end{align*}
Coverage gives
\[
d\delta\leq\sum_{i=1}^d x_i^c
=\E_{X\sim q^c}\lVert X\rVert_1\leq m.
\]
Since $m\leq d$,
\[
D\leq d+m\leq2d\leq\frac{2m}{\delta}.
\]
Because $\rho=\delta/2$, the mixed adversarial bound
is at most $Cm\sqrt{T\Lambda/\delta}$. For a
pure interval, $T,\Lambda\geq1$ and $\delta\leq1$ give
\[
\sqrt{mDT}+m
\leq\sqrt2\,m\sqrt{T/\delta}+m
\leq(\sqrt2+1)m\sqrt{T\Lambda/\delta}.
\]
This proves Equation~\eqref{eq:total-adv-empirical-regret} after multiplication by $J_T^2$.

For the stochastic scale, the mixed
bound becomes $2c_2m^2\Lambda/(\delta\Delta)$. On a pure
interval, $\Delta\leq L\leq m$, $\delta\leq1$, and $D\leq2m/\delta$, whence
\[
\frac{mD}{\Delta}\log(eT)
\leq\frac{2m^2}{\delta\Delta}\log(eT),
\qquad
\frac{m^2}{\Delta}\leq\frac{m^2}{\delta\Delta},
\]
and $m\leq m^2/(\delta\Delta)$. Since $\Lambda\geq1$, all
three terms are at most
\[
C\frac{m^2}{\delta\Delta}[\Lambda+\log(eT)].
\]
Multiplying by $J_T^2$ proves Equation~\eqref{eq:total-sto-empirical-regret}.
\end{proof}
%
\section{Main regret bounds and proofs}
\label{sec:regret-phased}

\begin{theorem}[Regret bound from phasewise empirical-regret control]
\label{thm:phasewise-to-regret}
Under Assumptions~\ref{ass:non-anticipation} and~\ref{ass:coverage}, run the mixing scheme underlying
\Duet{} with base learners satisfying Definition~\ref{def:base-interface}. For $N\geq2$
and $T\geq1$,
\[
\mathcal R_T(x^c)
\leq O\!\left(L\log\!\left(2+\frac{mT}{L\delta}\right)\right).
\]
With
\[
J_T:=3+\left\lceil\log_2\frac{2mT}{L\delta}\right\rceil,
\]
\begin{align*}
\sup_{x\in\cP}\mathcal R_T(x)
&\leq(1+8J_T)\mathfrak C_{\mathrm{adv}}(T)
      +L(J_T^2+8J_T),\\
\sup_{x\in\cP}\mathcal R_T(x)
&\leq(1+8J_T)\mathfrak C_{\mathrm{sto}}(T,\Delta)
      +L(J_T^2+8J_T).
\end{align*}
The adversarial bound uses~\eqref{eq:interface-adv}; the stochastic bound
uses Assumptions~\ref{ass:stochastic} and~\ref{ass:gap} and~\eqref{eq:interface-sto}. Corollary~\ref{cor:total-phase-empirical-regret} verifies both
empirical-regret premises for \Duet{}. The mixing scheme needs
neither the regime nor $\Delta$.
\end{theorem}

\begin{yellowbox}
\textbf{Role.} This theorem converts the base learner's empirical-regret
certificate into expected regret for \Duet{}.
\textbf{Proof~idea.} Sum baseline gains and losses within
stages, and charge stage thresholds to the
phases that triggered them. Apply unbiasedness and
the certificate bound after this pathwise accounting,
then take the supremum over fixed comparators.
\end{yellowbox}

\subsection{Proof of the general \Duet{} bound}
\label{sec:proof-general-phased-bound}

The next two lemmas establish the pathwise
accounting for Theorem~\ref{thm:phasewise-to-regret}; unbiasedness is applied afterward.

\begin{lemma}[Baseline regret within one stage]
\label{lem:baseline-one-stage}
For every stage $s$, the total estimated regret
to $x^c$ is at most $2L$. Hence
\[
\mathcal R_T(x^c)\leq2LJ_T.
\]
\end{lemma}

\begin{yellowbox}
\textbf{Role.} Baseline protection does not require stochastic certificates.
\textbf{Proof~idea.} Test priority gives each baseline-gap exit estimated
regret below $-3r_k$, where $r_k=\alpha_{s,k}R_s$. These geometric gains
cover the terminal cost of at most $2r_k$,
leaving $2L$ per stage. Summing stages and applying
unbiasedness at deterministic rounds gives the expected bound.
\end{yellowbox}

\begin{proof}
In stage $s$, set $r_k:=\alpha_{s,k}R_s$, where $\alpha_{s,k}$
is the phase's fixed mixture weight. Algorithm~\ref{alg:phased-sb}
stops increasing $k$ once $\alpha_{s,k}=1$. Since $R_s/L$
is dyadic, every reached phase has
$r_k=L2^{k-1}$. For its interval $I$, define
\[
A_I(x):=\sum_{t\in I}\langle b_t-x,\widehat\ell_t\rangle,
\qquad
\widehat B_I:=\sum_{t\in I}
  \langle x_t-x^c,\widehat\ell_t\rangle.
\]
Since $x_t=\alpha_{s,k}b_t+(1-\alpha_{s,k})x^c$ throughout the phase,
\begin{equation}
\widehat B_I=\alpha_{s,k}A_I(x^c).
\label{eq:baseline-mixture-identity}
\end{equation}
If the phase exits through the comparator-gap
test, the stage test did not fire,
so $U_I\leq R_s$. Let $x_I$ attain $G_I^c$. Then
\[
A_I(x^c)=A_I(x_I)-G_I^c\leq U_I-G_I^c<-3R_s.
\]
Equation~\eqref{eq:baseline-mixture-identity} gives
\begin{equation}
\widehat B_I<-3\alpha_{s,k}R_s=-3r_k
\label{eq:completed-phase-negative-baseline}
\end{equation}

Every nonterminal phase exits this way. For
the terminal interval $I_k$, $A_{I_k}(x^c)\leq U_{I_k}$. At
a stage crossing, Lemma~\ref{lem:threshold-crossing} and~\eqref{eq:baseline-mixture-identity} give
\begin{equation}
\widehat B_{I_k}
=\alpha_{s,k}A_{I_k}(x^c)
\leq\alpha_{s,k}U_{I_k}
\leq2\alpha_{s,k}R_s=2r_k.
\label{eq:crossing-terminal-phase}
\end{equation}
If the deterministic horizon ends before a
stage crossing, then $U_{I_k}\leq R_s$ and instead
\begin{equation}
\widehat B_{I_k}
\leq\alpha_{s,k}U_{I_k}
\leq\alpha_{s,k}R_s=r_k.
\label{eq:horizon-terminal-phase}
\end{equation}
The dyadic schedule gives the exact geometric sum
\begin{equation}
\sum_{j<k}r_j
=L\sum_{j=1}^{k-1}2^{j-1}
=L(2^{k-1}-1)=r_k-L.
\label{eq:phase-geometric-sum}
\end{equation}
A crossing in the first phase costs at
most $2r_1=2L$ by~\eqref{eq:crossing-terminal-phase}. For $k\geq2$, Equation~\eqref{eq:completed-phase-negative-baseline} gives
\[
\sum_{I\text{ in stage }s}\widehat B_I
<-3\sum_{j<k}r_j+2r_k
=-3(r_k-L)+2r_k=3L-r_k\leq2L,
\]
using $r_k\geq L$. If the stage
ends at the horizon instead,
\[
\sum_{I\text{ in stage }s}\widehat B_I
\leq-3\sum_{j<k}r_j+r_k
=-3(r_k-L)+r_k=3L-2r_k\leq L.
\]
There are at most $J_T$ stages, so pathwise
\begin{equation}
\sum_{t=1}^T\langle x_t-x^c,\widehat\ell_t\rangle
=\sum_s\sum_{I\text{ in stage }s}\widehat B_I
\leq2LJ_T.
\label{eq:pathwise-baseline-regret}
\end{equation}
Since $x_t$ is $\mathcal F_{t-1}$-measurable, unbiasedness (Lemma~\ref{lem:unbiased-and-overshoot}) and the
tower property at each deterministic round give
\begin{align*}
\mathcal R_T(x^c)
&=\sum_{t=1}^T\E\langle x_t-x^c,\ell_t\rangle\\
&=\sum_{t=1}^T\E\!\left[
 \E[\langle x_t-x^c,\widehat\ell_t\rangle
      \mid\mathcal G_t,\ell_t]\right]\\
&=\E\sum_{t=1}^T\langle x_t-x^c,\widehat\ell_t\rangle
\leq2LJ_T,
\end{align*}
by~\eqref{eq:pathwise-baseline-regret}.
\end{proof}

\begin{lemma}[Fixed-comparator regret over all phases]
\label{lem:fixed-comparator-all-phases}
Let $\mathcal I$ be the phase intervals
produced by the algorithm and put
$U_T:=\sum_{I\in\mathcal I}U_I$. For every fixed $x\in\mathcal P$,
\begin{equation}
\mathcal R_T(x)
\leq(1+8J_T)\E[U_T]+L(J_T^2+8J_T).
\label{eq:fixed-comparator-all-phases}
\end{equation}
\end{lemma}

\begin{yellowbox}
\textbf{Role.} This lemma reduces fixed-comparator regret to
$\E U_T$, where $U_T=\sum_I U_I$. \textbf{Proof~idea.} Each phase costs
at most $U_I+4R_{s(I)}+L$ in estimated regret. Charge
each threshold to its triggering phase; it
appears in at most $J_T$ phases. Sum
these pathwise bounds and apply unbiasedness.
\end{yellowbox}

\begin{proof}
Fix $x\in\mathcal P$. On an interval $I$ in stage
$s$, let $t_I:=\max I$, put $I^-:=I\setminus\{t_I\}$, and define
\[
A_I(x):=\sum_{t\in I}\langle b_t-x,\widehat\ell_t\rangle,
\qquad
C_I(x):=\sum_{t\in I}\langle x^c-x,\widehat\ell_t\rangle.
\]
The phase weight $\alpha$ is
constant on $I$, and hence
\begin{equation}
\sum_{t\in I}\langle x_t-x,\widehat\ell_t\rangle
=\alpha A_I(x)+(1-\alpha)C_I(x).
\label{eq:interval-regret-decomposition}
\end{equation}
By the definition of $U_I$,
\begin{equation}
\alpha A_I(x)
\leq\alpha\max_{y\in\mathcal P}A_I(y)
\leq\alpha U_I\leq U_I.
\label{eq:interval-base-bound}
\end{equation}
If $\alpha<1$, the comparator-gap test had not
fired before round $t_I$. With the
convention that the sum over an
empty $I^-$ is zero, this gives
\begin{equation}
\max_{y\in\mathcal P}\sum_{t\in I^-}
\langle x^c-y,\widehat\ell_t\rangle\leq4R_s.
\label{eq:preterminal-gap-test}
\end{equation}
The terminal round may trigger the test. Combining
\eqref{eq:preterminal-gap-test} on $I^-$ with the one-round bound~\eqref{eq:baseline-one-round-overshoot} gives
\begin{align}
(1-\alpha)C_I(x)
&=(1-\alpha)\sum_{t\in I^-}
  \langle x^c-x,\widehat\ell_t\rangle
 +(1-\alpha)\langle x^c-x,\widehat\ell_{t_I}\rangle\notag\\
&\leq4(1-\alpha)R_s
 +(1-\alpha)\max_{y\in\mathcal P}
    \langle x^c-y,\widehat\ell_{t_I}\rangle\notag\\
&\leq4R_s+L.
\label{eq:interval-baseline-bound}
\end{align}
The bound also holds for $\alpha=1$, when its
left side vanishes. Combining~\eqref{eq:interval-regret-decomposition}, \eqref{eq:interval-base-bound}, and~\eqref{eq:interval-baseline-bound} gives
\begin{equation}
\sum_{t\in I}\langle x_t-x,\widehat\ell_t\rangle
\leq U_I+4R_s+L.
\label{eq:one-interval-bound}
\end{equation}

Let $\mathcal S_T$ contain the reached stages and
$\mathcal I_s$ their phase intervals. The initial budget
is $R_1=2L$. For $s\geq2$, the crossing interval
$I_{s-1}^{\times}$ that creates stage $s$ satisfies
\[
U_{I_{s-1}^{\times}}>R_{s-1},
\qquad
R_s=2R_{s-1}<2U_{I_{s-1}^{\times}}.
\]
The crossing intervals are distinct, so pathwise
\begin{align}
\sum_{s\in\mathcal S_T}R_s
&=R_1+\sum_{\substack{s\in\mathcal S_T\\s\geq2}}R_s
\leq2L+2\sum_{\substack{s\in\mathcal S_T\\s\geq2}}
      U_{I_{s-1}^{\times}}
\leq2L+2\sum_{I\in\mathcal I}U_I
=2L+2U_T.
\label{eq:stage-budget-chain}
\end{align}
Since $|\mathcal I_s|\leq J_T$ for every stage,
\begin{equation}
\sum_{I\in\mathcal I}R_{s(I)}
=\sum_{s\in\mathcal S_T}|\mathcal I_s|R_s
\leq J_T\sum_{s\in\mathcal S_T}R_s
\leq J_T(2L+2U_T).
\label{eq:interval-stage-budget-bound}
\end{equation}
Sum~\eqref{eq:one-interval-bound} and apply $|\mathcal I|\leq J_T^2$ and~\eqref{eq:interval-stage-budget-bound}:
\begin{align}
\sum_{t=1}^T\langle x_t-x,\widehat\ell_t\rangle
&\leq\sum_{I\in\mathcal I}U_I
 +4\sum_{I\in\mathcal I}R_{s(I)}+L|\mathcal I|\notag\\
&\leq U_T+4J_T(2L+2U_T)+LJ_T^2\notag\\
&=(1+8J_T)U_T+L(J_T^2+8J_T).
\label{eq:pathwise-all-phase-accounting}
\end{align}
For fixed $x$ and predictable $x_t$,
unbiasedness and the tower property give
\begin{align*}
\mathcal R_T(x)
&=\sum_{t=1}^T\E\langle x_t-x,\ell_t\rangle\\
&=\sum_{t=1}^T\E\!\left[
 \E[\langle x_t-x,\widehat\ell_t\rangle
      \mid\mathcal G_t,\ell_t]\right]\\
&=\E\sum_{t=1}^T\langle x_t-x,\widehat\ell_t\rangle\\
&\leq(1+8J_T)\E[U_T]+L(J_T^2+8J_T),
\end{align*}
where the last inequality is Equation~\eqref{eq:pathwise-all-phase-accounting}.
\end{proof}

\begin{proof}[Proof of Theorem~\ref{thm:phasewise-to-regret}]
Lemma~\ref{lem:baseline-one-stage} gives $\mathcal R_T(x^c)\leq2LJ_T$. Since $mT/(L\delta)\geq1$,
\[
J_T=3+\left\lceil\log_2\frac{2mT}{L\delta}\right\rceil
\leq4+\log_2\frac{2mT}{L\delta}
=O\!\left(\log\!\left(2+\frac{mT}{L\delta}\right)\right).
\]
For a fixed deterministic comparator $x\in\mathcal P$,
\begin{align}
\mathcal R_T(x)
&=\sum_{t=1}^T\E\langle x_t-x,\ell_t\rangle\notag\\
&=\sum_{t=1}^T\E\!\left[
 \E[\langle x_t-x,\widehat\ell_t\rangle
      \mid\mathcal G_t,\ell_t]\right]\notag\\
&=\E\sum_{t=1}^T\langle x_t-x,\widehat\ell_t\rangle\notag\\
&\leq(1+8J_T)\E\!\left[\sum_{I\in\mathcal I}U_I\right]
  +L(J_T^2+8J_T),
\label{eq:tower-before-comparator}
\end{align}
by~\eqref{eq:pathwise-all-phase-accounting}. Substituting \eqref{eq:interface-adv} gives a
bound independent of $x$, so
\begin{align*}
\sup_{x\in\cP}\mathcal R_T(x)
&\leq\sup_{x\in\mathcal P}
 \{(1+8J_T)\mathfrak C_{\mathrm{adv}}(T)+L(J_T^2+8J_T)\}\\
&=(1+8J_T)\mathfrak C_{\mathrm{adv}}(T)+L(J_T^2+8J_T).
\end{align*}
Likewise, substituting~\eqref{eq:interface-sto} into~\eqref{eq:tower-before-comparator} and then
taking the supremum gives
\[
\sup_{x\in\cP}\mathcal R_T(x)
\leq(1+8J_T)\mathfrak C_{\mathrm{sto}}(T,\Delta)
  +L(J_T^2+8J_T).
\]

\end{proof}

\begin{corollary}[Regret bounds for \Duet]
\label{cor:complete-regret-bounds}
Under Assumptions~\ref{ass:non-anticipation} and~\ref{ass:coverage}, run \Duet{} with the
learners of Section~\ref{sec:base-learner}, without knowing $T$,
$\Delta$, or the regime. For $N\geq2$ and
each deterministic $T\geq1$, with $J_T$ from Theorem~\ref{thm:phasewise-to-regret},
\[
\mathcal R_T(x^c)\leq2LJ_T.
\]
For a universal constant $C$,
\begin{align}
\sup_{x\in\cP}\mathcal R_T(x)
&\leq CmJ_T^3\sqrt{\frac{T\log(eN)}{\delta}}+CLJ_T^2,
\label{eq:complete-adversarial-regret}\\
\sup_{x\in\cP}\mathcal R_T(x)
&\leq C\frac{m^2}{\delta\Delta}J_T^3
[\log(eN)+\log(eT)]+CLJ_T^2.
\label{eq:complete-stochastic-regret}
\end{align}
Equation~\eqref{eq:complete-adversarial-regret} holds against every non-anticipating adversary.
Equation~\eqref{eq:complete-stochastic-regret} holds under Assumptions~\ref{ass:stochastic} and~\ref{ass:gap}. The
implementation enumerates all $N=|\cX|$ actions.
\end{corollary}

\begin{yellowbox}
\textbf{Role.} This gives all three guarantees in
Theorem~\ref{thm:main-result} for \Duet{}. \textbf{Proof~idea.} Substitute
the conditional phase-certificate totals into the transfer
theorem; its accounting already supplies baseline protection.
\end{yellowbox}

\begin{proof}
The one-phase analysis, Lemma~\ref{lem:conditional-restart}, and Corollary~\ref{cor:total-phase-empirical-regret} verify
Definition~\ref{def:base-interface}, including its stopped empirical maxima.
The transfer theorem therefore gives $\mathcal R_T(x^c)\leq2LJ_T$.

For any fixed $x\in\mathcal P$, substitute Equation~\eqref{eq:total-adv-empirical-regret} into
the first inequality of Theorem~\ref{thm:phasewise-to-regret}. Since $J_T\geq3$,
\[
1+8J_T\leq9J_T,
\qquad
J_T^2+8J_T\leq4J_T^2.
\]
Consequently
\[
\mathcal R_T(x)
\leq9C_0mJ_T^3\sqrt{\frac{T\Lambda}{\delta}}+4LJ_T^2,
\]
where $C_0$ is the phase-bound constant. The
bound is independent of $x$; taking the
supremum proves~\eqref{eq:complete-adversarial-regret} after enlarging the universal constant.

Under Assumptions~\ref{ass:stochastic} and~\ref{ass:gap}, substituting \eqref{eq:total-sto-empirical-regret}
into Theorem~\ref{thm:phasewise-to-regret} similarly gives
\[
\sup_{x\in\cP}\mathcal R_T(x)
\leq9C_0\frac{m^2}{\delta\Delta}J_T^3
 [\Lambda+\log(eT)]+4LJ_T^2,
\]
which is Equation~\eqref{eq:complete-stochastic-regret}.

\end{proof}
 \appendix
%
\section{Hybrid-FTRL majorant bounds}
\label{app:hybrid-ftrl}

We prove the pure-phase stability, penalty, and self-bounding
estimates for the hybrid-FTRL construction of \citet{mainbib:zimmert2019beating}.

\begin{theorem}[Hybrid-FTRL majorant and self-bounding bounds]
\label{thm:hybrid-ftrl}
Let $\mathcal A\subseteq\{0,1\}^{D_{\mathrm{var}}}$ contain at least two distinct
actions, all of cardinality $m_0$, with no
coordinate constant on $\mathcal A$. Put $\mathcal Q:=\conv(\mathcal A)$ and
\[
  \Psi(z):=\sum_{i=1}^{D_{\mathrm{var}}}
  \bigl[-\sqrt{z_i}+(1-z_i)\log(1-z_i)\bigr],
  \qquad \eta_t:=t^{-1/2},
\]
using the continuous convention $0\log0=0$.
Starting with $\widehat L_0=0$, run
\[
  z_t:=\arg\min_{z\in\mathcal Q}
  \{\langle z,\widehat L_{t-1}\rangle+\sqrt t\,\Psi(z)\}.
\]
Lemma~\ref{lem:pure-interior} gives a unique optimizer with $0<z_{t,i}<1$.
Sample $A_t\in\mathcal A$ from a predictable distribution $q_t$
with marginal $z_t$. Let $\lambda_t\in[-1,1]^{D_{\mathrm{var}}}$ satisfy
\begin{equation}
  (\lambda_{t,i}+1)
  \bigl(\mathbb E[A_{t,i}\mid\mathcal F_{t-1},\lambda_t]-z_{t,i}\bigr)=0,
  \qquad i\in[D_{\mathrm{var}}].
  \label{eq:hybrid-direct-sampling}
\end{equation}
Losses may be adaptive. The condition is
vacuous when $\lambda_{t,i}=-1$, since the estimate is
then deterministically $-1$. From semi-bandit feedback, define
\begin{equation}
  \widehat\lambda_{t,i}:=\frac{A_{t,i}(\lambda_{t,i}+1)}{z_{t,i}}-1,
  \qquad
  \widehat L_t:=\widehat L_{t-1}+\widehat\lambda_t.
  \label{eq:hybrid-shifted-estimator}
\end{equation}
Unselected arms receive estimate $-1$.
Fix $a^\star\in\mathcal A$ and define
\[
  f(z):=\sum_{i:a_i^\star=0}\sqrt{z_i},
  \qquad
  g(z):=\sum_{i:a_i^\star=1}(1-z_i)\bigl(1-\log(1-z_i)\bigr),
\]
with continuous boundary values, and
\[
  \Phi_t(y):=\max_{z\in\mathcal Q}\{\langle z,y\rangle-\sqrt t\,\Psi(z)\}.
\]
For every integer $H\geq1$, put
\begin{align*}
  S_H&:=\sum_{t=1}^H
  \left[\langle z_t,\widehat\lambda_t\rangle
  +\Phi_t(-\widehat L_t)-\Phi_t(-\widehat L_{t-1})\right],\\
  P_H&:=-\Psi(z_1)-m_0
  +\sum_{t=2}^H(\sqrt t-\sqrt{t-1})[-\Psi(z_t)-m_0],
  \qquad C_H:=S_H+P_H.
\end{align*}
Then
\begin{equation}
  \mathbb E C_H\leq B_H
  :=\sum_{t=1}^H\frac{25}{\sqrt t}
    [f(\mathbb E z_t)+g(\mathbb E z_t)]+116m_0.
  \label{eq:hybrid-BH}
\end{equation}
If, in addition, the losses are i.i.d.
and independent of the learner's randomization, and
$a^\star$ is the unique mean-optimal action, define
\[
  \Delta_a:=\langle a-a^\star,\mathbb E\lambda_t\rangle,
  \qquad \Delta:=\min_{a\neq a^\star}\Delta_a>0,
  \qquad R_H^{\mathrm{pure}}:=\sum_{t=1}^H\sum_a\mathbb E[q_t(a)]\Delta_a.
\]
Then $R_H^{\mathrm{pure}}\leq\mathbb E C_H\leq B_H$ and, for a universal constant $c$,
\begin{equation}
  \mathbb E C_H\leq2B_H-R_H^{\mathrm{pure}}
  \leq c\left[
    \frac{m_0D_{\mathrm{var}}}{\Delta}\log(eH)
    +\frac{m_0^2}{\Delta}+m_0\right].
  \label{eq:hybrid-stochastic-self-bound}
\end{equation}
\end{theorem}

\begin{yellowbox}
\textbf{Role.} $C_H$ controls the pure phase's stopped empirical
maximum. \textbf{Proof~idea.} Equal cardinality cancels the loss shift
and makes penalty increments nonnegative. Stability and penalty
bounds give $\E C_H\leq B_H$; stochastic self-bounding controls $2B_H-R_H^{\mathrm{pure}}$.
Monotonicity of the continued majorant gives $U_{1:\tau}\leq C_\tau^{\mathrm{ext}}\leq C_H^{\mathrm{ext}}$.
\end{yellowbox}

\begin{proof}
Write $\mathbb E_t[\cdot]:=\mathbb E[\cdot\mid\mathcal F_{t-1},\lambda_t]$, $\mu_{t,i}:=\mathbb E_t A_{t,i}$, and $y_{t,i}:=\lambda_{t,i}+1$. Equation~\eqref{eq:hybrid-direct-sampling}
gives $y_{t,i}\mu_{t,i}=y_{t,i}z_{t,i}$, hence conditional unbiasedness:
\[
  \mathbb E_t\widehat\lambda_{t,i}
  =\frac{y_{t,i}\mu_{t,i}}{z_{t,i}}-1=\lambda_{t,i}.
\]
Equal cardinality also gives $\sum_i A_{t,i}=\sum_i z_{t,i}=m_0$, so
\[
 \mathbb E_t\langle A_t,\lambda_t\rangle
 =\sum_i y_{t,i}\mu_{t,i}-m_0
 =\sum_i y_{t,i}z_{t,i}-m_0
 =\langle z_t,\lambda_t\rangle
 =\mathbb E_t\langle z_t,\widehat\lambda_t\rangle.
\]
Taking expectations therefore yields
\begin{equation}
  \mathbb E S_H
  =\mathbb E\sum_{t=1}^H
    \left[\langle A_t,\lambda_t\rangle
    +\Phi_t(-\widehat L_t)-\Phi_t(-\widehat L_{t-1})\right].
  \label{eq:hybrid-stability}
\end{equation}
Since $\Psi(a)=-m_0$, the pathwise decomposition
\eqref{eq:pure-ftrl-decomposition}--\eqref{eq:pure-majorant-comparator} gives $C_H\geq0$ and
\[
  \sum_{t=1}^H\langle z_t-a,\widehat\lambda_t\rangle\leq C_H
  \qquad(a\in\mathcal A).
\]
Under stochastic losses, unbiasedness and independence
give $R_H^{\mathrm{pure}}\leq\mathbb E C_H$. We bound the penalty,
stability, and self-bounding terms in turn.

\paragraph{Penalty bound.}
Set
\[
  w_t:=\sqrt{t}-\sqrt{t-1},
  \qquad
  \Gamma(z):=-\Psi(z)-m_0,
\]
so that $P_H=\sum_{t=1}^H w_t\Gamma(z_t)$. Since $\sum_i z_i=m_0$,
\[
  \Gamma(z)
  =\sum_i\left[\sqrt{z_i}-z_i-(1-z_i)\log(1-z_i)\right]\geq0.
\]
We use the elementary scalar inequality
\[
  -(1-v)\log(1-v)\leq\frac12\sqrt v,
  \qquad 0\leq v\leq1.
\]
For $a\geq1$, $a-a^{-1}-2\log a\geq0$: its value at one is
zero and its derivative is $(a-1)^2/a^2$. Set
$u=1-v>0$ and $a=u^{-1/2}$ to obtain $-\log u\leq(1-u)/\sqrt u$, hence
\[
  -u\log u\leq\sqrt u(1-u)
  =\sqrt v\sqrt{u(1-u)}\leq\frac12\sqrt v;
\]
the endpoint $u=0$ follows by continuity.

Write $h(v):=-(1-v)\log(1-v)$. Since $m_0=\sum_i a_i^\star$, separating the
coordinates according to $a_i^\star\in\{0,1\}$ gives
\begin{align*}
  \Gamma(z)
  &=\sum_{i:a_i^\star=0}\bigl[\sqrt{z_i}+h(z_i)\bigr]
    +\sum_{i:a_i^\star=1}\bigl[\sqrt{z_i}-1+h(z_i)\bigr]\\
  &\leq\frac32\sum_{i:a_i^\star=0}\sqrt{z_i}
    +\sum_{i:a_i^\star=1}h(z_i)\\
  &\leq\frac32\left[
      \sum_{i:a_i^\star=0}\sqrt{z_i}
      +\sum_{i:a_i^\star=1}h(z_i)
    \right].
\end{align*}
Here we use $h(v)\leq\sqrt v/2$, $\sqrt v\leq1$, and $h(v)\geq0$. Also,
\[
  0\leq w_t=\frac{1}{\sqrt t+\sqrt{t-1}}\leq\frac1{\sqrt t}.
\]
Consequently,
\[
  \mathbb{E}P_H
  \leq\sum_{t=1}^H\frac{3}{2\sqrt t}\left[
    \sum_{i:a_i^\star=0}\mathbb{E}\sqrt{z_{t,i}}
    +\sum_{i:a_i^\star=1}\mathbb{E}h(z_{t,i})
  \right].
\]
Both $v\mapsto\sqrt v$ and $h$ are concave
on $[0,1]$. Thus Jensen's inequality gives
\[
  \mathbb{E}\sqrt{z_{t,i}}\leq\sqrt{\mathbb{E}z_{t,i}},
  \qquad
  \mathbb{E}h(z_{t,i})\leq h(\mathbb{E}z_{t,i})=-u_{t,i}\log u_{t,i},
  \qquad
  u_{t,i}:=1-\mathbb{E}z_{t,i},
\]
and hence
\begin{equation}
  \mathbb{E}P_H
  \leq\sum_{t=1}^H\frac{3}{2\sqrt{t}}
  \left[
    \sum_{i:a_i^\star=0}\sqrt{\mathbb{E}z_{t,i}}
    +\sum_{i:a_i^\star=1}(-u_{t,i}\log u_{t,i})
  \right],
\label{eq:hybrid-penalty-bound}
\end{equation}

\paragraph{Stability bound.}
Let
\[
  s_t:=\langle z_t,\widehat\lambda_t\rangle
  +\Phi_t(-\widehat L_t)-\Phi_t(-\widehat L_{t-1}),
  \qquad
  \Psi_t:=\sqrt t\,\Psi,
  \qquad
  \eta_t:=t^{-1/2}.
\]
For a differentiable convex function $F$, write $D_F(x,y):=F(x)-F(y)-\langle\nabla F(y),x-y\rangle$.
The FTRL optimality relation is $z_t=\nabla\Phi_t(-\widehat L_{t-1})$, and therefore
\[
  s_t=D_{\Phi_t}(-\widehat L_t,-\widehat L_{t-1}).
\]
Put $y_t:=-\widehat L_{t-1}$, $v_t:=\nabla\Psi_t(z_t)$, and let $\Psi_t^*$ be the
separable conjugate of the continuous extension of
$\Psi_t$ to $[0,1]^{D_{\mathrm{var}}}$. Its optimizer is interior
for finite dual vectors. To relax $\mathcal Q$
to this box, use optimality of $z_t$:
\[
  \langle w-z_t,y_t-v_t\rangle\leq0
  \qquad (w\in\mathcal{Q}).
\]
Substitution into the conjugate representations yields
\begin{align*}
  D_{\Phi_t}(y_t-\widehat\lambda_t,y_t)
  &=\sup_{w\in\mathcal{Q}}
    \left\{\langle w-z_t,y_t-\widehat\lambda_t\rangle
      -\Psi_t(w)+\Psi_t(z_t)\right\}\\
  &\leq\sup_{w\in\mathcal{Q}}
    \left\{\langle w-z_t,v_t-\widehat\lambda_t\rangle
      -\Psi_t(w)+\Psi_t(z_t)\right\}\\
  &\leq\sup_{w\in[0,1]^{D_{\mathrm{var}}}}
    \left\{\langle w-z_t,v_t-\widehat\lambda_t\rangle
      -\Psi_t(w)+\Psi_t(z_t)\right\}\\
  &=D_{\Psi_t^*}(v_t-\widehat\lambda_t,v_t).
\end{align*}
The last equality uses $z_t=\nabla\Psi_t^*(v_t)$ and Fenchel equality.

For $r\in[0,1]$, define
\[
  x_t(r):=\nabla\Psi_t^*(v_t-r\widehat\lambda_t),
  \qquad
  \widetilde z_t:=x_t(1),
\]
Let $J_{t,i}$ join $z_{t,i}=x_{t,i}(0)$ and $\widetilde z_{t,i}$. Monotonicity
of $\nabla\Psi_{t,i}^*$ keeps $x_{t,i}(r)$ in this interval.
Taylor's integral formula and separability give
\[
  s_t
  \leq\int_0^1(1-r)\sum_i\widehat\lambda_{t,i}^2
      [\Psi_{t,i}''(x_{t,i}(r))]^{-1}\,dr
  \leq\frac{\eta_t}{2}\sum_i\widehat\lambda_{t,i}^2
      \max_{x\in J_{t,i}}[\Psi_i''(x)]^{-1}.
\]

To bound $J_{t,i}$, write $p=z_{t,i}$, $y=\lambda_{t,i}+1\in[0,2]$, and
\[
  \psi(x):=-\sqrt x+(1-x)\log(1-x),
  \qquad
  \psi'(x)=-\frac1{2\sqrt x}-\log(1-x)-1.
\]
Equation~\eqref{eq:hybrid-shifted-estimator} gives, pathwise,
\[
  -1\leq\widehat\lambda_{t,i}\leq\frac2p.
\]
For $t\geq t_0:=58$,
\[
  \eta_t\leq\min\left\{\frac{\sqrt2-1}{2},\frac{\log2}{4}\right\}.
\]
Let $\widetilde p=\widetilde z_{t,i}$. By the definition
of the unconstrained dual update,
\[
  \widehat\lambda_{t,i}
  =\eta_t^{-1}[\psi'(p)-\psi'(\widetilde p)].
\]
If $p<1/2$ and $\widetilde p>2p$, monotonicity of
$\psi'$ and the displayed derivative imply
\[
  \widehat\lambda_{t,i}
  <\eta_t^{-1}[\psi'(p)-\psi'(2p)]
  <-\eta_t^{-1}\frac{\sqrt2-1}{2}\leq-1,
\]
contradicting the estimator range. For $p\geq1/2$,
the inequality $\widetilde p\leq2p$ is automatic. Similarly,
if $p>1/2$ and $\widetilde p<2p-1$, then
\[
  \widehat\lambda_{t,i}
  >\eta_t^{-1}[\psi'(p)-\psi'(2p-1)]
  >\eta_t^{-1}\log2\geq4>\frac2p,
\]
again a contradiction; for $p\leq1/2$, the
inequality $\widetilde p\geq2p-1$ is automatic. Thus
\[
  2p-1\leq\widetilde p\leq2p.
\]
Every $x\in J_{t,i}$ therefore satisfies $x\leq2p$ and $1-x\leq2(1-p)$. Since
\[
  \Psi_i''(x)=\frac1{4x^{3/2}}+\frac1{1-x},
  \qquad
  [\Psi_i''(x)]^{-1}\leq\min\{4x^{3/2},1-x\},
\]
we obtain
\begin{equation}
  \label{eq:local inverse-Hessian bound}
  \max_{x\in J_{t,i}}[\Psi_i''(x)]^{-1}
  \leq\min\{4(2p)^{3/2},2(1-p)\}.
\end{equation}

Put $\mu:=\mathbb E_t A_{t,i}$. The weighted condition in~\eqref{eq:hybrid-direct-sampling} gives
$y(\mu-p)=0$. Since $\widehat\lambda_{t,i}=A_{t,i}y/p-1$ and $A_{t,i}$ is binary,
\[
  \mathbb{E}_t\widehat\lambda_{t,i}^2
  =\frac{\mu y^2}{p^2}-\frac{2\mu y}{p}+1
  =\frac{y^2}{p}-2y+1
  \leq\frac4p.
\]
The last bound follows by maximizing the convex
quadratic $y^2-2py+p$ at $y=0,2$, where its values are
$p,4-3p\leq4$. Combining these bounds, it remains to use
\[
  \frac2p\min\{4(2p)^{3/2},2(1-p)\}
  \leq16\sqrt2\min\{\sqrt p,1-p\},
  \qquad 0<p<1,
\]
Use the first term when $\sqrt p\leq1-p$; otherwise
$p\geq1/4$ and $4(1-p)/p\leq16(1-p)$. Thus, for $t\geq58$,
\[
  \mathbb{E}_t s_t
  \leq16\sqrt2\,\eta_t\sum_i\min\{\sqrt{z_{t,i}},1-z_{t,i}\}.
\]
For every binary comparator coordinate,
\[
  \min\{\sqrt z,1-z\}
  \leq(1-a_i^\star)\sqrt z+a_i^\star(1-z).
\]
For $t\geq58$, take expectations and apply Jensen
to the square-root terms. For earlier rounds,
use the roundwise identity behind \eqref{eq:hybrid-stability}. Put
$z_t^+:=\nabla\Phi_t(-\widehat L_t)$. Convexity, $\widehat\lambda_{t,i}\geq-1$, and $\sum_i z_{t,i}^+=m_0$ give
\[
  \Phi_t(-\widehat L_t)-\Phi_t(-\widehat L_{t-1})
  \leq\langle z_t^+,-\widehat\lambda_t\rangle\leq m_0.
\]
Together with $\langle A_t,\lambda_t\rangle\leq m_0$, this bounds the early contribution
by $2t_0m_0=116m_0$. Extend the nonnegative refined sum to
all rounds and recall $u_{t,i}=1-\mathbb E z_{t,i}$ to obtain
\begin{equation}
  \mathbb{E}S_H
  \leq\sum_{t=1}^H\frac{16\sqrt{2}}{\sqrt{t}}
  \left[
    \sum_{i:a_i^\star=0}\sqrt{\mathbb{E}z_{t,i}}
    +\sum_{i:a_i^\star=1}u_{t,i}
  \right]+116m_0.
\label{eq:hybrid-stability-bound}
\end{equation}
Combining Equations~\eqref{eq:hybrid-penalty-bound} and \eqref{eq:hybrid-stability-bound}, and
using $3/2+16\sqrt2<25$, proves Equation~\eqref{eq:hybrid-BH}.

\paragraph{Stochastic self-bounding calculation.}
Write $\bar q_t:=\mathbb E q_t$ and, for a
distribution $q$ on $\mathcal A$, put
\[
  r(q):=\sum_{a\neq a^\star}q(a)\Delta_a.
\]
Extend $r$ linearly to nonnegative action weights
and extend $f$ to nonnegative vectors
by its same square-root formula. Define
\[
  \mathcal C_{\mathrm{sto}}
  :=\sup_{\alpha\in[0,\infty)^{\mathcal A}}
  \left\{f\!\left(\sum_a\alpha(a)a\right)-r(\alpha)\right\},
\]
\[
  \mathcal C_{\mathrm{add}}
  :=\sum_{t=1}^\infty\sup_{q\in\Delta(\mathcal A)}
  \left\{\frac{100}{\sqrt t}g\!\left(\sum_aq(a)a\right)-r(q)\right\}.
\]
For $\omega:=\sum_{a\neq a^\star}\alpha(a)$, we have $r(\alpha)\geq\Delta\omega$. Only suboptimal
actions contribute to $f$, each on at
most $m_0$ coordinates, so Cauchy--Schwarz gives
\[
  f\!\left(\sum_a\alpha(a)a\right)
  \leq\sqrt{D_{\mathrm{var}}m_0\omega}.
\]
Maximizing the scalar expression $\sqrt{D_{\mathrm{var}}m_0\omega}-\Delta\omega$ over $\omega\geq0$ yields
\begin{equation}
  \mathcal C_{\mathrm{sto}}\leq\frac{m_0D_{\mathrm{var}}}{4\Delta}.
  \label{eq:hybrid-Csto-bound}
\end{equation}

For $q\in\Delta(\mathcal A)$, write $z=\sum_aq(a)a$ and set
\[
  u:=\frac1{m_0}\sum_{i:a_i^\star=1}(1-z_i)\in[0,1].
\]
Since $u\leq\sum_{a\neq a^\star}q(a)$, we have $r(q)\geq\Delta u$. Concavity of
$v(1-\log v)$ on $[0,1]$ and Jensen's inequality give
\[
  g(z)\leq m_0u(1-\log u).
\]
For $a_t:=100m_0/\sqrt t$, the function $a_tu(1-\log u)-\Delta u$ is maximized on
$[0,1]$ at $u=e^{-\Delta/a_t}$, with maximum $a_te^{-\Delta/a_t}$. Hence
\[
  \sup_{q\in\Delta(\mathcal A)}
  \left\{\frac{100}{\sqrt t}g\!\left(\sum_aq(a)a\right)-r(q)\right\}
  \leq\frac{100m_0}{\sqrt t}
    \exp\!\left(-\frac{\Delta\sqrt t}{100m_0}\right).
\]
The bound decreases with $t$.
Integral comparison and $v=\sqrt x$ give
\begin{align}
  \mathcal C_{\mathrm{add}}
  &\leq100m_0\int_0^\infty x^{-1/2}
       \exp\!\left(-\frac{\Delta\sqrt x}{100m_0}\right)\,dx\notag\\
  &=200m_0\int_0^\infty
       \exp\!\left(-\frac{\Delta v}{100m_0}\right)\,dv
   =\frac{20000m_0^2}{\Delta}.
  \label{eq:hybrid-Cadd-bound}
\end{align}

Finally set
\[
  F_t:=f(\mathbb E z_t),\qquad G_t:=g(\mathbb E z_t),
  \qquad R_t:=r(\bar q_t),\qquad \alpha_t:=\frac{t}{10^4}\bar q_t.
\]
The marginal of $\bar q_t$ is $\mathbb E z_t$. Since $f(cz)=\sqrt c\,f(z)$,
\begin{align*}
  2B_H-R_H^{\mathrm{pure}}
  &=\sum_{t=1}^H\left(\frac{50}{\sqrt t}F_t-\frac12R_t\right)
    +\sum_{t=1}^H\left(\frac{50}{\sqrt t}G_t-\frac12R_t\right)
    +232m_0\\
  &=\sum_{t=1}^H\frac{5000}{t}
    \left\{f\!\left(\sum_a\alpha_t(a)a\right)-r(\alpha_t)\right\}\\
  &\quad+\frac12\sum_{t=1}^H
    \left\{\frac{100}{\sqrt t}G_t-R_t\right\}+232m_0\\
  &\leq5000\mathcal C_{\mathrm{sto}}\sum_{t=1}^H\frac1t
    +\frac12\mathcal C_{\mathrm{add}}+232m_0.
\end{align*}
The suprema are nonnegative (take $q$ concentrated
on $a^\star$), justifying the infinite sum.
Apply~\eqref{eq:hybrid-Csto-bound}--\eqref{eq:hybrid-Cadd-bound} and $\sum_{t\leq H}1/t\leq\log(eH)$. Finally, $R_H^{\mathrm{pure}}\leq\mathbb E C_H\leq B_H$ gives
\[
  \mathbb E C_H\leq B_H\leq2B_H-R_H^{\mathrm{pure}}.
\]
This completes Equation~\eqref{eq:hybrid-stochastic-self-bound} and the proof.
\end{proof}

\paragraph{Application to the original semi-bandit model.}
Take $\mathcal A=\widetilde\cX_0$, with $m_0\leq m$ and $D_{\mathrm{var}}\leq D=d+m$.
For projected losses $g_t\in[0,1]^{D_{\mathrm{var}}}$, set $\lambda_t=g_t-\one$,
$A_t=X_t^0$, and $z_t=b_t$. The estimator becomes
\[
  \widehat\lambda_{t,i}
  =\frac{X_{t,i}^0(\lambda_{t,i}+1)}{b_{t,i}}-1
  =\frac{X_{t,i}^0g_{t,i}}{b_{t,i}}-1
  =\widetilde\ell_{t,i}.
\]
For original coordinates, Assumption~\ref{ass:non-anticipation} and the tower property
imply~\eqref{eq:hybrid-direct-sampling}, since $\lambda_t$ is a deterministic function
of the original loss vector. Auxiliary coordinates
satisfy it because $g_{t,i}=0$. Equal cardinality gives
\begin{align*}
  \langle b_t-a,\widetilde\ell_t\rangle
  &=\langle b_t-a,\widehat g_t\rangle-\langle b_t-a,\one\rangle
   =\langle b_t-a,\widehat g_t\rangle,\\
  \langle a-a^\star,g_t-\one\rangle
  &=\langle a-a^\star,g_t\rangle
    -\underbrace{\langle a-a^\star,\one\rangle}_{m_0-m_0=0}
   =\langle a-a^\star,g_t\rangle.
\end{align*}
The shift preserves empirical regret and action
gaps. Deterministic padding and projection preserve the
conditional i.i.d. law and independence from learner
seeds after a restart (Lemma~\ref{lem:conditional-restart}), so the
theorem applies with the original gap $\Delta$.

%
\section{Auxiliary probability results}
\label{app:probability-tools}

We collect the probability bounds used
above; see \citet[Chapters~3 and~5]{mainbib:lattimore2020bandit} and \citet{mainbib:freedman1975tail}.

\subsection{Bernstein, Ville, and Freedman inequalities}
\label{app:martingale-inequalities}

\begin{lemma}[Conditional Bernstein mgf bound]
\label{lem:conditional-bernstein-mgf}
Let $X$ be square-integrable, let $\mathcal{H}$
be a sigma-field, and suppose
\[
  \mathbb{E}[X\mid\mathcal{H}]=0,
  \qquad X\leq c\ \text{a.s.},
  \qquad \mathbb{E}[X^2\mid\mathcal{H}]\leq v
\]
for an $\mathcal{H}$-measurable $v\geq0$ and a
constant $c>0$. Then for $0\leq\lambda<3/c$,
\begin{equation}
  \mathbb{E}[e^{\lambda X}\mid\mathcal{H}]
  \leq\exp\!\left(\frac{\lambda^2v}{2(1-\lambda c/3)}\right).
\label{eq:conditional-bernstein-mgf}
\end{equation}
Consequently, if $Y$ has conditional mean
$\mu=\mathbb{E}[Y\mid\mathcal{H}]$, $Y-\mu\leq c$, and $\mathbb{E}[(Y-\mu)^2\mid\mathcal{H}]\leq v$, then
\begin{equation}
  \log\mathbb{E}[e^{\lambda Y}\mid\mathcal{H}]
  \leq\lambda\mu+\frac{\lambda^2v}{2(1-\lambda c/3)}.
\label{eq:conditional-bernstein-log-mgf}
\end{equation}
\end{lemma}

\begin{yellowbox}
\textbf{Role.} This one-step bound constructs the Freedman and
misranking supermartingales. \textbf{Proof~idea.} Taylor's integral remainder uses the
upper increment bound; conditional centering removes the
linear term. A geometric bound on the
exponential series yields the denominator. Applying the
bound to $Y-\mu$ retains the conditional drift.
\end{yellowbox}

\begin{proof}
Taylor's formula with integral remainder
gives, for $x\leq c$ and $\lambda\geq0$,
\[
  e^{\lambda x}-1-\lambda x
  =\lambda^2x^2\int_0^1(1-r)e^{r\lambda x}\,dr
  \leq\lambda^2x^2\int_0^1(1-r)e^{r\lambda c}\,dr
  =\frac{x^2}{c^2}(e^{\lambda c}-1-\lambda c).
\]
Taking conditional expectations, using $\mathbb{E}[X\mid\mathcal{H}]=0$,
and then $1+s\leq e^s$ gives
\[
  \mathbb{E}[e^{\lambda X}\mid\mathcal{H}]
  \leq\exp\!\left(
    \frac{e^{\lambda c}-1-\lambda c}{c^2}
    \mathbb{E}[X^2\mid\mathcal{H}]
  \right).
\]
The elementary estimate
\[
  e^u-1-u\leq\frac{u^2}{2(1-u/3)},
  \qquad 0\leq u<3,
\]
follows from $1/k!\leq1/(2\cdot3^{k-2})$ for $k\geq2$. Set $u=\lambda c$ and
use the variance bound to obtain \eqref{eq:conditional-bernstein-mgf}. Applying
it to $Y-\mu$ and taking logarithms gives~\eqref{eq:conditional-bernstein-log-mgf}.

\end{proof}

\begin{lemma}[Ville's maximal inequality]
\label{lem:ville}
Let $(M_n)_{n=0}^H$ be a nonnegative supermartingale with
respect to $(\mathcal{H}_n)_{n=0}^H$. Then, for every $a>0$,
\begin{equation}
  \Pr\!\left(\max_{0\leq n\leq H}M_n\geq a\right)
  \leq\frac{\mathbb{E}M_0}{a}.
\label{eq:ville}
\end{equation}
The bound also holds under any phase-start conditional
law for which $(M_n)$ is a nonnegative supermartingale.
\end{lemma}

\begin{yellowbox}
\textbf{Role.} This bounds crossings over all prefixes, including
random phase endpoints. \textbf{Proof~idea.} Stop at the first
crossing. Predictable indicators and the supermartingale property bound
the stopped expectation by its initial value;
nonnegativity bounds the crossing probability. The argument
also holds under a phase-start conditional law.
\end{yellowbox}

\begin{proof}
Let $\tau$ be the first index
$n\in\{0,\ldots,H\}$ with $M_n\geq a$, with $\tau=H$ if
no such index exists. Since $\{\tau\geq n\}\in\mathcal H_{n-1}$,
\[
  \mathbb E M_\tau
  =\mathbb E M_0+\sum_{n=1}^H
    \mathbb E\!\left[\mathbf1\{\tau\geq n\}
      \mathbb E[M_n-M_{n-1}\mid\mathcal H_{n-1}]\right]
  \leq\mathbb E M_0.
\]
On the crossing event $M_\tau\geq a$, and off
that event $M_\tau\geq0$. Thus $a\Pr(\max_{0\leq n\leq H}M_n\geq a)\leq\mathbb E M_\tau\leq\mathbb E M_0$. The
same calculation proves the conditional version.
\end{proof}

\begin{lemma}[Bernstein-form corollary of Freedman]
\label{lem:freedman}
Let $(X_t)_{t=1}^H$ be martingale differences with respect
to $(\mathcal{H}_t)_{t=0}^H$. Suppose $X_t\leq c$ almost surely
for a constant $c>0$, and define
\[
  S_n:=\sum_{t=1}^nX_t,
  \qquad
  V_n:=\sum_{t=1}^n\mathbb{E}[X_t^2\mid\mathcal{H}_{t-1}].
\]
Then, for every $x>0$ and $v>0$,
\begin{equation}
  \Pr\bigl(\text{there exists }n\leq H:S_n\geq x\text{ and }V_n\leq v\bigr)
  \leq\exp\!\left(-\frac{x^2}{2(v+cx/3)}\right).
\label{eq:freedman}
\end{equation}
In particular, if $V_n\leq v$ almost surely for a
fixed $n$, the same right-hand side bounds $\Pr(S_n\geq x)$.
\end{lemma}

\begin{yellowbox}
\textbf{Role.} This controls estimated-gap deviations in \eqref{eq:freedman-gap-event}. \textbf{Proof~idea.}
The Bernstein mgf bound makes the exponential of
the martingale sum minus its variance correction a
nonnegative supermartingale. Apply Ville's inequality with $\lambda=x/(v+cx/3)$. The
fixed-time event is contained in the maximal event.
\end{yellowbox}

\begin{proof}
For $0<\lambda<3/c$, put
\[
  \sigma_n^2:=\mathbb{E}[X_n^2\mid\mathcal{H}_{n-1}]
\]
and define
\[
  M_n(\lambda)
  :=\exp\!\left(
    \lambda S_n-
    \frac{\lambda^2}{2(1-\lambda c/3)}V_n
  \right),
  \qquad 0\leq n\leq H.
\]
By Lemma~\ref{lem:conditional-bernstein-mgf}, $M_0(\lambda)=1$ and
\begin{align*}
  \mathbb{E}[M_n(\lambda)\mid\mathcal{H}_{n-1}]
  &=M_{n-1}(\lambda)
    \exp\!\left[-\frac{\lambda^2\sigma_n^2}{2(1-\lambda c/3)}\right]
    \mathbb{E}[e^{\lambda X_n}\mid\mathcal{H}_{n-1}]\\
  &\leq M_{n-1}(\lambda).
\end{align*}
Thus $M_n(\lambda)$ is a nonnegative supermartingale.
At any qualifying index in~\eqref{eq:freedman},
\[
  M_n(\lambda)
  \geq\exp\!\left(
    \lambda x-
    \frac{\lambda^2v}{2(1-\lambda c/3)}
  \right).
\]
Ville's inequality therefore gives
\[
  \Pr\bigl(\exists n\leq H:S_n\geq x,\ V_n\leq v\bigr)
  \leq\exp\!\left(
    -\lambda x+
    \frac{\lambda^2v}{2(1-\lambda c/3)}
  \right).
\]
Set $\lambda=x/(v+cx/3)\in(0,3/c)$. Then
\[
  -\lambda x+
  \frac{\lambda^2v}{2(1-\lambda c/3)}
  =-\frac{x^2}{2(v+cx/3)},
\]
proving~\eqref{eq:freedman}. Event inclusion gives the fixed-$n$ bound.

After rescaling, \citet[Theorem~(4.1)]{mainbib:freedman1975tail} gives the sharper Bennett exponent
\[
  \exp\!\left[-\frac{v}{c^2}
    h\!\left(\frac{cx}{v}\right)\right],
  \qquad
  h(u):=(1+u)\log(1+u)-u.
\]
For $g(u):=3u^2/[2(u+3)]$, $h(0)=g(0)$, $h'(0)=g'(0)$, and
\[
  (h'-g')'(u)
  =\frac1{1+u}-\frac{27}{(u+3)^3}
  =\frac{u^2(u+9)}{(u+1)(u+3)^3}\geq0.
\]
Hence $h\geq g$, recovering the Bernstein exponent in~\eqref{eq:freedman}.
\end{proof}

\subsection{Stopped martingales}
\label{app:stopped-martingales}

\begin{lemma}[$L^2$ identity for a stopped martingale sum]
\label{lem:stopped-l2}
Let $(X_t)_{t=1}^H$ be square-integrable martingale differences for $(\mathcal{H}_t)_{t=0}^H$,
and let $\tau\in\{0,1,\ldots,H\}$ be a stopping time. Then
\begin{equation}
  \mathbb{E}\left[\left(\sum_{t=1}^{\tau}X_t\right)^2\right]
  =\mathbb{E}\sum_{t=1}^{\tau}\mathbb{E}[X_t^2\mid\mathcal{H}_{t-1}].
\label{eq:stopped-l2-identity}
\end{equation}
Consequently, if predictable $v_t\geq0$ satisfy $\mathbb{E}[X_t^2\mid\mathcal{H}_{t-1}]\leq v_t$, then
\begin{equation}
  \mathbb{E}\left|\sum_{t=1}^{\tau}X_t\right|
  \leq\left(\mathbb{E}\sum_{t=1}^{\tau}v_t\right)^{1/2}.
\label{eq:stopped-l1-bound}
\end{equation}
In particular, if $v_t\leq cr_t$ for predictable $r_t\geq0$,
the right-hand side is at most $(c\mathbb{E}\sum_{t\leq\tau}r_t)^{1/2}$.
\end{lemma}

\begin{yellowbox}
\textbf{Role.} This bounds the optimal-action fluctuation at a
stopped mixed phase, as required by~\eqref{eq:optimum-positive-fluctuation}. \textbf{Proof~idea.} The
squared martingale sum minus its accumulated conditional
variance is a martingale. Predictable stopping indicators
preserve its expectation. Cauchy--Schwarz then converts the
variance bound into an expected absolute fluctuation.
\end{yellowbox}

\begin{proof}
Set $S_n:=\sum_{t=1}^nX_t$, $V_n:=\sum_{t=1}^n\mathbb E[X_t^2\mid\mathcal H_{t-1}]$, and $M_n:=S_n^2-V_n$, with $S_0=V_0=0$. Square
integrability makes $M_n$ integrable, and conditional centering gives
\[
  \mathbb E[M_n-M_{n-1}\mid\mathcal H_{n-1}]
  =\mathbb E[2S_{n-1}X_n+X_n^2\mid\mathcal H_{n-1}]
    -\mathbb E[X_n^2\mid\mathcal H_{n-1}]=0.
\]
Thus $M_n$ is a martingale.
Since $\tau\leq H$ and $\{\tau\geq n\}\in\mathcal H_{n-1}$,
\[
  \mathbb E M_\tau
  =\sum_{n=1}^H\mathbb E\!\left[
    \mathbf1\{\tau\geq n\}
    \mathbb E[M_n-M_{n-1}\mid\mathcal H_{n-1}]
    \right]=0.
\]
This proves~\eqref{eq:stopped-l2-identity}; it is bounded optional stopping
\citep[Theorem~3.8(a)]{mainbib:lattimore2020bandit}. By Cauchy--Schwarz and the conditional-variance bound,
\[
  \mathbb E|S_\tau|
  \leq\sqrt{\mathbb E S_\tau^2}
  =\sqrt{\mathbb E V_\tau}
  \leq\left(\mathbb E\sum_{t=1}^{\tau}v_t\right)^{1/2},
\]
which proves~\eqref{eq:stopped-l1-bound}. Substitute $v_t\leq cr_t$ for the final assertion.
\end{proof}

\subsection{Tail integration, exponential series, and layer cake}
\label{app:tail-series-layer-cake}

\begin{lemma}[Tail integration for a maximum
over a finite set]
\label{lem:tail-integration-finite-max}
For every nonnegative random variable $Z$,
\begin{equation}
  \mathbb{E}Z=\int_0^\infty\Pr(Z>u)\,du.
\label{eq:tail-integration}
\end{equation}
Consequently, if $N\geq1$, $\lambda>0$, and
\[
  \Pr(Z>u)\leq\min\{1,Ne^{-\lambda u}\}
  \qquad(u\geq0),
\]
then
\begin{equation}
  \mathbb{E}Z\leq\frac{\log N+1}{\lambda}.
\label{eq:tail-integration-exponential}
\end{equation}
In particular, suppose that for each $a$ in
a finite set $\mathcal{A}$ of size $N$, $(e^{\lambda D_n(a)})_{n=0}^H$
is a nonnegative supermartingale starting at one. Then
\begin{equation}
  \mathbb{E}\max_{a\in\mathcal{A}}\sup_{n\leq H}[D_n(a)]_+
  \leq\frac{\log N+1}{\lambda}.
\label{eq:expected-max-via-ville}
\end{equation}
\end{lemma}

\begin{yellowbox}
\textbf{Role.} This bounds the expected misranking maximum,
with maximization inside expectation. \textbf{Proof~idea.} Ville's inequality
and a union bound give the tail
$\min\{1,Ne^{-\lambda u}\}$. Tonelli's theorem permits tail integration;
splitting at $u=(\log N)/\lambda$ yields the bound.
\end{yellowbox}

\begin{proof}
Tonelli applied to $Z=\int_0^\infty\mathbf1\{Z>u\}\,du$ proves \eqref{eq:tail-integration}, including infinite
expectations. For the exponential tail, set $u_0:=(\log N)/\lambda\geq0$. Then
\begin{align*}
  \mathbb{E}Z
  &=\int_0^\infty\Pr(Z>u)\,du\\
  &\leq\int_0^{u_0}1\,du
    +\int_{u_0}^\infty Ne^{-\lambda u}\,du\\
  &=u_0+\frac{N}{\lambda}e^{-\lambda u_0}
   =\frac{\log N+1}{\lambda}.
\end{align*}
This proves Equation~\eqref{eq:tail-integration-exponential}.

For the final assertion, set
\[
  Z_{\mathcal A}:=
  \max_{a\in\mathcal A}\max_{0\leq n\leq H}[D_n(a)]_+.
\]
For $a\in\mathcal A$ and $u\geq0$, Lemma~\ref{lem:ville} gives
\begin{align*}
  \Pr\!\left(\max_{0\leq n\leq H}[D_n(a)]_+>u\right)
  &=\Pr\!\left(\max_{0\leq n\leq H}D_n(a)>u\right)\\
  &\leq\Pr\!\left(\max_{0\leq n\leq H}e^{\lambda D_n(a)}
                   \geq e^{\lambda u}\right)
   \leq e^{-\lambda u}.
\end{align*}
A union bound gives
\[
  \Pr(Z_{\mathcal A}>u)
  \leq\min\{1,Ne^{-\lambda u}\}.
\]
Applying Equation~\eqref{eq:tail-integration-exponential} to $Z_{\mathcal A}$ proves Equation~\eqref{eq:expected-max-via-ville}.
\end{proof}

\begin{lemma}[Bounds for two exponential series]
\label{lem:two-exponential-series}
Let $N\geq1$, $\Lambda:=\log(eN)$, and $a,b>0$.
For every integer $H\geq1$,
\begin{align}
  \sum_{r=0}^{H-1}\min\{1,Ne^{-a\sqrt{r}}\}
  &\leq2+\frac{5\Lambda^2}{a^2},
\label{eq:sqrt-exponential-series}\\
  \sum_{r=0}^{H-1}\min\{1,Ne^{-br}\}
  &\leq2+\frac{2\Lambda}{b}.
\label{eq:linear-exponential-series}
\end{align}
In particular, if $0<s\leq m$, $0<\rho\leq1$, and
\[
  a=\frac{s\sqrt{\rho\Lambda}}{8m},
  \qquad
  b=\frac{\rho s^2}{40m^2},
\]
then
\begin{equation}
  \sum_{r=0}^{H-1}
  \left[\min\{1,Ne^{-a\sqrt{r}}\}+\min\{1,Ne^{-br}\}\right]
  \leq\frac{404m^2\Lambda}{\rho s^2}.
\label{eq:combined-exponential-series}
\end{equation}
\end{lemma}

\begin{yellowbox}
\textbf{Role.} These bounds sum gap-tail estimates uniformly over
the phase length for the layer-cake argument. \textbf{Proof~idea.}
Bound each decreasing series by its initial term
plus an integral. Split where the exponential meets
one and substitute $y=\sqrt{x}$ for the square-root
exponent. Both parameter choices give $s^{-2}$ dependence.
\end{yellowbox}

\begin{proof}
For nonincreasing $f:[0,\infty)\to[0,1]$, monotonicity gives
\begin{align*}
  \sum_{r=0}^{H-1}f(r)
  &=f(0)+\sum_{r=1}^{H-1}f(r)\\
  &\leq1+\sum_{r=1}^{H-1}\int_{r-1}^{r}f(x)\,dx\\
  &=1+\int_0^{H-1}f(x)\,dx
  \leq1+\int_0^\infty f(x)\,dx.
\end{align*}

Put $\theta:=\log N=\Lambda-1\geq0$. For the square-root series, define
\[
  f_a(x):=\min\{1,Ne^{-a\sqrt{x}}\},
  \qquad x_a:=\left(\frac{\theta}{a}\right)^2.
\]
Splitting at $x_a$ and substituting $y=\sqrt{x}$ gives
\begin{align*}
  \int_0^\infty f_a(x)\,dx
  &=x_a+N\int_{x_a}^\infty e^{-a\sqrt{x}}\,dx\\
  &=\frac{\theta^2}{a^2}
    +2N\int_{\theta/a}^\infty ye^{-ay}\,dy\\
  &=\frac{\theta^2+2\theta+2}{a^2}
   =\frac{\Lambda^2+1}{a^2},
\end{align*}
using $Ne^{-\theta}=1$. Integral comparison gives
\[
  \sum_{r=0}^{H-1}\min\{1,Ne^{-a\sqrt r}\}
  \leq1+\frac{\Lambda^2+1}{a^2}
  \leq1+\frac{2\Lambda^2}{a^2},
\]
using $\Lambda\geq1$, which implies~\eqref{eq:sqrt-exponential-series}.

For the linear-exponential series, define $f_b(x):=\min\{1,Ne^{-bx}\}$
and $x_b:=\theta/b$. Splitting at $x_b$ gives
\begin{align*}
  \int_0^\infty f_b(x)\,dx
  &=x_b+N\int_{x_b}^\infty e^{-bx}\,dx\\
  &=\frac{\theta+1}{b}=\frac{\Lambda}{b}.
\end{align*}
The same comparison now gives
\[
  \sum_{r=0}^{H-1}\min\{1,Ne^{-br}\}
  \leq1+\frac{\Lambda}{b}
  \leq2+\frac{2\Lambda}{b},
\]
which proves~\eqref{eq:linear-exponential-series}.

For the last claim, set $K:=m^2\Lambda/(\rho s^2)$. Substitution gives
\[
  2+\frac{5\Lambda^2}{a^2}=2+320K,
  \qquad
  2+\frac{2\Lambda}{b}=2+80K.
\]
Since $K\geq1$, their sum is
\[
  4+400K\leq404K
  =\frac{404m^2\Lambda}{\rho s^2},
\]
which is~\eqref{eq:combined-exponential-series}.
\end{proof}

\begin{lemma}[Integral formula for the expected action gap]
\label{lem:action-gap-layer-cake}
Let $\mathcal{A}$ be finite, let $q$ be
a possibly random probability vector on $\mathcal{A}$,
and let deterministic gaps satisfy $0\leq\Delta_a\leq M$. Define
\[
  Q(s):=\sum_{a:\,\Delta_a\geq s}q(a),
  \qquad 0<s\leq M.
\]
Then
\begin{equation}
  \mathbb{E}\sum_{a\in\mathcal{A}}q(a)\Delta_a
  =\int_0^M\mathbb{E}Q(s)\,ds.
\label{eq:action-gap-layer-cake}
\end{equation}
Suppose, in addition, that $0<\Delta\leq M$, every positive gap
is at least $\Delta$, and for some constant
$A_0>0$ the following tail-mass bound has been established:
\[
  \mathbb{E}Q(s)\leq\frac{A_0}{s^2}
  \qquad(\Delta\leq s\leq M),
\]
Then
\begin{equation}
  \mathbb{E}\sum_a q(a)\Delta_a\leq\frac{2A_0}{\Delta}.
\label{eq:action-gap-layer-cake-bound}
\end{equation}
For a finite sequence $q_n$, the identity
holds after summing over $n$. The
bound holds if $\sum_n\mathbb E Q_n(s)\leq A_0/s^2$ for $\Delta\leq s\leq M$.

\end{lemma}

\begin{yellowbox}
\textbf{Role.} This converts gap-tail mass into expected pseudo-regret
with $1/\Delta$ dependence. \textbf{Proof~idea.} Express each gap
as an integral of threshold indicators and
apply Tonelli. The tail mass is constant
below $\Delta$; integrate $A_0/s^2$ above it. The
argument also applies to sums over rounds.
\end{yellowbox}

\begin{proof}
The first identity is the layer-cake formula $\Delta_a=\int_0^M\mathbf{1}\{s\leq\Delta_a\}\,ds$,
followed by Tonelli's theorem. When no gap lies
in $(0,\Delta)$, $Q(s)=Q(\Delta)$ for almost every $s\in(0,\Delta)$. Hence
\[
  \int_0^M\mathbb{E}Q(s)\,ds
  \leq\Delta\frac{A_0}{\Delta^2}
  +A_0\int_\Delta^M\frac{ds}{s^2}
  \leq\frac{2A_0}{\Delta}.
\]
The time-summed form follows from
the same finite-sum/Tonelli interchange.
\end{proof}

In Step~3 of Lemma~\ref{lem:one-phase-empirical-regret}, Equation~\eqref{eq:gap-mass-series} verifies the
time-summed hypothesis with $M=m$ and $A_0=404m^2\Lambda/\rho$.

\subsection{Max-increment and importance-weighted overshoot}
\label{app:max-increment-overshoot}

\begin{lemma}[One-step bound for a maximum of cumulative sums]
\label{lem:running-max-increment}
Let $\mathcal{A}$ be a nonempty finite
set, let $v_t(a)\in\mathbb{R}$, and set
\[
  V_n:=\max_{a\in\mathcal{A}}\sum_{t=1}^n v_t(a),
  \qquad U_n:=[V_n]_+,
  \qquad V_0=U_0=0.
\]
Then, for every $n\geq1$,
\begin{align}
  V_n-V_{n-1}
  &\leq\max_{a\in\mathcal{A}}v_n(a),
\label{eq:running-max-increment}\\
  U_n-U_{n-1}
  &\leq\left[\max_{a\in\mathcal{A}}v_n(a)\right]_+.
\label{eq:positive-running-max-increment}
\end{align}
The bounds also hold for affine $v_t$ on
a compact polytope, by maximizing over its vertices.

In addition, let $z$ and every comparator $x$
be coordinatewise nonnegative, and let $h\geq0$ coordinatewise. Then
\begin{equation}
  \max_x\langle z-x,h\rangle\leq\langle z,h\rangle.
\label{eq:nonnegative-comparator-drop}
\end{equation}
If a binary random vector $X$ has arm-selection
marginal $p$, if $p\geq\beta z$ coordinatewise for $\beta\in[0,1]$, if
$\ell\geq0$ coordinatewise, and if $\widehat\ell_i=X_i\ell_i/p_i$ on coordinates
with $p_i>0$ (and $\widehat\ell_i:=0$ when $p_i=0$), then
\begin{equation}
  \beta\max_x\langle z-x,\widehat\ell\rangle
  \leq\langle p,\widehat\ell\rangle
  =\langle X,\ell\rangle.
\label{eq:importance-weighted-overshoot}
\end{equation}
Coordinates with $p_i=0$ are unselected almost surely
and contribute zero to both sides.
\end{lemma}

\begin{yellowbox}
\textbf{Role.} This controls threshold overshoots when the empirical
comparator changes. \textbf{Proof~idea.} Compare the previous cumulative sum
at the current maximizer and use subadditivity of
the positive part. Nonnegative losses let us drop
the comparator term; $p\geq\beta z$ then bounds the
weighted estimate by the observed action loss.
\end{yellowbox}

\begin{proof}
Choose a maximizer $a_n$ for $V_n$. Then
\[
  V_n=\sum_{t=1}^{n-1}v_t(a_n)+v_n(a_n)
  \leq V_{n-1}+\max_{a\in\mathcal{A}}v_n(a),
\]
which is~\eqref{eq:running-max-increment}. Using $U_j=[V_j]_+$, we now have
\begin{align*}
  U_n-U_{n-1}
  &=[V_n]_+-[V_{n-1}]_+\\
  &\leq[V_n-V_{n-1}]_+\\
  &\leq\left[\max_{a\in\mathcal A}v_n(a)\right]_+,
\end{align*}
which is~\eqref{eq:positive-running-max-increment}, using $[r]_+-[s]_+\leq[r-s]_+$ and
monotonicity of the positive part.

Coordinatewise nonnegativity gives $\langle z-x,h\rangle\leq\langle z,h\rangle$. Since $p_i\geq\beta z_i$ and $\widehat\ell_i\geq0$,
\[
  \beta\langle z,\widehat\ell\rangle
  \leq\langle p,\widehat\ell\rangle
  =\sum_iX_i\ell_i.
\]
This proves Equations~\eqref{eq:nonnegative-comparator-drop} and~\eqref{eq:importance-weighted-overshoot}.
\end{proof}

\clearpage
\section{Well-definedness of the pure-phase update}
\label{app:supplementary-note}
We justify the pure-phase optimizer, potential
gradient, and positive importance-weighting denominators.

\begin{lemma}[Well-defined pure update]\label{lem:pure-interior}
For $\eta>0$ and $y\in\mathbb R^{D_{\mathrm{var}}}$, define
\[
  \Phi_\eta(y)
  :=\max_{z\in\conv(\widetilde\cX_0)}
  \{\langle z,y\rangle-\eta^{-1}\Psi(z)\}.
\]
Its maximizer is unique and lies strictly between
zero and one on every nonconstant coordinate. Moreover,
$\nabla\Phi_\eta(y)$ equals this maximizer, and the restored original
marginal is strictly positive on every original arm.
\end{lemma}

\begin{yellowbox}
\textbf{Role.} This justifies pure-phase importance weighting and
the potential-gradient update. \textbf{Proof~idea.} Moving a boundary
optimizer inward gives a regularizer gain that
dominates linear costs. Strict concavity gives uniqueness;
objective comparisons give the gradient. Coverage ensures
restored constant original coordinates have marginal one.
\end{yellowbox}
\begin{proof}
The maximum exists by continuity on
the compact projected polytope. Write
\[
  \psi(s):=-\sqrt s+(1-s)\log(1-s),
\]
with the continuous convention $0\log0=0$, so that
$\Psi(z)=\sum_j\psi(z_j)$. Fix a maximizer $b$ of
\[
  F(z):=\langle z,y\rangle-\eta^{-1}\Psi(z)
\]
over the projected polytope, and fix
a nonconstant coordinate $i$. On $(0,1)$,
\[
  \psi''(s)=\frac1{4s^{3/2}}+\frac1{1-s}>0.
\]
Thus $F$ is strictly concave in the open
cube. The endpoint argument below places every maximizer
there, proving uniqueness. Let $b_h$ maximize the objective
at $y+h$. Compactness and uniqueness give $b_h\to b$: every
convergent subsequence maximizes the limiting objective. Testing $b$
at $y+h$ and $b_h$ at $y$ yields
\[
  \langle b,h\rangle\leq\Phi_\eta(y+h)-\Phi_\eta(y)\leq\langle b_h,h\rangle.
\]
Thus the difference from $\langle b,h\rangle$ is
$o(\lVert h\rVert)$, proving the gradient identity.

If $b_i=0$, choose a vertex $a$ with $a_i=1$
and set $z(u)=(1-u)b+ua$. For small $u$, bounded derivatives
give total regularizer change over interior coordinates at
most $C_bu$ in absolute value. Boundary coordinates
move inward: if $b_j=0$ and $a_j=1$, then
\begin{equation}
  -\psi(u)+\psi(0)
  =\sqrt u-(1-u)\log(1-u)
  \ge\sqrt u,
  \label{eq:psi-zero-endpoint}
\end{equation}
whereas if $b_j=1$ and $a_j=0$, then, for $0<u\le e^{-1}$,
\begin{equation}
  -\psi(1-u)+\psi(1)
  =\sqrt{1-u}-1-u\log u
  \ge u(-1-\log u)\ge0;
  \label{eq:psi-one-endpoint}
\end{equation}
Here we used $\sqrt{1-u}\ge1-u$. Consequently, with $C_y:=|\langle a-b,y\rangle|$,
\[
  F(z(u))-F(b)
  \ge\eta^{-1}\sqrt u-(C_y+\eta^{-1}C_b)u>0
\]
for all sufficiently small $u>0$, contradicting maximality.

If $b_i=1$, choose $a$ with
$a_i=0$ and use~\eqref{eq:psi-one-endpoint} instead:

\[
  F(z(u))-F(b)
  \ge\eta^{-1}u(-1-\log u)-(C_y+\eta^{-1}C_b)u>0
\]
for small $u$, again contradicting maximality. Thus
every retained coordinate lies in $(0,1)$. Coverage
excludes constant-zero original arms, so restored constant
coordinates have marginal one. Hence $b_{t,i}>0$ for
all original arms, as required by Definition~\ref{def:base-interface}.

\end{proof}

\clearpage
\section{Extended related work}
\label{app:related-work}

We relate Theorem~\ref{thm:main-result} to semi-bandit learning,
baseline constraints, and stochastic--adversarial adaptation, distinguishing
their feedback, comparators, and guarantees.

\subsection{Combinatorial decisions and feedback models}

Combinatorial action sets can be exponentially large in
the number of components. \citet{mainbib:cesabianchi2012combinatorial} study online linear
optimization with aggregate bandit feedback, which reveals only
the total action loss. Semi-bandit feedback reveals each
selected component's loss. \citet{mainbib:audibert2014regret} analyze full-information, semi-bandit, and
aggregate-bandit feedback, obtaining minimax-optimal adversarial semi-bandit guarantees via
online stochastic mirror descent for an equal-cardinality action
model. Neither work imposes our baseline constraint.

Our mixed learner uses exponential weights, whose
full-information and adversarial-bandit foundations are given by
\citet{mainbib:freund1997decision} and \citet{mainbib:auer2002nonstochastic}, respectively. Baseline coverage bounds
every arm's observation probability below during mixed
phases, enabling entropy-FTRL analysis. We separately prove
the estimates needed at random phase endpoints.

\subsection{Stochastic combinatorial learning and instance dependence}

\citet{mainbib:gai2012combinatorial} obtain logarithmic stochastic regret for linear combinatorial
rewards by sharing observations across actions. \citet{mainbib:chen2013combinatorial}
allow nonlinear rewards under monotonicity and bounded
smoothness, with approximation-oracle guarantees measured against an
approximation benchmark rather than an exact optimum.

For linear semi-bandits, \citet{mainbib:kveton2015tight} analyze CombUCB1 and give
gap-dependent and gap-independent bounds matching lower bounds up
to constants and logarithmic factors, respectively. Their reward
components may be dependent. \citet{mainbib:combes2015revisited} derive instance-specific asymptotic
lower bounds and ESCB under independent Bernoulli rewards;
their separate adversarial aggregate-bandit analysis does not
give one semi-bandit algorithm for both environments.

\citet{mainbib:degenne2016covariance} use a known sub-Gaussian covariance upper bound
in OLS-UCB, obtaining correlation-dependent regret and lower bounds.
\citet{mainbib:wang2018thompson} prove gap-dependent logarithmic Thompson-sampling regret under
independent outcomes and exact optimization, with a
separate matroid result relaxing independence. \citet{mainbib:perrault2020statistical} improve
action-size dependence for independent bounded outcomes and
analyze Gaussian sampling for multivariate sub-Gaussian outcomes.

We allow dependence within loss vectors, but require
i.i.d. vectors independent of learner randomness and a
unique optimal action in the full action set.
Our simultaneous baseline and adversarial guarantees introduce
coverage and restart factors; they do not
improve these unconstrained rates or oracle guarantees.

\subsection{Conservative exploration relative to a baseline}

\citet{mainbib:evendar2008regret} obtain constant regret to a designated expert
distribution and near-square-root best-expert regret under full information.
Their PhasedAggression algorithm mixes with the baseline, increases
the learning weight, and restarts based on observed
performance. \citet{mainbib:muller2025best} obtain constant expected baseline regret and
near-square-root adversarial regret under bandit feedback, assuming positive
baseline probability for every action. Their known-horizon algorithm
uses a fixed restart threshold and directly motivates
\Duet{}. Its empirical tests adapt both the mixing
weight and the regret budget. With arm-level feedback
and coverage, \Duet{} also obtains gap-dependent polylogarithmic stochastic
regret without knowing the horizon.
Our baseline bound is logarithmic rather than
constant; their separate stochastic bound is gap-independent.

\citet{mainbib:wu2016conservative} study cumulative baseline-performance constraints in stochastic and
adversarial bandits, including expected and high-probability stochastic guarantees.
\citet{mainbib:kazerouni2017conservative} extend this approach to contextual linear bandits:
CLUCB preserves a prescribed fraction of baseline
performance uniformly over time, with separate learning
and conservative exploration costs. \citet{mainbib:garcelon2020improved} develop CLUCB2
and constraints checked only at predetermined checkpoints.

\citet{mainbib:du2021one} give a framework for conservative multi-armed, linear,
and contextual combinatorial bandits with realized cumulative reward
constraints and a default arm of prescribed deterministic
reward. \citet{mainbib:yang2022reduction} derive upper and lower bounds by
reducing unconstrained bandits and reinforcement learning to conservative
learning, using the baseline's exploration budget. Their
budget perspective relates to our accounting argument.

Our constraint is an additive, logarithmic bound
on expected cumulative loss relative to a
fixed baseline distribution. Arm coverage also makes
baseline observations useful for learning. These differences
in constraint and feedback require checking a
conservative wrapper's assumptions before applying its guarantee.

\subsection{Structured conservative and safe semi-bandits}

\citet{mainbib:katariya2019conservative} interleave exploration with a baseline action in
stochastic combinatorial semi-bandits. Exchangeable action spaces and
itemwise comparisons give an action-level constraint and
gap-dependent logarithmic regret. Our action set is
arbitrary and our baseline is a distribution.

\citet{mainbib:hou2023probably} require every selected action's sum of
component variances to stay below a budget,
simultaneously with high probability, and provide regret
and information-theoretic lower bounds. Our expected cumulative
baseline-regret bound implies neither this variance constraint
nor the itemwise safety guarantee above.

\subsection{Cumulative constraints and resource budgets}

\citet{mainbib:mahdavi2012trading} replace per-round convex feasibility with aggregate constraints
and analyze regret--violation tradeoffs via online convex-concave optimization.
Their extension uses two-point constraint feedback while observing
the loss functions. Both the constraints and
feedback differ from our loss-dependent baseline comparison.

\citet{mainbib:badanidiyuru2018bandits} introduce bandits with knapsacks, where limited resources
constrain achievable reward. \citet{mainbib:sankararaman2018combinatorial} add combinatorial actions and
semi-bandit feedback. Our budget concerns a signed loss
difference: outperforming the baseline contributes negatively. This quantity
and its comparator differ from standard knapsack accounting.

\subsection{Stochastic and adversarial adaptation}

Best-of-both-worlds algorithms combine adversarial robustness with faster stochastic
regret. \citet{mainbib:bubeck2012best} test consistency with a stochastic model.
\citet{mainbib:seldin2014practical} obtain stochastic and oblivious-adversarial guarantees without knowing
the horizon, using EXP3 with arm-dependent exploration.

\citet{mainbib:auer2016nearly} sharpen simultaneous pseudo-regret bounds and separate
pseudo-regret from expected hindsight regret under adaptive
adversaries. For random comparator-dependent regrets $D_T(X)$,
\[
 \max_{X\in\cX}\E D_T(X)
 \leq \E\max_{X\in\cX}D_T(X),
\]
so controlling the left side does not control
the right. Our final comparator maximum is outside
expectation; the estimated positive maximum used in
the proof is bounded separately. Applying the
separation of \citet{mainbib:auer2016nearly} to this certificate would
require matching the comparator, adversary, and rates.

\citet{mainbib:xu2021simple} obtain corruption-dependent stochastic semi-bandit regret and
approximation-oracle guarantees with an approximation-regret benchmark. Our
i.i.d. assumption does not cover adversarial corruptions,
and our benchmark uses exact feasible actions.

\subsection{Adaptive regularization and hybrid FTRL}

\citet{mainbib:wei2018adaptive} use optimistic mirror descent and log-barrier regularization
for data-dependent bounds, including stochastic fast rates and
combinatorial semi-bandit results; different guarantees may require different
instantiations. \citet{mainbib:zimmert2021tsallis} develop Tsallis-INF and self-bounding analysis for
stochastic, adversarial, and certain corrupted stochastic multi-armed bandits.

The pure-phase learner in \Duet{} follows the hybrid-FTRL method of
\citet{mainbib:zimmert2019beating}, which combines logarithmic stochastic and square-root
adversarial semi-bandit pseudo-regret without knowing the regime
or horizon. We apply its regularizer after
the reduction in Section~\ref{sec:base-learner} and prove a
monotone empirical-regret majorant for the resulting implementation.

\citet{mainbib:ito2021hybrid} combine stochastic, corrupted-stochastic, and adversarial guarantees with
loss-sequence measures. \citet{mainbib:tsuchiya2023further} refine regularization and learning rates
for semi-bandits, including variance-dependent stochastic bounds. \citet{mainbib:ito2024adaptive} study
stability--penalty matching for FTRL learning rates; \citet{mainbib:nguyen2025data} use
it for further data-dependent and best-of-both-worlds multi-armed bandit
guarantees. Replacing the base learners in \Duet{} with
these unconstrained refinements would require verifying their stopped
empirical-maximum bounds.

\subsection{Perturbation methods and contextual extensions}

\citet{mainbib:zhan2025nearly} use follow-the-perturbed-leader with shape-two Fr\'echet perturbations
for $m$-set semi-bandits. They obtain logarithmic stochastic
regret with a unique optimum and near-best-of-both-worlds
guarantees when $m\leq d/2$. Our arbitrary finite action
sets and baseline constraint require further arguments.

\citet{mainbib:li2026efficient} study contextual combinatorial semi-bandits with Shannon-entropy FTRL
and efficient projection. Their linear model uses
independent bounded contexts with positive-definite covariance, additional
independent context samples for estimation, and adversarial
loss coefficients independent of the action sequence.
We use no contexts and impose expected
regret constraints relative to a covering baseline.

\subsection{Computation, alternative adaptation, and optimal actions}

\citet{mainbib:neu2016importance} combine geometric resampling and follow-the-perturbed-leader for oracle-efficient
adversarial semi-bandits without computing all sampling probabilities. Our
implementation enumerates actions. \citet{mainbib:neu2015firstorder} obtain first-order regret bounds
adapting to the best action's cumulative loss.
This differs from gap adaptation: that loss
may grow linearly despite a positive gap.
\citet{mainbib:lykouris2018small} give high-probability small-loss bounds for actual
bandit and semi-bandit regret, a stronger probability
guarantee than ours, without our baseline constraint.

\citet{mainbib:jin2023improved} establish best-of-both-worlds FTRL guarantees with general regularizers
and multiple optimal arms in multi-armed bandits. Our
certificate and restart analysis retains a unique optimal
combinatorial action and a positive minimum gap;
those results do not remove this assumption.

\subsection{The certificate required by baseline-constrained phases}

Mixed sampling changes observation probabilities, empirical tests select
phase endpoints, and past observations determine restarts. A
fixed-length expected-regret bound alone does not control the
positive empirical comparator maximum at those endpoints.

Lemma~\ref{lem:one-phase-empirical-regret} proves this control using coverage and entropy
FTRL in mixed phases, and a nonnegative, nondecreasing
hybrid-FTRL majorant in pure phases. Lemma~\ref{lem:conditional-restart} justifies
reuse after data-dependent restarts. Our martingale estimates
use the concentration framework of \citet{mainbib:freedman1975tail}; \citet{mainbib:lattimore2020bandit}
provide background on regret, importance weighting, and
concentration. Appendix~\ref{app:probability-tools} states the versions used here.

\Duet{} satisfies the baseline, adversarial, and stochastic objectives
at every fixed horizon without taking that horizon
as input. The guarantees
are neither simultaneous high-probability baseline bounds nor
regret bounds at arbitrary data-dependent stopping times.


%
\end{document}